\documentclass[onefignum,onetabnum, preprint]{siamonline220329}
\usepackage{braket,amsfonts}
\usepackage{xr}
\usepackage{array}

\usepackage{pgfplots}
\usepackage{bm}
\usepackage{amssymb}
\newsiamthm{example}{Example}
\newsiamthm{claim}{Claim}
\newsiamremark{remark}{Remark}
\newsiamremark{hypothesis}{Hypothesis}
\crefname{hypothesis}{Hypothesis}{Hypotheses}

\newsiamthm{assumption}{Assumption}

\usepackage{algorithm}
\usepackage[noend]{algpseudocode}

\usepackage{epstopdf}

\Crefname{ALC@unique}{Line}{Lines}

\usepackage{amsopn}

\usepackage{xspace}
\usepackage{bold-extra}
\usepackage[most]{tcolorbox}

\colorlet{texcscolor}{blue!50!black}
\colorlet{texemcolor}{red!70!black}
\colorlet{texpreamble}{red!70!black}
\colorlet{codebackground}{black!25!white!25}
\usepackage{hhline}
\usetikzlibrary{chains}
\usetikzlibrary{spy}
\usepackage{caption}
\usepackage{multirow}

\usepackage{mathtools}
\newcommand{\E}{{\mathbb E}}
\newcommand{\PZ}{{\mathbb P_Z}}
\newcommand{\Pdata}{{\mathbb P_{\mathrm{data}}}}
\newcommand{\R}{\mathbb{R}}

\graphicspath{{images/}{./}}

\title{Multilevel Diffusion: Infinite Dimensional Score-Based Diffusion Models for Image Generation \thanks{\funding{NY and LR's work was supported in part by NSF awards DMS 1751636, DMS 2038118, AFOSR grant FA9550-20-1-0372, and US DOE Office of Advanced Scientific Computing Research Field Work Proposal 20-023231. PH and GS gratefully acknowledge funding within  the DFG-SPP 2298 ,,Theoretical Foundations of Deep Learning''. }}}

\author{Paul Hagemann \thanks{TU Berlin (E-mail: {\{hagemann, steidl\}}@math.tu-berlin.de, mildenberger@tu-berlin.de)} \and Sophie Mildenberger\footnotemark[2] \and Lars Ruthotto \thanks{Department of Mathematics, Emory University, Atlanta, GA (E-mail: {\{lruthotto, tyang31\}@emory.edu})}  \and Gabriele Steidl \footnotemark[2] \and \hspace{1cm} Nicole Tianjiao Yang \footnotemark[3]}  
\usepackage{tikz-cd}
\usepackage{adjustbox}
\usepackage{subfigure}

\begin{document}
\maketitle

\begin{tcbverbatimwrite}{tmp_\jobname_abstract.tex}
\begin{abstract}
Score-based diffusion models (SBDM) have recently emerged as state-of-the-art approaches for image generation. 
Existing SBDMs are typically formulated in a finite-dimensional setting, where images are considered as tensors of finite size.
This paper develops SBDMs in the infinite-dimensional setting, that is, we model the training data as functions supported on a rectangular domain. 
In addition to the quest for generating images at ever-higher resolutions, our primary motivation is to create a well-posed infinite-dimensional learning problem that we can discretize consistently on multiple resolution levels.
We thereby intend to obtain diffusion models that generalize across different resolution levels and improve the efficiency of the training process.
We demonstrate how to overcome two shortcomings of current SBDM approaches in the infinite-dimensional setting. 
First, we modify the forward process using trace class operators to ensure that the latent distribution is well-defined in the infinite-dimensional setting and derive the reverse processes for finite-dimensional approximations.
Second, we illustrate that approximating the score function with an operator network is beneficial for multilevel training. After deriving the convergence of the discretization and the approximation of multilevel training, we demonstrate some practical benefits of our infinite-dimensional SBDM approach on a synthetic Gaussian mixture example, the MNIST dataset, and a dataset generated from a nonlinear 2D reaction-diffusion equation.
\end{abstract}


\end{tcbverbatimwrite}
\input{tmp_\jobname_abstract.tex}

\section{Introduction}

Score-based diffusion models (SBDMs), first introduced in \cite{pmlr-v37-sohl-dickstein15}, have achieved impressive results in high dimensional image generation \cite{DDPM, songestimating}. 
The papers \cite{huangVarScore, SongEtAl2020}  provided the underlying theoretical foundations by leveraging the probabilistic foundations of diffusion models. 
In \cite{song2021maximum}, it was shown that for each diffusion model, there is a continuous normalizing flow that produces the same marginals, enabling exact likelihood matching.

Most, if not all, existing SBDMs for image generation interpret the training data as tensors that consist of a finite number of pixels or voxels and a few channels.  
This finite-dimensional viewpoint and the resulting learning schemes have recently emerged as state-of-the-art approaches for image generation and produced impressive results; see, e.g., \cite{rombach,SongEtAl2020}.

In this paper, we model the training data as functions supported on a rectangular domain and extend SBDMs to function spaces, in our case, a separable Hilbert space. There is keen interest in generating high-resolution images; hence, an ill-posed infinite-dimensional formulation could limit progress in this area.
Further, similar abstractions from discrete to continuous have, for example, led to new neural network architectures motivated by ODEs and PDEs~\cite{E:2017kz, HaberRuthotto2017, chen2018neural, ruthotto2020deep} and created the field of PDE-based image processing, which has generated new insights, models, and algorithms in the last decades~\cite{chan2003variational}. 
Third, a well-posed infinite-dimensional problem and a suitable discretization can have practical advantages. For example, it allows one to approximate the solutions of the infinite-dimensional problem by using a coarse-to-fine hierarchy of resolution levels, thereby lowering computational costs.

As shown below, the extension and interpretation of SBDMs in the infinite limit is nontrivial, and we tackle several theoretical challenges.

 \paragraph{Contributions and main results} 
Existing SBDMs use a standard Gaussian as a latent distribution, which is not well defined in the infinite-dimensional setting.
To overcome this issue, we choose a Gaussian with a trace class covariance operator as the latent distribution and modify the forward and reverse process accordingly. We show that the finite-dimensional forward stochastic differential equation possesses a well-defined limit as the solution of a Hilbert space-valued stochastic differential equation. The reverse equation is trickier to handle in the presence of the infinite-dimensional score. We circumvent this technical issue by considering the limit of the finite-dimensional discrete stochastic differential equations and the reverse SDE theory \cite{FOLLMER198659}. With this approach, we retain weak convergence guarantees. 

We give an approximation result for replacing the reverse score with a learned neural network and show that under Lipschitz conditions, the path measures of the forward and reverse SDE are close; see \cref{thm:main}. To motivate multilevel training, we also prove in \cref{thm:multilevel} that training on coarse scores yields good approximations on finer levels.

Following our theoretical guideline on constructing well-posed infinite-dimensional SBDM, we first show the limitations of standard U-Nets, which can be sensitive to image resolution as they interpret their input images as finite-dimensional tensors and not as functions. Therefore, we use operator-valued NNs to approximate the score functions. We experimentally compare different trace class priors (FNO, Bessel, and a combined prior of FNO and Laplacian) to the standard Gaussian.
Our experiments show that operator-based architectures and trace-class priors yield SBDMs that generate similar quality images at the training resolution and improve generalization across resolutions; see \cref{sec:numerics}. 

\paragraph{Related work}

Infinite-dimensional extensions of generative adversarial networks (GANs) have recently shown promise for generating function-valued data such as images. 
For example, GANOs \cite{GANOs} replace the latent distribution by Gaussian random fields and introduce a continuous loss. In \cite{dupontetal} the infinite-dimensional representation of images is combined with an adversarial loss. Furthermore, \cite{karras2018progressive} uses a layer-adding strategy to enable multilevel training. The recent paper \cite{sinDiff} trains a diffusion model on a single image view by a hierarchical ansatz, which generates first coarse and then finer images based on a single image. This is another way to achieve coarse-to-fine image generation. 

The ideas in our paper most closely resemble the themes of the concurrent works \cite{func_diff,kerrigan,nvidia_func_diff,jakiw_diffusion}. While \cite{jakiw_diffusion} considers a Langevin SDE, we take a time schedule into account. Furthermore, our proof of the Wasserstein distance on the path measures uses some ideas in \cite{jakiw_diffusion}. While their work is focused on defining the infinite-dimensional score object, we provide in-depth proofs of approximations on different levels and operate in finite dimensions to ensure the Lipschitz assumptions. 

An infinite-dimensional extension of the DDPM framework~\cite{DDPM} and promising results for generating time series data is presented in \cite{kerrigan}. In contrast to this work, we prefer the SBDM as it also leads to a continuous process in time.
 The infinite-dimensional SBDM frameworks in \cite{func_diff,nvidia_func_diff,jakiw_diffusion} and our approach model the generation problem in Hilbert spaces, employ trace class noise and contain similar theoretical results with some subtle differences.
 Similar to~\cite{nvidia_func_diff} we see benefits of approximating the score function with a Fourier neural operator network.
 Compared to~\cite{nvidia_func_diff}, which solves the training problem in Hilbert space, we construct sequences of finite-dimensional approximations by solving suitable training problems in Euclidean space and show their convergence. The recent paper~\cite{func_diff} considers a different kind of score model numerically, and they employ the notion of reverse solution of \cite{millet1989time}. They show a sampling theorem, which discusses the condition that one can reconstruct an infinite-dimensional function from a countable set of points. Another recent work \cite{bond2023diff} uses the truncated Gaussian kernel and a U-Net-inspired multiscale architecture. It demonstrates good performance on super-resolution and inpainting problems.
 
 What sets our work apart from the above papers is our focus on classical image generation problems, albeit with relatively simple datasets, and our emphasis on multilevel training. We provide the theoretical guarantee of the convergence of the discretization of the infinite-dimensional diffusion model and demonstrate the viability of multilevel training on the path of the forward SDE. Moreover, we believe the experiment on different choices of prior distributions sheds some light on future improvements in infinite-dimensional diffusion models.
 
A multiscale diffusion approach based on Haar-Wavelets and non-uniform diffusion models is presented in \cite{Multiscale_Wavelet_Diff}.
This method uses different score models for each diffusion scale and anneals each pixel at a different speed. 
The numerical results suggest some advantages of exploiting different scales in terms of training time.
While our multilevel approach considers a coarse-to-fine hierarchy of convergent discretizations and, importantly, we use a trace class operator to model the latent distributions and achieve a well-posed infinite-dimensional problem; see also discussion in \cite{NonIsotrDiff}.
A coarse-to-fine diffusion approach was also proposed in \cite{rissanen,lee} with different choices of the forward operator. A heat dissipation forward operator was presented in~\cite{rissanen}, whereas \cite{lee} suggested a blurring operator that diffuses different frequencies at different speeds. The idea of training at different resolutions without infinite-dimensional motivation has also been employed in \cite{hoogeboom_multiscale}. A recent wave of infinite-dimensional generative models has also been applied to inverse problems, namely \cite{alberti2023continuous, baldassari2023conditional}.

The paper is organized as follows.
In \cref{sec:score_finite}, we consider the finite-dimensional Euclidean setting and define an SBDM that maps the data distribution to a Gaussian with zero mean and general covariance matrix.
Then, in the \cref{SDE-H}, we interpret and analyze the SBDM in the infinite-dimensional Hilbert space setting, present our discretization approach, and show their consistency. In \cref{sec: theorytoexp}, we demonstrate the design of diffusion models that enable the operations on function space, mainly in terms of the neural network architecture and the prior distribution of the forward SDEs.
In \cref{sec:numerics}, we provide experiments with different covariance operators, network architectures, and resolution scales using the MNIST dataset and PDE pushforward operators.
Finally, in \cref{eq:discuss}, we discuss our findings and identify open questions.

\section{Finite-dimensional Score-Based Diffusion Models}
\label{sec:score_finite}
Following the general steps of \cite{huangVarScore, SongEtAl2020, song2021maximum}, 
we present the forward and reverse process of a score-based diffusion model that maps the data distribution 
$\mathbb{P}_{\rm data}$ to the latent distribution $\mathbb{P}_Z$ on $\R^d$ for some finite dimension $d$. 
Our presentation in \cref{sderd} lays  the foundation 
for the infinite-dimensional case by generalizing the SBDM 
for Gaussians with general covariance matrices.
In \cref{motiv_trace}, we show why the common choice of an identity covariance matrix 
and other covariance matrices whose trace grows in the dimension does not yield 
a well-posed infinite-dimensional limit.

\subsection{Finite-dimensional Stochastic differential equations}\label{sderd}

The main idea of SBDMs consists of constructing a pair of stochastic differential equations, namely a forward and a reverse one, such that the forward process transforms the data into noise, and the reverse process transforms the noise into data. 

The forward SDE computes a chain of random variables $X_t \in \R^d$, $t \in [0,T]$ 
by solving
\begin{equation} 
\label{eq}
{\rm d} X_t = f(t,X_t) \, {\rm d} t + g(t) \, {\rm d} W_t^Q
\end{equation}
with some initial condition 
$X_0 \sim \mathbb{P}_0 := \mathbb{P}_{{\rm{data}}}$, noise process $W_t^Q = \sqrt{Q} W_t$, where $W_t$ is a standard Brownian motion and $Q\in\R^{d\times d}$  is a symmetric 
positive definite covariance matrix, and, similar to \cite{huangVarScore,SongEtAl2020}, we choose
\begin{equation}\label{our}
f(t,X_t) \coloneqq -\frac{1}{2}\alpha_t X_t,\quad g(t,X_t) = g(t) \coloneqq \sqrt{\alpha_t}.
\end{equation}
Here, $\alpha_t \in (0,\infty)$ is a positive and increasing schedule.
With the above choices, the final distribution $X_T$ is near  $Y_0 \sim \PZ$. By $p_t$ we denote the density of $X_t \sim \mathbb{P}_{X_t}$ at $t \in [0,T]$ with respect to the Lebesgue measure.
 It is most common to use $Q = \mathrm{Id}$ (the identity matrix on $\R^d$). However, we show that this choice is inadequate in the infinite-dimensional setting. 


Under certain assumptions, cf. \cite{Anderson1982, haussmann1986time}, the reverse process $Y_t := X_{T-t}$ satisfies 
\begin{equation}
\label{eq_rev_finite}
{\rm d} Y_t = \left(-f( T-t, Y_t) +  g(T-t)^2 Q \nabla \log p_{T-t}(Y_t)\right) \, {\rm d} t + g(T-t) \, {\rm d} W_t^Q,
\end{equation}
which is, in contrast to the forward SDE, a nonlinear SDE whose drift depends on the gradient of the \emph{score}, $\nabla \log p_{t}$, the log density of $X_t$.
The reverse SDE can be derived from the standard SBDM used in~\cite{huangVarScore, SongEtAl2020, song2021maximum} with the Brownian motion $W_t^{\mathrm{Id}}$, by setting $g(t) := \sqrt{\alpha_t Q}$ in the standard SBDM.

To reverse the forward SDE, i.e., to go from the latent distribution $Y_0 \sim \PZ$ to $Y_T$ near $\Pdata$, 
we need to learn the score of the forward process.
In \cite{DDPM, huangVarScore, SongEtAl2020}, it is proposed to approximate the score with the neural network $s_{\theta}(t,x(t))$  and train its weights by minimizing the expected least-square error between the output of the network and the true score via
$$
\min_{\theta} \mathbb{E}_{t \sim U[0,T]} \mathbb{E}_{x(t) \sim \mathbb{P}_{t} }\left[ \Vert s_\theta(t,x(t)) -Q \nabla \log p_t(x(t)) \Vert^2 \right].
$$
Note that there are different ways to define this loss if $Q\neq \mathrm{Id}$. For example, we could also optimize $s_{\theta}$ in such a way that $Q\ s_{\theta} (t,x(t))\approx Q \nabla \log p_t(x(t))$. Since the score function cannot be observed in practice, \cite{hyvarinen_score, vincent_dsm} showed that the above minimization problem is equivalent to $\min_{\theta} L_{\rm DSM}(\theta)$ with
$$
 L_{\rm DSM} (\theta) = \mathbb{E}_{x(0) \sim \Pdata, t \sim U[0,T]} \mathbb{E}_{x(t) \sim \mathbb{P}_{X_t|X_0 = x(0)} }\left[ \Vert s_\theta(t,x(t)) - Q \nabla \log p_t(x(t)|x(0)) \Vert^2 \right],
$$
which does not involve the score $\nabla \log p_t$ itself, but instead the conditional distribution. In our case \eqref{our} the conditional distribution is Gaussian, i.e.,
$\mathbb{P}_{X_t|X_0 = x(0)} = \mathcal{N}(\widetilde{\alpha_t} x(0), \widetilde{\beta_t} Q)$.
where $\widetilde{\alpha_t} = a_t^{-1}$, $\widetilde{\beta_t} = (1-\frac{1}{a_t^2})$, and $a_t = e^{\frac{1}{2}\int_0^t \alpha_r \, {\rm d} r}$. 
This can be seen either by discretization~\cite{DDPM} or by our \cref{var}, which uses the variation of constants formula. This implies
\begin{equation}
\begin{aligned}
\nabla \log p_t(x(t)|x(0)) &=
\nabla \left(-\tfrac12 \left(x(t) - \widetilde \alpha_t x(0) \right)^T \widetilde \beta_k^{-1} Q^{-1}\left(x(t) - \widetilde \alpha_t x(0) \right) \right)\\
&=- \widetilde \beta_t^{-1} Q^{-1}\left(x(t) - \widetilde \alpha_t x(0) \right).
\end{aligned}
\end{equation}
Plugging this into the loss function, we get
\begin{equation}
\begin{aligned}
L_{DSM} (\theta) & = \mathbb{E}_{x(0) \sim \Pdata, t \sim U[0,T]} \mathbb{E}_{x(t) \sim \mathbb{P}_{X_t|X_0 = x(0)} }
\left[ \left\| s_\theta(t,x(t)) + \frac{
x(t) - \widetilde \alpha_t x(0)}{\widetilde{\beta_t}} \right\|^2 \right]. 
\end{aligned}
\end{equation}
Setting  $\eta \coloneqq  \widetilde \beta_t^{-1/2}(x(t) -\widetilde{\alpha_t} x(0) )  \sim \mathcal N(0, Q)$, this can be further simplified to 
\begin{equation}
\begin{aligned}
\label{eq:LDSM}
L_{DSM} (\theta) & =  \mathbb{E}_{x(0) \sim \Pdata, t \sim U[0,T]} \mathbb{E}_{x(t) \sim \mathbb{P}_{X_t|X_0 = x(0)} }
\left[ \left\| s_\theta(t,x(t)) + \widetilde{\beta_t}^{-\frac12}\eta \right\|^2 \right],
\end{aligned}
\end{equation}
see also \cite{SongEtAl2020, NonIsotrDiff}. 


\begin{remark}
When discretizing in time, the forward and reverse models
become a pair of Markov chains. A rigorous framework of these schemes also in the case of non-absolutely continuous distributions and its relation to (stochastic) normalizing flow architectures was considered in \cite{HHS2022,HHS2023}. For a finite number of time steps, the DDPM framework \cite{DDPM,pmlr-v37-sohl-dickstein15,SongEtAl2020} is recovered. 
In particular, these works minimize the joint Kullback--Leibler divergence between a forward path and the reverse path as a loss function.
\end{remark}

\subsection{Motivation for trace class Gaussian measures}
\label{motiv_trace}
We briefly discuss that the common choice of $Q={\rm Id}$, while leading to a well-defined SBDM problem in finite dimensions, is not suitable for the infinite-dimensional extension of SBDM.
The key issue is that there is no standard Gaussian random variable in infinite dimensions. 
To see why, consider the Euclidean norm of the standard Gaussian random variable $X^d$ in $\mathbb{R}^d$, then it holds that 
 \begin{equation}\label{eq:N01example}
 \mathbb{E}[\Vert X^d \Vert_2^2] =  \mathrm{tr}(\mathrm{Id}) = d,    
 \end{equation}
 see also \cite{vershynin2018high}.
 In particular, for $d \rightarrow \infty$ as in the function space limit, we see that $\mathbb{E}[\Vert X^d \Vert^2] \rightarrow \infty$. This is one way to see why the standard Gaussian distribution becomes problematic in infinite dimensions; see, e.g., \cite{hairer2009introduction,stuart2010inverse} for in-depth discussions of this issue.

This motivates the use of so-called \textit{trace class operators} in infinite dimensions, i.e., operators whose trace remains bounded as $d\to\infty$. 
As can be seen in~\eqref{eq:N01example}, if $Q$ is not trace class, then the $\ell^2$-sum of the entries of $X^d$  goes to $\infty$ as the resolution approaches $\infty$. 



\section{Score-Based Diffusion Models on Infinite-Dimensional Hilbert Spaces}\label{SDE-H}
In this section, we consider the SBDM problem for $H$-valued random variables, 
where $H$ is a separable Hilbert space.
The necessary preliminaries on Gaussian distributions on $H$, 
trace class operators $Q \in L(H)$, and  $Q$-Wiener processes are given in the supplementary material. 
One of the key ideas is to express the processes in terms of the spectral decomposition of $Q$. 
In \cref{sdehil}, we deal with the solution of the forward process
and its approximation by finite-dimensional discretizations using the truncated spectral decomposition of $Q$.
We prove an explicit formula for the discretization error and show convergence.
In \cref{sub:discr}, we derive reverse SDEs for the discretized forward process
and prove Wasserstein error bounds between the learned approximations of the reverse process 
and the data distribution.

Note that in this section, we provide outlines of the proofs to streamline the discussion; for full detail, we provide references to the supplementary material.

\subsection{Forward SDE and its Discretization}
\label{sdehil}

Let $(H,\lVert \cdot\rVert)$ be a real separable Hilbert space, $K(H)$ the space of compact operators on $H$, 
and $Q \in K(H)$ a trace class operator. 
By $H_Q\coloneqq Q^\frac12 H$, we denote the Cameron-Martin space of $Q$ and
by $L_2(H_Q,H)$ the space of Hilbert-Schmidt operators from $H_Q$ into $H$.
For a terminal time $T>0$, consider continuous drift and diffusion functions
$f \colon  [0,T] \times H \to H$
and
$g \colon [0,T] \to L_2(H_Q,H)$, respectively. We need the notion of strong and weak solutions of SDEs. For a given random variable $X_0 \in H$ on some probability space 
$(\Omega, \mathcal F, \mathbb P)$ and a $Q$-Wiener process $W^Q$ on $(\Omega, \mathcal F, \mathbb P)$ that is independent of $X_0$, 
we call a random process $(X_t)_{t \in [0,T]}$ on  
a \emph{strong solution} of \eqref{eq}
with initial condition $X_0$, 
if $(X_t)_{t \in [0,T]}$ has continuous sample paths, satisfies $\mathbb P$-almost surely
\begin{equation}\label{sol}
    X_t = X_0 + \int_0^t f(s,X_s) \, {\rm d} s + \int_0^t g(s) \, {\rm d} W^Q_s , \qquad t \in [0,T]
\end{equation}
and is adapted to the smallest complete filtration $\mathbb F=(\mathcal F)_{t \geq 0}$ in which $X_0$ is $\mathcal{F}_0$-measurable and $W^Q$ is $\mathbb F$-adapted. 
It is easy to check that, due to the independence of $X_0$ and $W^Q$, the process $W^Q$ is also a $Q$-Wiener process on the \emph{filtered probability space} $(\Omega, \mathcal F, \mathbb F, \mathbb P)$. 
 Let $(e_k)_{k\in\mathbb{N}}$ be an orthonormal basis of eigenfunctions of $Q$ corresponding to eigenvalues $(\lambda_k)_{k\in\mathbb{N}}$. By \cite[Theorem 3.17]{scheutzowLectureNotes} an $H$-valued random process $W^Q$ on $(\Omega, \mathcal F, \mathbb P)$ is a $Q$-Wiener process if and only if there exist mutually independent standard Wiener processes $(\beta_k)_{k \in \mathbb{N}}$ on $(\Omega, \mathcal F, \mathbb P)$ such that 
\begin{equation}\label{eq:q}
    W^{Q}_t = \sum_{k \in \mathbb{N}} \sqrt{\lambda_k} \, \beta_k(t) \, e_k,
\end{equation} 
where the series on the right-hand side, the so-called \emph{Karhunen–Lo\`eve decomposition} of $W^Q$, is convergent in $L^2\left(\Omega,\mathcal F,\mathbb P;\mathcal C([0,T],H) \right)$.

We call a triple $((\Omega,\mathcal F,\mathbb F,\mathbb P),(X_t)_{t \in [0,T]},W^Q)$ a 
\emph{weak solution} of \eqref{eq} with initial condition $\mathbb P_0$ if $(\Omega,\mathcal F,\mathbb F,\mathbb P)$ is a filtered probability space with normal filtration, 
$W^Q$ is a $Q$-Wiener process on $(\Omega,\mathcal F,\mathbb F,\mathbb P)$, and 
$(X_t)_{t \in [0,T]}$ is $\mathbb F$-adapted, satisfies $X_0 \sim \mathbb P_0$, 
has continuous sample paths, and fulfills $\mathbb P$-almost surely \eqref{sol}. 
Note that the filtered probability space and the driving $Q$-Wiener process are part of the notion of weak solutions and not chosen a priori as is in the case of strong solutions. Clearly, any strong solution corresponds to a weak solution. However, the reverse is not true (\cite[Ex.~3.7]{scheutzowWT4}).

We define the forward process by choosing the schedule $\alpha \in \mathcal{C}^1([0,T];\mathbb R_{> 0})$, the drift 
$f(t,x) \coloneqq \frac{1}{2}\alpha_t x$, 
 the diffusion
$g(t) \coloneqq \sqrt{\alpha_t}$, i.e., $g(t): H_Q \to H$ with $y \mapsto \sqrt{\alpha_t} y$.
Thus we have
\begin{equation}\label{eq:SDE}
    {\rm d} X_t = -\frac{1}{2}\alpha_t X_t \, {\rm d} t + \sqrt{\alpha_t} \, {\rm d} W^Q_t
\end{equation}
with initial condition $X_0 \sim \Pdata$, where $\Pdata$ is the data distribution. 
The following theorem establishes the existence of a unique strong solution of \eqref{eq:SDE} and 
provides an explicit formula for this solution. The existence and uniqueness statements of the theorem are based on \cite[Theorem~3.3]{GawMan10}, and the explicit formula can be derived by a variation of constants type argument. The full proof is  given in the supplementary material. 

\begin{theorem}[Variation of constants formula]\label{var}
    Let $X_0 \in H$ be a random variable and let $W^Q$ be a $Q$-Wiener process that is independent of $X_0$. 
		Then the SDE \eqref{eq:SDE} has a unique strong solution $(X_t)_{t \in [0,T]}$, 
		which is given $\mathbb{P}$-almost surely for all $t \in [0,T]$ by
    \begin{equation}\label{eq:var}
        X_t = a_t^{-1}  \Big( X_0 + \int_0^t  a_s \, \sqrt{\alpha_s} \, {\rm d} W^Q_s\Big), 
				\qquad a_t \coloneqq e^{\frac{1}{2}\int_0^t \alpha_r \, {\rm d} r}.
    \end{equation}
 If $\mathbb E \left[ \lVert X_0 \rVert^2\right]<\infty$, then it holds $\mathbb E \left[ \sup_{t \in [0,T]} \lVert X_t\rVert^2 \right]<\infty$. Let  $W^Q$ have the Karhunen–Lo\`eve decomposition \eqref{eq:q}.
Then, for all $t \in [0,T]$, the stochastic integral
\begin{equation}\label{eq:B}
    B_t \coloneqq \int_0^t  a_s \, \sqrt{\alpha_s} \, {\rm d} W^Q_s = \sum_{k \in \mathbb N} \sqrt{\lambda_k}\, 
		\Big(\int_0^t  a_s \, \sqrt{\alpha_s} \, {\rm d}\beta_k(s)\Big) \, e_k,
\end{equation}
is a centered $(a_t^2-1)Q$-Gaussian. 
\end{theorem}

Next, we discretize the $H$-valued forward process defined by \eqref{eq:SDE} 
by truncating the Karhunen–Lo\`eve decomposition \eqref{eq:q} of $W^Q$.  
Restricting the forward SDE \eqref{eq:SDE} onto $H_n \coloneqq {\rm span}\{e_1,\ldots,e_n\}$
yields the SDE
\begin{equation}     \label{finiteforward}
    {\rm d} X^n_t = -\frac{1}{2} \alpha_t X^n_t \, {\rm d} t + \sqrt{\alpha_t} \, {\rm d} W^{Q_n}_t, 
\end{equation}
with initial condition $X_0^n\coloneqq \sum_{k=1}^n \langle X_0,e_k \rangle \, e_k$ and 
\begin{equation} \label{eq:n}
W^{Q_n}_t \coloneqq 
 \sum_{k=1}^n \sqrt{\lambda_k} \, \beta_k(t) \, e_k.
\end{equation}
Setting 
$Q_n \coloneqq \sum_{k=1}^n \lambda_k \, \langle e_k, \cdot \rangle e_k$, we note that $W^{Q_n}$ is a $Q_n$-Wiener process, which justifies the notation. By \cref{var}, the restricted equation \eqref{finiteforward} has a unique strong solution that satisfies $\mathbb P$-almost surely
    \begin{equation}\label{eq:finvar}
        X^n_t = a_t^{-1} \, ( X^n_0 + B^n_t ), \quad
        B^n_t \coloneqq \int_0^t a_s \, \sqrt{\alpha_s}  \, {\rm d} W^{Q_n}_s 
				= \sum_{k=1}^n  \sqrt{\lambda_k}  \, \int_0^t a_s \, \sqrt{\alpha_s} \, {\rm d} \beta_k(s) \, e_k.
    \end{equation}
Since $X^n_t$ is simply the projection of  $X_t$ onto $H_n$, 
we expect that it is a good approximation of $X_t$ for large $n$. 
The next theorem gives an explicit bound on the approximation error and is similar to the argument in \cite[Theorem 1]{jakiw_diffusion}.

\begin{theorem}[Convergence of forward process]\label{forConv}
    Let $W^Q$ be a $Q$-Wiener process with trace class operator $Q$ having the Karhunen–Lo\`eve decomposition \eqref{eq:q} and let $W^{Q_n}$ be its truncation \eqref{eq:n}. 
		If $(X_t)_{t \in [0,T]}$ denotes the solution of the SDE \eqref{eq:SDE} 
		and $(X_t^n)_{t \in [0,T]}$ the solution of the approximate problem \eqref{finiteforward}, 
		then
    \begin{equation} \label{eq:drei}
        \mathbb E \left[\lVert X_t -X_t^n \rVert^2 \right] 
				\le 2 \, a_t^{-2}\, \mathbb E \left[ \lVert X_0-X_0^n \rVert^2 \right]+ 2\,  (1-a_t^{-2}) \, \mathrm{tr}(Q-Q_n), \qquad t \in [0,T],
    \end{equation}
    and
    \begin{equation*}
        \mathbb E \Big[\sup_{t \in [0,T]}\lVert X_t -X_t^n \rVert^2 \Big]  \le 2 \, \mathbb E \left[ \lVert X_0-X_0^n \rVert^2 \right]
				+ 8 \, \big(a_T^2-1\big) \, \mathrm{tr}(Q-Q_n).
    \end{equation*}
    If $X_0 \in H$ fulfills $\mathbb E \left[ \|X_0\|^2 \right] < \infty$,
		then $\mathbb E \left[ \sup_{t \in [0,T]} \lVert X_t -X_t^n \rVert^2 \right] \to 0$ as $n \to \infty$.
\end{theorem}

\begin{proof} 
    By applying \cref{var} to the initial value $X_0-X_0^n$ and the $(Q-Q_n)$-Wiener process $W^Q-W^{Q_n}$, 
		we obtain that
    \begin{equation*}
        \delta B_t^n \coloneqq \int_0^t a_s \, \sqrt{\alpha_s} \,  {\rm d} (W_s^Q-W_s^{Q_n})
    \end{equation*}
    is a centered $(a_t^2-1)(Q-Q_n)$-Gaussian and
    \begin{equation*}
        \mathbb E \left[ \left\lVert \delta B_t^n \right \rVert^2 \right] = (a_t^2-1) \, \mathrm{tr}(Q-Q_n).
    \end{equation*} 
    Utilizing the variation of constants formula \eqref{eq:var}, we conclude
  \begin{equation} \begin{aligned}
        X_t-X_t^n 
				&= a_t^{-1} \Big(X_0-X_0^n + \int_0^t  a_s \,\sqrt{\alpha_s} \, {\rm d} W^Q_s
				-\int_0^t a_s \,\sqrt{\alpha_s} \, {\rm d} W^{Q_n}_s\Big)\\
        & = a_t^{-1} \Big(X_0-X_0^n + \delta B_t^n \Big).
   \end{aligned}\end{equation}
    In summary, we get
  \begin{equation} \begin{aligned}
        \mathbb E \left[\lVert X_t -X_t^n \rVert^2 \right] 
				&  \le 2 \, a_t^{-2} \, \mathbb E \left[ \lVert X_0-X_0^n \rVert^2 \right]+ 2 \,  a_t^{-2}\, \mathbb E \left[\left\lVert \delta B_t^n \right\rVert^2\right]\\
        & \le 2 \, a_t^{-2}\, \mathbb E \left[ \lVert X_0-X_0^n \rVert^2 \right]+ 2\,  (1- a_t^{-2}) \, \mathrm{tr}(Q-Q_n),
   \end{aligned}\end{equation}
    which proves \eqref{eq:drei}.
    Furthermore, since $ \delta B_t^n$ is a square-integrable martingale, 
		Doob's maximal inequality \cite[Thm~2.2]{GawMan10} yields
    \begin{equation*}
        \mathbb E \Big[ \sup_{t \in [0,T]} \left\lVert \delta B_t^n \right\rVert^2\Big] 
				\le 4 \, \mathbb E \left[ \left\lVert \delta B_T^n \right\Vert^2 \right] 
				= 4 \, \big(a_T^2-1\big) \, \mathrm{tr}(Q-Q_n).
    \end{equation*}
    This proves
    \begin{equation*}
        \mathbb E \Big[\sup_{t \in [0,T]}\lVert X_t -X_t^n \rVert^2 \Big]  
				\le 2 \, \mathbb E \left[ \lVert X_0-X_0^n \rVert^2 \right]+ 8 \, \big(a_T^2-1\big) \, \mathrm{tr}(Q-Q_n).
    \end{equation*} 
    For any $\omega \in \Omega$, we have $\Vert X_0(\omega) - X_0^n(\omega) \Vert^2 = \sum_{k=n+1}^{\infty} \vert \langle X_0(\omega), e_k \rangle \vert^2$ which goes to zero. Thus, the assertion follows by dominated convergence since $\mathbb E \left[ \|X_0\|^2 \right] < \infty$. 
\end{proof}

Next, we estimate the approximation error with respect to the Wasserstein-2 distance~\cite[Def~6.1]{villani2009optimal} that for two measures $\mu,\nu$ on a separable metric space $(K,d)$ is defined by 
\begin{equation*}
W_2(\mu,\nu) 
\coloneqq \Big (\min_{\pi\in\Pi(\mu,\nu)} \int d(x,y)^2\, d\pi(x,y)\Big)^{1/2} 
= \Big (\min_{X \sim \mu, Y \sim \nu} \mathbb{E}(d(X,Y)^2)\Big)^{1/2},
\end{equation*}
where $\Pi(\mu,\nu)$ denotes the set of measures on $K \times K$ with marginals $\mu$ and $\nu$, i.e. 
$\pi(B \times K) = \mu(B)$, $\pi(K\times B) = \nu(B)$ for all $B \in \mathcal B(K)$, where $\mathcal B(K)$ is the Borel $\sigma$-algebra of $K$.
The Wasserstein-2 distance is a metric on the space of probability measures on $K$ having finite second moments. The following corollary, which is an immediate consequence of \cref{forConv},
establishes convergence of the forward process with respect to the Wasserstein-2 distance. That is, we consider the Wasserstein-2 distance between the probability measures of the forward process and the approximated forward process on $(H, \|\cdot\|)$.

\begin{corollary}\label{cor:f}
Let $X_0 \in H$ be a random variable fulfilling
$\mathbb E \left[\lVert X_0\rVert^2\right]<\infty$.
Let $(X_t)_{t \in [0,T]}$ denote the solution of the SDE \eqref{eq:SDE} 
and $(X_t^n)_{t \in [0,T]}$ the solution of the approximate problem \eqref{finiteforward} with corresponding laws
$\mathbb P_{X_t}$ and $\mathbb P_{\smash{X_t^n}}$.
Then it holds for $t \in [0,T]$ that
\begin{equation*}
\begin{aligned}
        W_2^2(\mathbb{P}_{\smash{X_t^n}}, \mathbb{P}_{X_t}) 
		& \le         \mathbb E \Big[\lVert X_t -X_t^n \rVert^2 \Big] \\
  & \le
		2 \,a_t^{-2} \, \mathbb E \left[ \lVert X_0-X_0^n \rVert^2 \right]+ 2\,  (1-a_t^{-2}) \, \mathrm{tr}(Q-Q_n) \to 0 \text{ as } n\to \infty.
\end{aligned}
\end{equation*}
\end{corollary}

\subsection{Reverse Process}\label{sub:discr}

In the finite-dimensional setting, the reverse process $Y_t\coloneqq X_{T-t}$, $t \in [0,T]$, 
is the solution of a reverse SDE whose drift depends on the score function; see our discussion in \cref{sec:score_finite}. 
This score is defined via the Lebesgue density of the finite-dimensional forward process. Since Lebesgue densities no longer exist in the infinite-dimensional Hilbert space setting, the definition of the score is not straightforward. 
Some results on time reversibility of infinite-dimensional diffusion processes are discussed in \cite{FOLLMER198659,millet1989time,jakiw_diffusion}. 

In this paper, we circumvent the issue of time-reversal of the infinite-dimensional 
forward process by discretizing the forward process 
and working only with time-reversals of the discretizations. 
As discussed in \cref{sec:score_finite}, 
the time reversal of a diffusion process in $\mathbb R^n$ is well understood. 
Thus, it is natural to relate the discretized forward SDE \eqref{finiteforward} 
to an equivalent SDE in $\mathbb R^n$ via coordinate-vectors 
and then transfer the reverse SDE in $\mathbb R^n$ back to an SDE in $H_n$. 
To this end, we define for $n \in \mathbb N$, the isometric isomorphism
\begin{equation*}
    \iota_n \colon H_n \to \mathbb R^n, \qquad \iota_n(x)\coloneqq (\langle x,e_k\rangle)_{k=1}^n.
\end{equation*}
Using the coordinate process $\hat{X}^n_t \coloneqq \iota_n(X^n_t)$, we can formulate a reverse SDE 
to the restricted forward SDE \eqref{finiteforward} under mild assumptions.

\begin{theorem}[Discretized reverse process]  \label{lem:reverse} 
Let $W^Q$ be a $Q$-Wiener process with injective trace class operator $Q$ 
having the Karhunen–Lo\`eve decomposition \eqref{eq:q},
and let $W^{Q_n}$ be its truncation \eqref{eq:n}. 
Let $X_t^n$ be the solution of the approximate problem \eqref{finiteforward}.
If the random variable $\hat{X}_0^n \in \mathbb R^n$ admits a Lebesgue density $p_0^n$, then $\hat{X}^n_t \in \mathbb R^n$ 
admits a Lebesgue density $p_t^n$ and
it holds for $t \in (0,T]$ that
     \begin{equation}\label{eq:sco}
        \nabla \log p_t^n(z)
				=-\frac{a_t}{(a_t^2-1)}\Big(\frac{(\lambda_k)^{-1} \, 
				\left( p^n_0* (z_k b^n_t) \right)(a_t \, z)}{(p^n_0* b^n_t)(a_t z)}\Big)_{k=1}^n, 
    \end{equation}
    where $*$ denotes the convolution and
		 \begin{equation}\label{bt}
		b_t^n(z) \coloneqq e^{-\frac{(a_t^2-1)^{-1}}{2} \, z^\intercal \, \mathrm{diag}(\lambda_1,...,\lambda_n)^{-1} z}, \quad t \in (0,T].
		\end{equation}
		In particular, the discretized score
    \begin{equation}\label{eq:score}
        f^n(t,\cdot)\coloneqq \alpha_{T-t} \, Q_n \big(\iota_n^{-1} \circ \nabla \log p_{T-t}^n \circ \iota_n \big), \qquad t \in [0,T), 
    \end{equation}
    is well-defined, and Lipschitz continuous in the second variable uniformly on 
		$[0,T-\delta]\times \{x \in H_n \mid \lVert x\rVert\le N\}$ for all $\delta>0$ and $N \in \mathbb N$.
		The discretized reverse SDE
    \begin{equation}\label{reverse_disc}
        {\rm d} Y^n_t = \Big(\frac{1}{2} \alpha_{T-t} \, Y_t^n +f^n(t,Y^n_t) \Big) \, {\rm d} t + \sqrt{\alpha_{T-t}} 
				\, {\rm d} \hat{W}^{Q_n}_t,
    \end{equation}
    with initial condition $Y_0^n \sim \mathbb P_{X_T^n}$ admits a weak solution 
	 $((\hat{\Omega}, \hat{\mathcal F},\hat{\mathbb F}, \hat{\mathbb P}),(Y^n_t)_{t \in [0,T)},\hat{W}^{Q_n})$
		such that $(Y^n_t)_{t \in [0,T)}$ is equal to $(X_{T-t}^n)_{t \in [0,T)}$ in distribution. Further, $(\hat{\Omega}, \hat{\mathcal F},\hat{\mathbb F}, \hat{\mathbb P})$ can be chosen as an extension of $(\Omega,\mathcal F, \mathbb F,\mathbb P)$, where $\mathbb F$ denotes the completion of the natural filtration of $(X_{T-t}^n)_{t\in[0,T)}$, and $(Y^n_t)_{t \in [0,T)}$ is the canonical extension of $(X_{T-t}^n)_{t \in [0,T)}$ to this extended probability space, cf. \cite[Def.~3.13]{scheutzowWT4}.
  Furthermore, solutions of \eqref{reverse_disc} are unique in law, i.e., any weak solution with initial distribution $\mathbb P_{X_T^n}$ is equal to $(X_{T-t}^n)_{t \in [0,T)}$ in distribution. 
\end{theorem}

The following proof contains the main ideas. For convenience, we have further detailed some parts in an extended
proof in the the supplementary material. 

\begin{proof} 
1.  First, we show that $p^n_t$ has the form \eqref{eq:sco} which then implies that $f^n$ in \eqref{eq:score} is well-defined
with the desired properties.
We first check that $\hat{X}_t^n$ has the density
    \begin{align*}
        p^n_t(z) \coloneqq  \frac{a_t^n(\lambda_1...\, \lambda_n)^{-1/2}}{\left(2\pi(a_t^2-1) \right)^{n/2}} \, 
				(p^n_0*b^n_t)(a_t \, z).
    \end{align*}
Then, using regularity properties, we conclude that $f^n$ as in \eqref{eq:score} is well-defined and Lipschitz continuous in the second variable uniformly on $[0,T-\delta]\times \{x \in H_n\mid \lVert x\rVert\le N\}$.    \\[1ex] 
 2.   Next, we transfer \eqref{finiteforward} to an SDE on $\mathbb R^n$ 
to obtain a reverse equation. Let $\beta_k(t)$, $k \in \mathbb N$ be as in \eqref{eq:q}.
Then 
$\bm{\beta}^n_t\coloneqq (\beta_k(t))_{k=1}^n$ 
is an $n$-dimensional Brownian motion and  $(\hat{X}^n_t)_{t \in [0,T]}$ is the unique strong solution of 
       $ {\rm d} \hat{X}^n_t = - \frac{1}{2} \alpha_t \, \hat{X}^n_t \, {\rm d} t + \sqrt{\alpha_t} \, \mathrm{diag}(\sqrt{\lambda_1},...,\sqrt{\lambda_n}) \, {\rm d} \bm{\beta}^n_t$    with initial value $\hat{X}_0^n$. We set $\Lambda = \mathrm{diag}(\lambda_1,...,\lambda_k).$ We conclude that$(\hat{X}^n_{T-t})_{t \in [0,T)}$ 
		is a weak solution of 
  \begin{equation}
  \label{eq:revF}
  \begin{aligned}
        {\rm d} \hat{Y}^n_t 
				=
				\Big( \frac{1}{2} \alpha_{T-t} \hat{Y}^n_t + \alpha_{T-t} \ \Lambda \nabla \log p^n_{T-t}(\hat{Y}^n_t)\Big) \, {\rm d}t + \sqrt{\alpha_{T-t}} \,\sqrt{\Lambda} \, {\rm d} {\bm{\hat{\beta}}}^n_t
    \end{aligned}
    \end{equation}
    with initial condition $\hat{Y}^n_0 \sim \mathbb P_{\hat{X}^n_T}$, by \cite[Thm~2.1]{haussmann1986time}. Thus, $(\hat{Y}^n_t)_{t \in [0,T)}$ is equal to 
		$(\hat{X}^n_{T-t})_{t \in [0,T)}$ in distribution. 
		We set $Y^n_t \coloneqq \iota_n^{-1}(\hat{Y}_t^n).$ 
		Clearly, $(Y_t^n)_{t \in [0,T)}$ and $(X_{T-t}^n)_{t \in [0,T)}$ 
		are equal in distribution. 
		Since $((\hat{\Omega}, \hat{\mathcal F}, \hat{\mathbb F}, \hat{\mathbb P}),(\hat{Y}^n_t)_{t \in [0,T)},\hat{\bm{\beta}}^n)$ 
		is a weak solution of \eqref{eq:revF}, the process $(Y^n_t)_{t \in [0,T)}$ 
		has continuous sample-paths and is adapted to $\hat{\mathbb F}$. 
		Further, since $Y^n_t$ fulfills \eqref{sol} for the setting  \eqref{eq:revF}, we have 
		$\hat{\mathbb P}$-almost surely for all $t \in [0,T)$ that  $Y^n_t$ is equal to

\begin{align*}
        Y^n_0 + \int_0^t \frac{1}{2} \alpha_{T-s} \, Y^n_s 
				+ \underbrace{\alpha_{T-s} \, Q_n \big(\iota_n^{-1} \circ \nabla \log p_{T-t}^n \circ \iota_n\big)(Y_s^n)}_{=f^n(s,Y^n_s)} \, {\rm d} s 
				+ \int_0^t \sqrt{\alpha_{T-s}}  \, {\rm d}  \hat{W}^{Q_n}_s
   \end{align*}
Thus  $((\hat{\Omega}, \hat{\mathcal F}, \hat{\mathbb F}, \hat{\mathbb P}),(Y^n_t)_{t \in [0,T)},\hat{W}^{Q_n})$ 
		is a weak solution of \eqref{reverse_disc} with $Y_0^n \sim \mathbb P_{X^n_T}$. 

We can analogously show that $((\hat{\Omega}, \hat{\mathcal F}, \hat{\mathbb F}, \hat{\mathbb P}),(\hat{Y}^n_t)_{t \in [0,T)},\hat{\bm{\beta}}^n)$ with $\hat{Y}^n_t\coloneqq \iota_n(Y_t^n)$ and $\hat{\bm \beta}^n_t \coloneqq \big( (\sqrt{\lambda_k})^{-1} \langle\hat{W}^{Q_n}_t,e_k\rangle \big)_{k=1}^n$ yields a weak solution of \eqref{eq:revF}. Thus, showing uniqueness in law of \eqref{eq:revF} is sufficient to prove uniqueness in law of \eqref{reverse_disc}. Since the drift and diffusion coefficients of \eqref{eq:revF} are Lipschitz continuous in the second variable on sets of the form $[0,T-\delta] \times \{z \in \mathbb R^n\mid \lVert z\rVert \le N\}$, we know by \cite[Lem.~12.4]{russo22} and \cite[Definition~12.1]{russo22} that pathwise uniqueness holds on $[0,T-\delta]$. A limiting argument $\delta \to 0$ yields pathwise uniqueness on $[0,T)$.
\end{proof}

As a consequence of the theorem, we obtain the following corollary.

\begin{corollary}\label{cor:r}
Let $X_0 \in H$ be a random variable fulfilling $\mathbb E \left[\lVert X_0\rVert^2\right]<\infty$. 
Further, let $X_t^n$ be the solution of the approximate problem \eqref{finiteforward}, and let
$\hat{X}_0^n \in \mathbb R^n$ admit a Lebesgue density. Then  we have for
$t \in [0,T)$ that
\begin{equation*}
    W_2^2( \mathbb{P}_{Y_t^n}, \mathbb{P}_{X_{T-t}}) 
		\le 
		2 \, a_{T-t}^{-2} \, \mathbb E \left[ \lVert X_0-X_0^n \rVert^2 \right]
		+ 2\,  (1-a_{T-t}^{-2}) \, \mathrm{tr}(Q-Q_n) \to 0 \text{ as } n\to \infty.
\end{equation*}
\end{corollary}  

\begin{proof} The assertion follows from the triangle inequality
$$
W_2(\mathbb P_{Y_t^n},\mathbb P_{X_{T-t}}) 
\le
W_2(\mathbb P_{Y_t^n},\mathbb P_{X_{T-t}^n}) + W_2(\mathbb P_{X_{T-t}^n},\mathbb P_{X_{T-t}})
$$
\cref{lem:reverse} and \cref{cor:f}.
\end{proof}

Besides the discretized reverse SDE \eqref{reverse_disc}, we also consider the case of approximation errors in the reverse drift, i.e., 
\begin{equation}     \label{reverse_learned}
    {\rm d} \widetilde{Y}^n_t = \Big(\frac{1}{2}\alpha_{T-t} \, \widetilde{Y}^n_t+ \tilde{f}^n(t,\widetilde{Y}^n_t)\Big) \, {\rm d} t 
		+ \sqrt{\alpha_{T-t}} \,  {\rm d} \hat{W}^{Q_n}_t, 
\end{equation}
where $\tilde{f}^n$ could be a neural network, which is learned via the DSM loss, 
that is supposed to approximate the forward path score $f^n(t,X^n_{T-t})$ in \eqref{eq:score}. 
In contrast to the discretized reverse SDE, the approximate reverse SDE is started at the latent distribution 
$ \widetilde{Y}^n_0 \sim \mathcal{N}(0,Q_n)$. 

We now prove that when the denoising score matching loss in~\eqref{eq:LDSM} decreases, the distribution of the approximated reverse equation also goes to the distribution of the time-changed forward process. Similar to \cite[Theorem 2]{jakiw_diffusion}, 
we bound the \emph{Wasserstein distance between path measures obtained by our SDE}. However, we include a time schedule, consider the finite-dimensional case, and explicitly bound the distance for all intermediate times. More precisely, we consider a path $Z=(Z_t)_{t \in[0,t_0]}$ as a $\mathcal C([0,t_0],H)$-valued random variable on $(\Omega,\mathcal F,\mathbb P)$ with 
$$\Vert Z \Vert_{L^2(\Omega,\mathcal F,\mathbb P; \mathcal C([0,t_0],H))}^2 \coloneqq  \mathbb{E}[\sup_{t \in [0,t_0]} \Vert Z_t \Vert^2].$$ 
Then we can examine the Wasserstein distance of two path measures of the form $\mathbb{P}_{(Z_t)_{t \in [0,t_0]}}$ by noting that the space $\mathcal C([0,t_0],H)$ equipped with the uniform norm is a separable metric space. For notational simplicity, we denote a path measure $\mathbb{P}_{(Z_t)_{t \in [0,t_0]}}$ by $\mathbb{P}_{(Z_t)_{t_0}}$.
The proof combines classical concepts in stochastic analysis such as, e.g.,  Gronwall's lemma and 
Jensen's inequality, as well as results from
\cite[Thm 18.11]{schilling2014brownian},\cite[Prop. 7.18]{da2014stochastic}, \cite[Thm 1.2]{carmona2016lectures}, and \cite[Prop. C.1]{ichiba2020relative}.
We give a streamlined version of the proof and refer to the supplements for details.

\begin{theorem}\label{thm:main}
Let $X_0 \in H$ be a random variable fulfilling $\mathbb E \left[\lVert X_0\rVert^2\right]<\infty$ and let $Q$ be an injective trace class operator on $H$. For $n \in \mathbb N$, let $(X_t^n)_{t \in [0,T]}$ denote the solution of the approximate problem \eqref{finiteforward} and let $\hat{X}_0^n \in \mathbb R^n$ admit a Lebesgue density. For $t_0<T$, assume that for any $s \in [0,t_0]$ the approximated score $\tilde{f}^n$ restricted to $[0,s]\times H_n$ 
is continuous and Lipschitz continuous in the second variable with Lipschitz constant $L_s^n$ and that
\begin{equation}\label{xxx}
\mathbb{E}\left[\int_0^{t_0} \big\Vert f^n(s,X_{T-s}^n) - \tilde{f}^n(s,X_{T-s}^n) \big\Vert^2 \, {\rm d}s \right] \leq \varepsilon
\end{equation} 
is satisfied.
By $\mathbb{P}_{\smash{(\widetilde{Y}}_t^n)_{t_0}}$ denote the path measure induced by \eqref{reverse_learned} with initial distribution $\mathcal{N}(0,Q_n)$. Then it holds 
\begin{equation}\label{xxx1}
W_2^2(\mathbb{P}_{\smash{(X_{T-t}^n)_{t_0}}}, \mathbb{P}_{\smash{(\widetilde{Y}}_{t}^n)_{t_0}}) 
\leq 
12
				\left(t_0 \, \varepsilon \, e^{4 \, \xi(t_0)   t_0^2} + e^{3 \, \xi(t_0)  t_0^2 } \, W_2^2\left(\mathbb P_{X^n_T},\mathcal{N}(0,Q_n) \right)
				\right),
\end{equation}
where $\xi(t_0) \coloneqq \sup_{s \in [0,t_0]} \left(\frac{\alpha_{T-s}^2}{4} +  (L_s^n)^2\right) $.
\end{theorem}

\begin{proof}
1.  First, we conclude by Lipschitz continuity of $\tilde{f}^n$, for any initial condition and $Q_n$-Wiener process independent of the initial condition the SDE \eqref{reverse_learned} has a unique strong solution on $[0,t_0]$. We use $\mathbb P_{\smash{(\breve{Y}_t^n)_{t_0}}}$ to denote the path measure induced by \eqref{reverse_learned}, but with initial distribution $\mathbb P_{X_T^n}$.
    By the triangle inequality, we obtain
  \begin{equation} \begin{aligned}
				W_2^2(\mathbb{P}_{\smash{(X_{T-t}^n)_{t_0}}}, \mathbb{P}_{\smash{(\widetilde{Y}}_{t}^n)_{t_0}}) 
				&\leq  
				4	W_2^2(\mathbb{P}_{\smash{(X_{T-t}^n)_{t_0}}}, \mathbb{P}_{\smash{(Y_{t}^n)_{t_0}}}) 
				\\ &+ 2 W_2^2(\mathbb{P}_{\smash{(Y_{t}^n)_{t_0}}}, \mathbb{P}_{\smash{(\breve{Y}_{t}^n})_{t_0}})
				+ 4 W_2^2(\mathbb P_{\smash{(\breve{Y}_t^n)_{t_0}}},\mathbb P_{\smash{(\widetilde{Y}}_t^n)_{t_0}}).
				 \end{aligned}\end{equation}
    By \cref{lem:reverse}, the first term is zero. 
    \\
    2.
		Next, we bound the third term. Let $\breve{Y}^n_0 $ and $\widetilde{Y}^n_0$ be any realizations of $\mathbb P_{X^n_T}$ and $\mathcal{N}(0,Q_n)$, respectively, that are defined on the same probability space. For some driving $Q_n$-Wiener process independent of  $\breve{Y}^n_0 $ and $\widetilde{Y}^n_0$ and defined on the same probability space (or possibly an extension of it), let $(\breve{Y}_t^n)_{t \in [0,t_0]}$ and $(\widetilde{Y}^n_t)_{t \in [0,t_0]}$ be the unique strong solutions of \eqref{reverse_learned} started from  $\breve{Y}^n_0 $ and $\widetilde{Y}^n_0$, respectively. By \eqref{sol} and using Jensen's inequality, we obtain for $s \in [0,t_0]$,
  \begin{equation} \begin{aligned}
    \big\lVert \breve{Y}^n_s - \widetilde{Y}^n_s \big\rVert^2 
		& \leq 
		 3 \Big( \big\lVert \breve{Y}_0^n - \widetilde{Y}^n_0 \big\rVert^2 +
		 \Big\lVert \int_0^s  \tfrac{1}{2} \alpha_{T-r} \, (\breve{Y}_r^n - \tilde{Y}_r^n)  {\rm d} r \Big\rVert^2 \\
		&+  \Big\lVert\int_0^s  \tilde{f}^n(r,\breve{Y}^n_r)  - \tilde{f}^n(r,\widetilde{Y}^n_r) \, {\rm d} r \Big\rVert^2 \Big)
		\\
    & \le 3 \, \big\lVert \breve{Y}_0^n - \widetilde{Y}^n_0 \big\rVert^2
		+ 
		3 t_0 \sup_{r \in [0,t_0]} \Big(\frac{\alpha_{T-r}^2}{4}+  (L_r^n)^2 \Big)  
		\int_0^s \big\lVert  \breve{Y}_r^n - \tilde{Y}_r^n\big\rVert^2 \, {\rm d} r.		
 \end{aligned}\end{equation} 
Setting $h_1(t) \coloneqq 
\mathbb E \big[ \sup_{s \in [0,t]} \big\lVert \breve{Y}^n_s-\widetilde{Y}^n_s\big\rVert^2\big]$ for $t \in [0,t_0]$, we consequently obtain
  \begin{equation} \begin{aligned}
h_1(t) 
&\leq
3 \mathbb E \big[\big\lVert \breve{Y}_0^n - \widetilde{Y}^n_0 \big\rVert^2 \big]
+ 
3 t_0 \xi(t_0)
\int_0^t  h_1(r)\, {\rm d} r,
   \end{aligned}\end{equation}
by Fubini's theorem. By \cref{var}, $\mathbb E \big[ \sup_{t \in [T-t_0,T]} \lVert X^n_t\rVert^2\big]<\infty$. In particular, $\mathbb E [\lVert \breve{Y}^n_0\rVert^2]=\mathbb E [\lVert X^n_T\rVert^2]<\infty$. $\mathbb E [\lVert \smash{\widetilde{Y}}^n_0 \rVert^2]=\mathrm{tr}(Q_n)<\infty$. 
Since $h_1 \in L^\infty([0,t_0])$, we apply Gronwall's lemma to $h_1$ and obtain for all $t \in [0,t_0]$ that
    \begin{equation*}
     h_1(t) 
				\le  
				3\mathbb E \big[\big\lVert \breve{Y}^n_0-\widetilde{Y}^n_0 \big\rVert^2 \big] \, 
				e^{3 \xi(t_0) t_0 \, t }.
    \end{equation*}
    In particular, the Wasserstein-2 distance can be estimated by
  \begin{equation} \begin{aligned}
         W_2^2( \mathbb P_{\smash{(\breve{Y}^n_t)_{t_0}}},\mathbb P_{\smash{(\widetilde{Y}^n_t)_{t_0}}} )  
				&\leq
				\mathbb E \Big[ \sup_{s \in [0,t_0]}  \big\lVert \breve{Y}_s^n - \widetilde{Y}^n_s \big\rVert^2 \Big] = h_1(t_0)			
				\\
				&= 3 \,e^{3 \xi(t_0)  t_0^2 } \, W_2^2\left(\mathbb P_{X^n_T},\mathcal{N}(0,Q_n) \right),
   \end{aligned}\end{equation}
   where we used that the previous estimates are valid for an arbitrary realization of  $\breve{Y}^n_0 $ and $\widetilde{Y}^n_0$.
   \\
    3. 			We define $(\breve{Y}^n_t)_{t \in [0,t_0]}$ 
		as the strong solution of the approximate reverse  \eqref{reverse_learned} with the initial condition 
		$\breve{Y}^n_0=Y^n_0$ and the same driving noise $\hat{W}^{Q_n}$.
		As before, we can bound the difference as follows for $s \in [0,t_0]$:
 \begin{equation}\begin{aligned}
\big\lVert Y^n_s - \breve{Y}^n_s \big\rVert^2 
&\leq 
\Big\lVert Y_0^n - \breve{Y}^n_0 + 
\int_0^s \frac{1}{2} \alpha_{T-r} \, (Y_r^n - \breve{Y}_r^n) \, {\rm d} r 
+ \int_0^s f^n(r,Y^n_r)  - \tilde{f}^n(r,\breve{Y}^n_r) \, {\rm d} r \Big\rVert^2\\
 &\leq 
2 t_0  
\sup_{r \in [0,t_0]} \frac{\alpha_{T-r}^2}{4} \, \int_0^s \big\lVert Y_r^n - \breve{Y}_r^n \big\rVert^2 \, {\rm d} r 
+ 2 t_0 \, \int_0^s \big\lVert f^n(r,Y^n_r)  - \tilde{f}^n(r,\breve{Y}^n_r) \big\rVert^2\, {\rm d} r .
 \end{aligned}\end{equation}
For the second term, we apply again the triangle inequality and obtain 
 \begin{equation}\begin{aligned}
    & \big\lVert f^n(r,Y^n_r)  - \tilde{f}^n(r,\breve{Y}^n_r)\big\rVert^2 \le 2 \big\lVert f^n(r,Y^n_r)  - \tilde{f}^n(r,Y^n_r)\big\rVert^2 + 2 \sup_{r \in [0,t_0]} (L_r^n)^2 \big\lVert Y_r^n - \breve{Y}^n_r \big\rVert^2.
 \end{aligned}\end{equation}
Since $(Y^n_t)_{t \in [0,t_0]}$ is an extension of $(X^n_{T-t})_{t \in [0,t_0]}$, we obtain from assumption \eqref{xxx} that
\begin{equation}
    \mathbb E \Big[ \sup_{t \in [0,s]} \int_0^s \big\lVert f^n(r,Y^n_r)-\tilde{f}^n(r,Y^n_r)\big\rVert^2 \, {\rm d}r \Big] \le \mathbb E \Big[\int_0^{t_0} \big\lVert f^n(r,Y^n_r)-\tilde{f}^n(r,Y^n_r)\big\rVert^2 \, {\rm d}r \Big]\le \varepsilon.
\end{equation}
We set $ h_2(t) \coloneqq 
		\mathbb E \big[ \sup_{s \in [0,t]} \big\lVert Y^n_s -\breve{Y}_s^n \big\rVert^2\big]$ for $t \in [0,t_0]$ and by Gronwall's lemma, it yields
\begin{equation}
W_2^2(\mathbb{P}_{\smash{(Y_{t}^n)_{t_0}}}, \mathbb{P}_{\smash{(\breve{Y}_{t}^n)_{t_0}}})
\leq h_2(t_0)\le
4 \, t_0 \varepsilon \, e^{4 \xi(t_0)   t_0^2}.
\end{equation}
In summary, we obtain \eqref{xxx1}.
\end{proof}

    Note that in the proof of \cref{thm:main} we use the assumption that $\hat{X}^n_0$ admits a Lebesgue density $p^n_0$ only to apply \cref{lem:reverse}. Thereby, we establish the existence of a weak solution of the discretized reverse SDE \eqref{reverse_disc} that is equal to the time-changed discretized forward process in distribution as well as uniqueness in law of solutions of \eqref{reverse_disc}. In \cref{lem:reverse}, the assumption that $\hat{X}^n_0$ admits a Lebesgue density can be relaxed to $\hat{X}^n_t$ admitting a Lebesgue density $p^n_t$ for all times $t>0$ such that $p^n_t$ and $\nabla \log p^n_t$ are sufficiently smooth \footnote{For all $\delta>0$ and all bounded open sets $U \subset \mathbb R^n$, $k=1,...,n$, $\int_{\delta}^T \int_U \lvert p_t^n(z)\rvert^2+\alpha_t \lambda_k\, \lvert  \partial_k p_t^n(z) \rvert^2 \, {\rm d}z \, {\rm d}t<\infty,$ is satisfied and $\nabla \log p^n_t$ is well-defined and Lipschitz continuous in the second variable uniformly on $[\delta,T] \times \{z \in \mathbb R^n \mid \lVert z\rVert \le N\}$ for all $\delta>0$ and $N \in \mathbb N$.} to guarantee well-defined reverse SDEs. This weaker assumption can be met even if the data distribution is degenerate, as is illustrated in the following example where we consider Gaussian data distributions, which can be degenerate.

Our approximation result in \cref{thm:main} requires bounded Lipschitz constants of the score for fixed $n \in \mathbb N$. 
For the well-posedness of the infinite-dimensional diffusion model, and to enable multilevel training, we do not want the sequence of Lipschitz constants to explode if $n \rightarrow \infty$.
Therefore, we consider the following example.

\begin{example}
\label{example_gaussians}
We consider a Gaussian data distribution, i.e., 
$\mathbb{P}_{data} = \mathcal{N}(\mu,P)$.  We assume that $Q$ is injective. 
Then the processes $\hat{B}^n_t$ and $\hat{X}^n_{ 0}$ are independent Gaussians.
 \begin{equation}\begin{aligned}
    & \hat{B}^n_t \sim \mathcal{N} \big( 0, ({a_t^2}-1)\, \hat{Q}_n \big) \text{ with } \hat{Q}_n \coloneqq \mathrm{diag}(\lambda_1,...,\lambda_n) \text{ and }\\
    & \hat{X}^n_{ 0} \sim \mathcal{N} \big(\hat{\mu}^n, \hat{P}^n \big) \text{ with } \hat{\mu}^n
		\coloneqq \iota_n\big(\sum_{k=1}^n \langle \mu,e_k \rangle \, e_k\big) \text{ and }\hat{P}_n 
		\coloneqq (\langle P e_i,e_j\rangle)_{i,j=1}^n,
 \end{aligned}\end{equation}
 By \eqref{eq:finvar} we have $ \hat{X}^n_t = a_t^{-1} \,  ( \hat{X}^n_0 + \hat{B}^n_t )$ 
and therefore $\hat{X}^n_t = \mathcal{N} \big(  a_t^{-1}\,  \hat{\mu}^n, a_t^{-2} \, \hat{P}_n+ (1-a_t^{-2})\,\hat{Q}_n\big)$.
In particular, the coordinate process $\hat{X}^n_t$ has a density $p^n_t$ for $t \in (0,T]$ that satisfies
 \begin{equation}\begin{aligned}
    \hat{Q}_n \, \nabla \log p^n_t (z) 
		&= -(1-a_t^{-2})^{-1} \hat{Q}_n \big(\tfrac{a_t^{-2}}{1-a_t^{-2}} \, \hat{P}_n + \hat{Q}_n\big)^{-1}(z-a_t^{-1} \, \hat{\mu}^n).
 \end{aligned}\end{equation}
We assume that the spectral norms of
$
A_n \coloneqq \hat{Q}_n \big(\tfrac{a_t^{-2}}{1-a_t^{-2}} \, \hat{P}_n + \hat{Q}_n\big)^{-1} 
$
are uniformly bounded in $n$ by some constant $C>0$. 
This is equivalent to the condition that the smallest eigenvalues of 
\begin{equation} \label{hi}
(A_n^{\tiny{T}} A_n)^{-1} \coloneqq \hat{Q}_n^{-1} \big(\tfrac{a_t^{-2}}{1-a_t^{-2}} \, \hat{P}_n + \hat{Q}_n\big)^2 \hat{Q}_n^{-1}
\end{equation}
are bounded from below by $C^{-2}$.
Note that while the eigenvalues of $A_n^{-1}$ coincide with those of 
$\big(\tfrac{a_t^{-2}}{1-a_t^{-2}} \, \hat{Q}_n^{-1/2} \hat{P}_n  \hat{Q}_n^{-1/2} + I_n\big)$
and are therefore larger or equal than one, this is not necessarily longer true for $(A_n^{\tiny{T}} A_n)^{-1}$.
A sufficient condition for \eqref{hi} to hold is that $\hat{P}_n$ and $\hat{Q}_n$ commute, which is the case if and only if $(\hat{P}_n)_{i,j} = 0$ if $\lambda_i \neq \lambda_j$. If $\hat{Q}_n$ has an eigenvalue with multiplicity larger than one, this can also be fulfilled by nondiagonal $\hat{P}_n$. 
It is not clear if commutation is a necessary condition for \eqref{hi} to hold.

Then we have for fixed $t_0 <T$ that
\begin{equation*}
   L^n_{t_0} \coloneqq  \sup_{t \in [T-t_0,T]} \big\lVert -\alpha_{t} \, (1-a_t^{-2})^{-1} \hat{Q}_n \big(\tfrac{a_t^{-2}}{1-a_t^{-2}} \, \hat{P}_n + \hat{Q}_n\big)^{-1}\big\rVert  \le \frac{C \, \sup_{t \in [T-t_0,T]} \alpha_t }{1-e^{-\int_0^{T-{t_0}} \alpha_r \,  {\rm d} r}},
\end{equation*}
where $\lVert \cdot \rVert$ denotes the spectral norm on $\mathbb R^{n,n}$. 
In particular, 
\begin{equation*}
    f^n(t,y)=\alpha_{T-t}\, Q_n \big( \iota_n^{-1} \circ \nabla \log p^n_{T-t} \circ \iota_n\big)(y)=\alpha_{T-t}\, \hat{Q}_n \nabla \log p^n_{T-t}(\iota_n(y))
\end{equation*} 
is Lipschitz continuous in the second variable uniformly on $[0,t_0]\times H_n$ with the Lipschitz constant $L_{t_0}^n$ bounded independently of $n$ for any $t_0 \in [0,T)$.
\end{example}

The next theorem shows the viability of multilevel training under the assumption that the score path random variable $f^n(t,X_{T-t}^n)$ converges. In this case, the contribution of the $e_k$-component of the true reverse process becomes insignificant in the limit for large $k$. Thus, the learned score at level $n$ provides a good approximation for higher levels, and choosing the approximate score $\tilde{f}^m\coloneqq \tilde{f}^n$ at level $m\geq n$ is sufficient for the error \eqref{xxx} to converge to zero for large $n$. Note that the approximate score, $\tilde{f}^n$, is learned from data in $H_n$ and therefore can not depend on any $e_k$-components with $k\geq n$ in a meaningful way. Because of this, we make the reasonable assumption that $\tilde{f}^n(t,x)=\tilde{f}^n(t,P^n x)$ for all $x \in H$, where $P^n x$ denotes the projection of $x$ to $H_n$. In our experiments, this assumption corresponds to the cut-off of high-frequency Fourier modes as described in \cref{sub:latent}.

\begin{theorem}\label{thm:multilevel}
Let $t_0 <T$ and assume that $F^n_t \coloneqq f^n(t,{X_{T-t}^n})$ converges to some process $F_t$ in $L^2(\Omega, \mathcal F,\mathbb P;L^2([0,t_0];H))$. For fixed $n \in \mathbb N$, let $\tilde{\varepsilon}_n$ denote the maximal score path approximation error above level $n$ and let $\varepsilon_n$ be the training loss at level $n$, i.e., we set
\begin{equation*}
    \tilde{\varepsilon}_n \coloneqq \sup_{m\geq n} \lVert F^m - F \rVert_{L^2(\Omega\times [0,t_0]; H)}^2 \quad \text{and} \quad \varepsilon_n \coloneqq \lVert F^n-\tilde{f}^n(\cdot, X^n_{T-\cdot})\rVert_{L^2(\Omega \times [0,t_0];H)}^2.
\end{equation*}
If the approximate score $\tilde{f}^n$ at level $n$ is Lipschitz continuous on $[0,s]\times H_n$ with Lipschitz constant $L^n_s$ and satisfies $\tilde{f}^n(s,x)=\tilde{f}^n(s,P^n x)$ for all $x \in H$ and $s\le t_0$, then, for all $m \geq n$, we have
\begin{equation*}
    W_2^2(\mathbb{P}_{(X^m_{T-t})_{t_0}},\mathbb{P}_{(\tilde{Y}^m_t)_{t_0}}) \le 12(4 \tilde{\varepsilon}_n+2\varepsilon_n) t_0 e^{4 \xi_n(t_0) t_0^2}+ 12 e^{3 \xi_n(t_0) t_0^2} W^2_2 (\mathbb{P}_{X^m_T}, \mathcal{N}(0,Q_m)),
\end{equation*}
where $\xi_n(t_0)\coloneqq \sup_{s \in [0,t_0]} \big( \tfrac{\alpha_{T-s}^2}{4}+(L^n_s)^2\big)$ as before. Here $\tilde{Y}^m$ denotes the approximate reverse process at level $m$ for the choice $\tilde{f}^m\coloneqq \tilde{f}^n$. In particular, the right-hand side converges to zero uniformly in $m\geq n$ whenever the Lipschitz constant $L^n_{t_0}$ is bounded uniformly in $n$, and the training loss $\varepsilon_n$ converges to zero. 
\end{theorem}

\begin{proof}
    As in \eqref{eq:finvar}, we note that $X^n_t=P^n X_t$ and in particular $P^n X^m_t=X^n_t$ for all $t \in [0,T]$ and $m\geq n$. Thus, we also have $\tilde{f}^n(s, X^n_{T-s})= \tilde{f}^m(s, X^m_{T-s})$ for all $s \in [0,t_0]$ and $m\geq n$. In particular, we can bound the training loss \eqref{xxx} at level $m\geq n$ by 
    \begin{equation*}
        \mathbb{E}\Big[ \int_0^{t_0} \lVert f^m(s,X^m_{T-s}) - \tilde{f}^m(s,X^m_{T-s})\rVert^2 {\rm d} s \Big] \le  2 \lVert F^m-F^n\rVert_{L^2(\Omega\times [0,t_0]; H)}^2+ 2\varepsilon_n\le 4 \tilde{\varepsilon}_n+2 \varepsilon_n.
    \end{equation*}
    Applying Theorem \ref{thm:main} immediately yields the claimed error estimate.
\end{proof}

Note that the convergence of $f^n(t,X_t^n)$ in $L^2(\Omega,\mathcal F,\mathbb P;L^2([0,t_0];H))$ assumed in the previous theorem is weaker than convergence in $L^2(\Omega,\mathcal F,\mathbb P;\mathcal C([0,t_0];H))$. Such uniform convergence of $f^n(t,X_t^n)$ is a reasonable assumption as is discussed in \cite{jakiw_diffusion}. Further, we need uniform Lipschitz conditions. One example of the uniform Lipschitz continuity is given here in \cref{example_gaussians} and in \cite[Theorem 2]{jakiw_diffusion} for data distributions with bounded support in the Cameron-Martin space.  

\section{Multilevel Learning Approaches}
\label{sec: theorytoexp}

In this section, we demonstrate that our infinite-dimensional diffusion model framework enables training or testing across multiple resolutions, which we refer to as ‘multilevel training’ in our paper (in spirit with classical numerical analysis such as \cite{bornemann1996cascadic}). Specifically, we mean this in the spirit of \cref{thm:multilevel}, where the loss at a finer level is bounded by the score trained on a coarser level, i.e., a model trained on the coarse resolution also performs reasonably well on higher resolutions. We illustrate this from the network architecture for the score operator and the prior distributions that are consistent with our theory. 

Note that $\widetilde{\beta_t}$ in $L_{DSM}$ defined by \eqref{eq:LDSM} vanishes when $t$ approaches $0$. In the experiments, we follow \cite{huangVarScore} and consider the weighted loss based on $L_{DSM}$, 
\begin{equation}
    \label{newloss}
\mathbb{E}_{x(0) \sim \Pdata, t \sim U[0,T]} \mathbb{E}_{x(t) \sim \mathbb{P}_{X_t|X_0 = x(0)} }
\left[ \left\| \widetilde{\beta_t}^{\frac12} \frac{s_\theta(t,x(t)) }{g(t)}+ \eta \right\|^2 \right],
\end{equation}
where $\eta \sim \mathcal N(0, Q)$.
That is, we regress the score so that it matches $g(t) Q \nabla \log p_t(x(t)|x(0))$. Note that we can also treat $Q$ in the score as an induced bias, so that we write $Q s_\theta(t,x(t))$ in place of $s_\theta(t,x(t))$ in \eqref{newloss} and thus the score is the parametrization of $g(t) \nabla \log p_t(x(t)|x(0))$ instead. We experimented with these choices, and found that the former choice of parametrizing the pre-conditioned score $g(t) Q \nabla \log p_t(x(t)|x(0))$ shows better numerical performance.
Our discretization of the infinite-dimensional SBDM differs from most commonly used diffusion models in two aspects: First, we use a neural operator architecture to learn the score operator $s_\theta$, for instance, FNO \cite{stuartNO}; Second, we change the latent distribution $\PZ$ as well as the noise $\eta$ in the loss function
to a Gaussian noise with trace class covariance operator $Q$. 
Both of these choices are discussed in the following sections. 

We also note that the approximation in Theorem~\ref{thm:main} is ensured by the bound with respect to the norm on the Hilbert space $H_n$ in \eqref{xxx}. The map $\iota_n$ defines an isometry between $H_n$ and the Euclidean space $\mathbb{R}^n$. When we train or evaluate on different resolutions, we need to respect this isometric isomorphism when defining the Euclidean basis elements. We need to rescale the corresponding coarser images when using downscaling operations such as average pooling in pixel space. Otherwise, comparing metrics or errors across different resolutions can be misleading.

\subsection{FNO parameterization of the score function}
\label{fno_def}
We use the Fourier neural operator (FNO) architecture \cite{stuartNO}  to approximate the score function more consistently across different resolutions.
In the spirit of \cite{kovachki2021universal}, we want to learn the score operator as a map between function spaces. However, since our viewpoint is finite-dimensional in training (we circumvent the definition of the infinite-dimensional score), we formulate this section in finite dimensions. We consider image input of size $n = N \times N$. 

The input to our score operator $s_{\theta}^{\rm{FNO}}: [0,T] \times \R^{N,N} \rightarrow \mathbb{R}^{N,N}$ are an image and a time. We concatenate the image $x$ with the time component and obtain $\tilde{x}_t = (x,t) \in \mathbb{R}^{2,N,N}$,i.e.,  we do not consider Fourier features of the time component. The first channel is the image of size $N \times N$, and the second is a constant image whose value equals $t$.  The score network $s_{\theta}^{\rm{FNO}}(\tilde{x}_t)$ is defined by the neural operator architecture, which is the composition of the following components:
The lifting mapping, $R$, is parameterized by the neural network with $R(\tilde{x}_t) $ and increases the channel size; i.e., it is a map from $\mathbb{R}^{2,N,N} \rightarrow \mathbb{R}^{ch,N,N}$. Similarly, $A: \mathbb{R}^{ch,N,N} \rightarrow \mathbb{R}^{N,N}$ reduces the feature map to one channel. For simplicity, we define the operator for ${\rm ch} = 1$ and refer to the original work for the full case. The layers $\mathcal{L}_{\ell}$ for $\ell = 1, \ldots, L$ are what make the FNO architecture resolution invariant. These layers are given by 
$$\mathcal{L}_l(v)= \sigma(W_l\ v + b_l + \mathcal{K}(f_{\theta_l})v).$$
Here,  $\mathcal{K}(f_{\theta})$ a convolution operator defined by the periodic function $f_{\theta_l}$. Assuming periodic boundary conditions on the images, we can compute this efficiently in Fourier space
\[
   \mathcal{K}(f_\theta)v  = \mathrm{FFT}^{-1}( P_\theta \odot \mathrm{FFT}(v)),
\]
where $\odot$ is the Hadamard product and the equality holds if $P_\theta$ are the eigenvalues of $\mathcal{K}(f_\theta)$, which can be computed via FFT. As suggested in~\cite{stuartNO}, we directly parameterize the non-zero entries $P_\theta$ associated with low frequencies.
The cutoff for modes is given by predefined via the hyperparameters $(k^{\max}_1,k^{\max}_2)$, where we set $P_\theta$ to zero if $k_1 > k^{\max}_1$ or $k_2 > k^{\max}_2$.

Hence, the score network is defined as 
$$ s_{\theta}^{\rm{FNO}}(\tilde{x}_t) = A \circ \mathcal{L}_L \circ ...  \circ \mathcal{L}_1 \circ R(\tilde{x}_t).$$
As shown in \cite[Theorem 5]{kovachki2021universal}, FNOs can approximate any continuous operator between Sobolev spaces to arbitrary accuracy when $k_{\max}$ is sufficiently large.

\subsection{Prior distributions}\label{sub:latent}

We outline the several prior distributions used in our experiments.

Firstly, we consider the standard Gaussian distribution as a benchmark to investigate the role of trace-class operators in multilevel training. That is, for $N \times N$ images, we define the prior distribution as $x \sim \mathcal{N}(0,I_{N^2})$. The identity operator is not a trace-class operator, and there is no decay in the eigenvalues. We set this up as the benchmark prior.

Secondly, the FNO prior is implemented by a randomly initialized spectral convolution layer, and it is applied to a standard Gaussian. The main idea is to transfer images to the frequency domain and to apply a "low-pass" filter to these. First, we choose a maximum number of modes, apply the Fourier transform to our input $x$. Second,  we cut off (or set them to 0) all the modes of higher order, and perform a spectral convolution like in \cite{stuartNO}. 
    Experimentally, we randomly initialize $\varphi$ and impose the conjugate symmetry, although it could also be learned. As a result, we construct the FNO prior via 
    \[A_{\rm FNO}(x) = \mathrm{FFT}^{-1}(\varphi \odot \ \mathrm{FFT}(x)),\]
    where $x \sim \mathcal{N}(0,I_{N^2})$.
    
Thirdly, we obtain the negative inverse Laplacian $-\Delta^{-k}$ as a covariance kernel by using the standard five-point stencil.
    Assuming periodic boundary conditions on the images, as explained in the previous section, we can compute the eigenvalues of the Laplacian, denoted by 
    $\lambda(-\Delta)$,  with FFT and efficiently invert it in Fourier space by
    \begin{equation}
        A_{\rm Lap}(x) = \mathrm{FFT}^{-1}( (\lambda(-\Delta))^{-k/2} \odot \mathrm{FFT}(x)),
    \end{equation}
%
to  sample efficiently from  $\mathcal{N}(0, -\Delta^{-k})$, by first  sampling $x$ from standard Gaussian distribution and then apply $A_{\mathrm{Lap}}.$

Now, we can easily adapt the FNO prior $A_{\mathrm{FNO}}$ and the negative inverse Laplacian prior $A_{\rm lap}(x)$ to the following two priors that are theoretically sound and yield good multilevel performance in the numerical experiments. We present numerical evidence in \cref{sec:numerics}.
\begin{itemize}
    \item We propose a new prior that combines the FNO prior and Laplacian prior and reads $A_{\mathrm{Comb}} = \gamma_0 A_{\mathrm{FNO}}+\gamma_1 A_{\rm Lap} $. Here, the hyperparameter $\gamma_0, \gamma_1 \geq 0$ controls the influence of each of those priors. We recover $A_{\rm lap}(x)$ when $\gamma_0 = 0$, and $A_{\mathrm{FNO}}$ when $\gamma_1 = 0$. This yields a linear map. Therefore, applying $A_{\mathrm{Comb}}$ to a standard Gaussian yields a Gaussian. Furthermore, the influence of the inverse Laplacian takes care that we indeed obtain a non-degenerate covariance. 

    \item We consider the Bessel potential operator \cite{aronszajn1961theory}, \cite[Example 6.17]{stuart2010inverse} (which is also used in \cite{nvidia_func_diff}) and define the following Bessel prior
        \begin{equation}
        \label{eq: bessel}
        A_{\rm Bes}(x) = \mathrm{FFT}^{-1}( (\lambda(\gamma_2 I -\Delta))^{-k/2} \odot \mathrm{FFT}(x)),
    \end{equation}
to sample efficiently from $\mathcal{N}(0, (\gamma_2 I -\Delta)^{-k})$, where $x \sim \mathcal{N}(0,I_{N^2})$. We take $\gamma_2 > 0$ and $k > 1$ to ensure that we have a positive-definite, trace-class covariance operator.  We recover $A_{\rm lap}(x)$ when $\gamma_2 = 0$.
\end{itemize}


\section{Experiments} \label{sec:numerics}

In this section, we show the benefits of multilevel learning with neural operator architectures and trace class covariance operators through numerical experiments. We focus on the following three tests in our experiments.
\begin{itemize}
\item In \cref{sec: toy}, we use a synthetic Gaussian mixture dataset that allows comparing the estimated score to the true score. The result motivates our choice of the downsampling method and our use of the neural operator as the score network architecture.
\item In \cref{5.1}, we use the MNIST dataset to show that our choices of architectures and priors enable super-resolution; i.e., the models can generate higher resolution images than their training resolutions. With consistent downsampling, the loss at a higher resolution also
decreases when training at a coarse resolution. 
\item In \cref{5.3}, we demonstrate for MNIST that the weights trained on a coarse resolution can serve as a warm start for training at higher resolutions. 
\item In \cref{5.2}, we use the same model setup and parameter count to train 2D reaction-diffusion data at different resolutions. We show that the same model setup yields similar performances when retraining at different resolutions and using model weights at a coarse resolution to warm start fine-resolution training. 
\end{itemize}
The Python code is available at \url{https://github.com/PaulLyonel/multilevelDiff.git}.

\subsection{Gaussian Mixture}
\label{sec: toy}
In this section, we verify many of our design choices on a synthetic Gaussian mixture dataset with known ``ground truth" score function, $Q \nabla \log p_t(x)$. This is used in many theoretical considerations on diffusion models and has been recently used for generative modeling in \cite{scarvelis2023diff} by smoothing out the ``empirical score". We draw the data from a Gaussian mixture model $\mathbb{P}_{data} = \frac{1}{L} \sum_{i=1}^L \mathcal{N}(f_i, Q)$ consisting of $L$ Gaussians with different means. The prior distribution is $\mathcal{N}(0, Q)$, so the data and the prior share the same covariance. We can derive that $\mathbb{P}_{X_t} = \frac{1}{L}\sum_{i=1}^L \mathcal{N}(\tilde{\alpha}_t f_i, \tilde{\beta}_t Q + \tilde{\alpha}_t^2 Q)$, using the results in \cref{sec:score_finite}. Hence, we can efficiently track $g^2(t) Q \nabla \log p_t(x)$ and check the performance on different resolutions in terms of the approximation of the score. Here, we consider $L = 2$, the Gaussian mixture mean functions are $f_1(y_1,y_2) = (y_2,y_2)$ and $f_2(y_1,y_2) = (1-y_1, 1-y_1)$, where $(y_1,y_2)$ is any point in the $N \times N$ mesh on $[0,1]^2$. To ensure they possess enough regularity we consider $Q^{1/2} f_1$ and $Q^{1/2}f_2$ as the means instead. Thus, in this section, we demonstrate that our proposed consistent downsampling, which cuts off the high frequencies outperforms standard average pooling. We show empirically that the operator architecture has advantages over the U-Net. 

We calculate the score via autograd and use the closed-form formula for $\mathbb{P}_{X_t}$ to get the difference $\mathbb{E}_{t,x(t)}[\Vert s_{\theta}(t,x(t)) - Q \nabla \log p_t(x(t)) \Vert^2]$. Note that due to our parametrization we instead report the difference $\mathbb{E}_{t,x(t)}[\Vert s_{\theta}(t,x(t)) - Q g(t) \nabla \log p_t(x(t)) \Vert^2]$. We train a U-Net with Fourier downsampling, a FNO with standard average pooling and a FNO with Fourier downsampling which induces a different scaling due to the isometry. We use the combined prior in this experiment. For this we draw samples at resolution 64 and use the different downsamplings/operators. We track the score difference at two different resolutions and obtain the graphs in Figure \ref{fig:toy}. We can see that the U-Net fails to generalize to higher dimensions, and that the standard average pooling (without scaling) has consistent higher score errors than the Fourier downsampling. Further we can see that the downsampled data at resolution 32 may not perfectly represent the ground truth at the same resolution, as we observe an optimal stopping time (epochs) after which the approximation error on resolution 64 increases in \cref{fig:toy}(c).

This motivates the downsampling and network choices in the following experiments.
\begin{figure}
\centering
\subfigure[U-Net + Fourier Downsampling]{\includegraphics[width=0.3\textwidth]{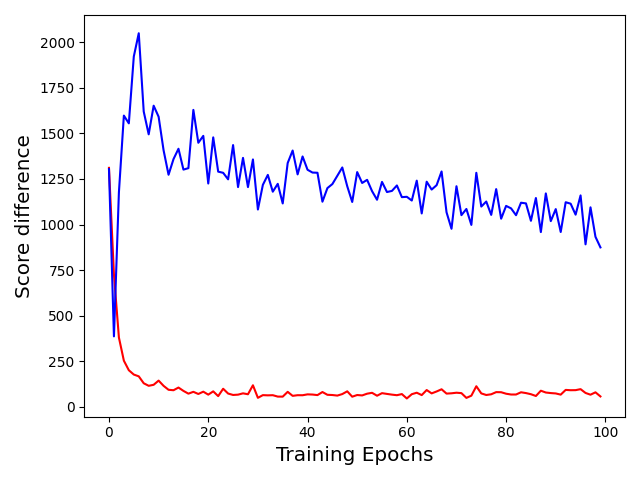}}
\subfigure[FNO + Pooling]{\includegraphics[width=0.3\textwidth]{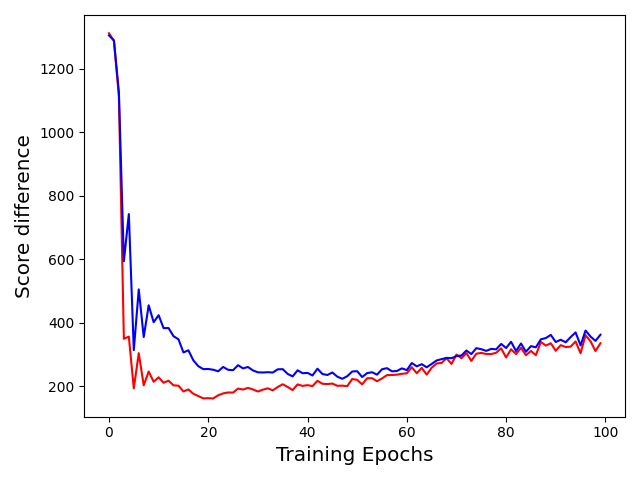}}
\subfigure[FNO + Fourier Downsampling]{\includegraphics[width=0.3\textwidth]{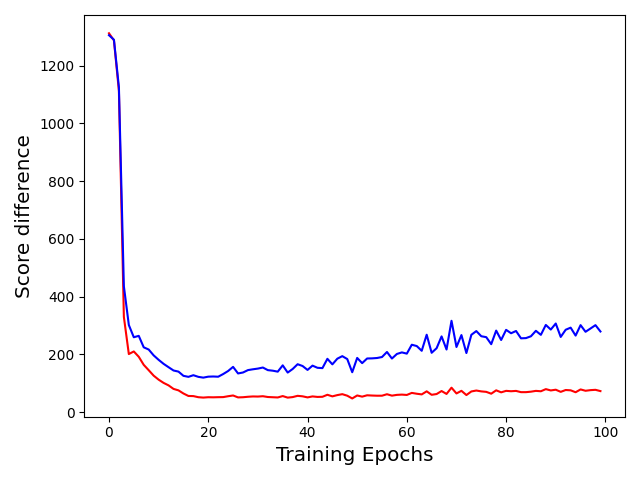}}
\caption{Approximation of the true score at different resolutions (red is resolution 32, blue is resolution 64). Note that the models are trained on the downsampled images, which are distinct from the resolution 32 images.}
\label{fig:toy}
\end{figure}


\subsection{MNIST experiment: Super-resolution}
\label{5.1}

Here, we explain the different design choices we make in the following experiments using the MNIST dataset \cite{726791}.
\paragraph{Operator Networks} While U-Nets \cite{ronneberger2015u} have contributed to the success of diffusion models by incorporating an implicit coarse-to-fine representation, they are not operators and, in particular, their action is sensitive to the resolution of the input. In our experiments, U-Nets trained on $32 \times 32$ images fail to generate images resembling the MNIST dataset at higher resolutions; see~\cref{fig:mnist_unet}.

\begin{figure}
\centering
\subfigure[U-Net resolution 32]{\includegraphics[width=0.3\textwidth]{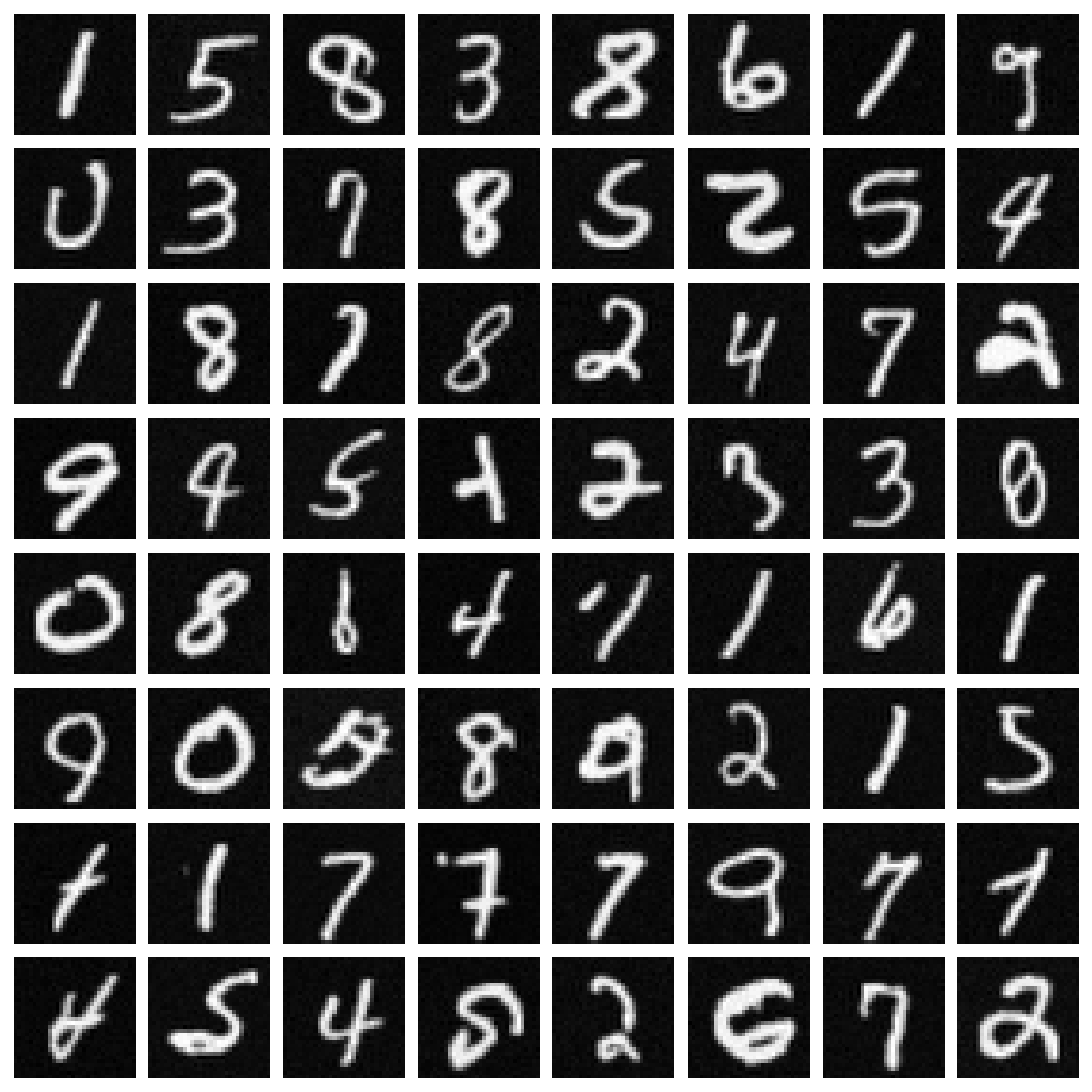}}
\subfigure[U-Net resolution 50]{\includegraphics[width=0.3\textwidth]{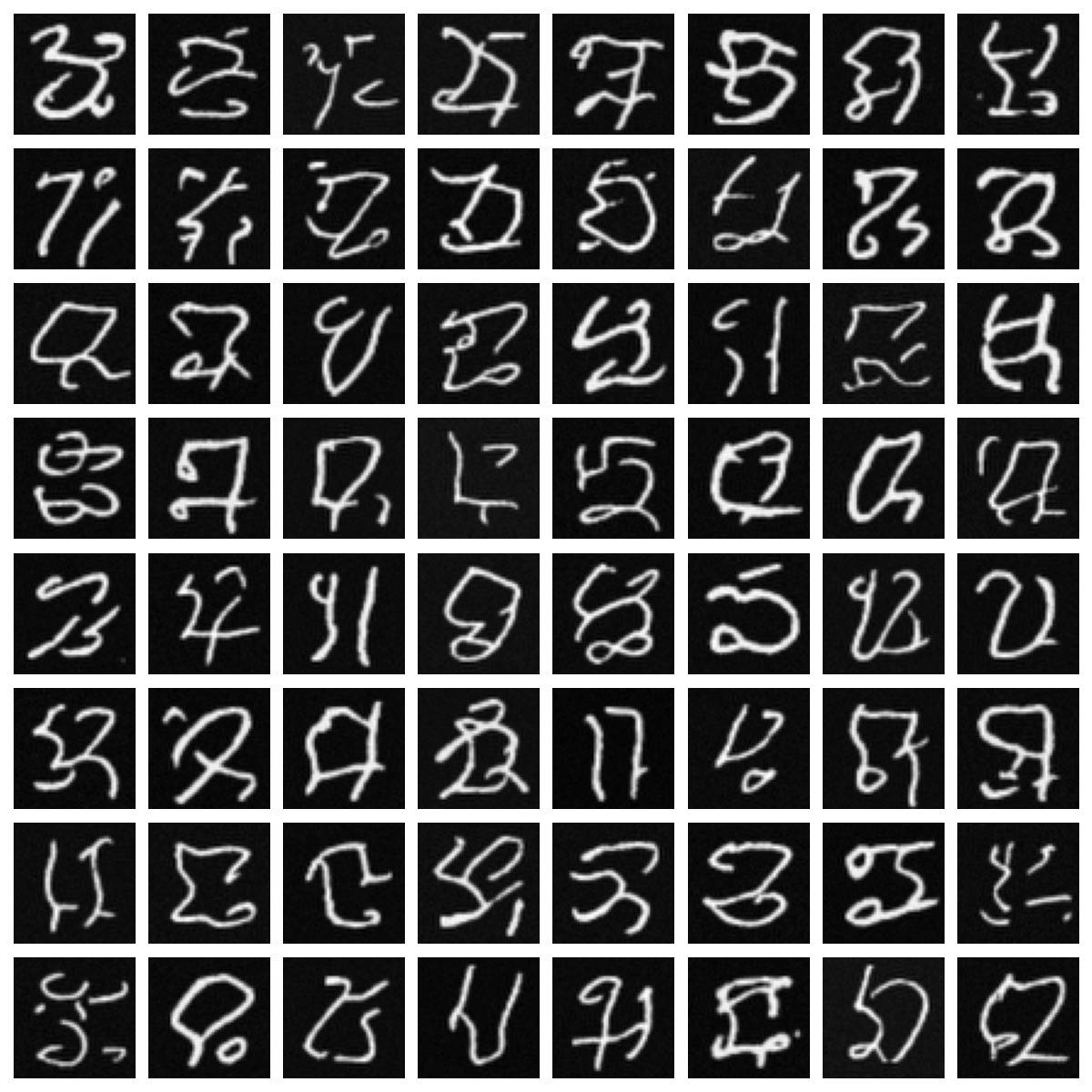}}
\subfigure[U-Net resolution 64]{\includegraphics[width=0.3\textwidth]{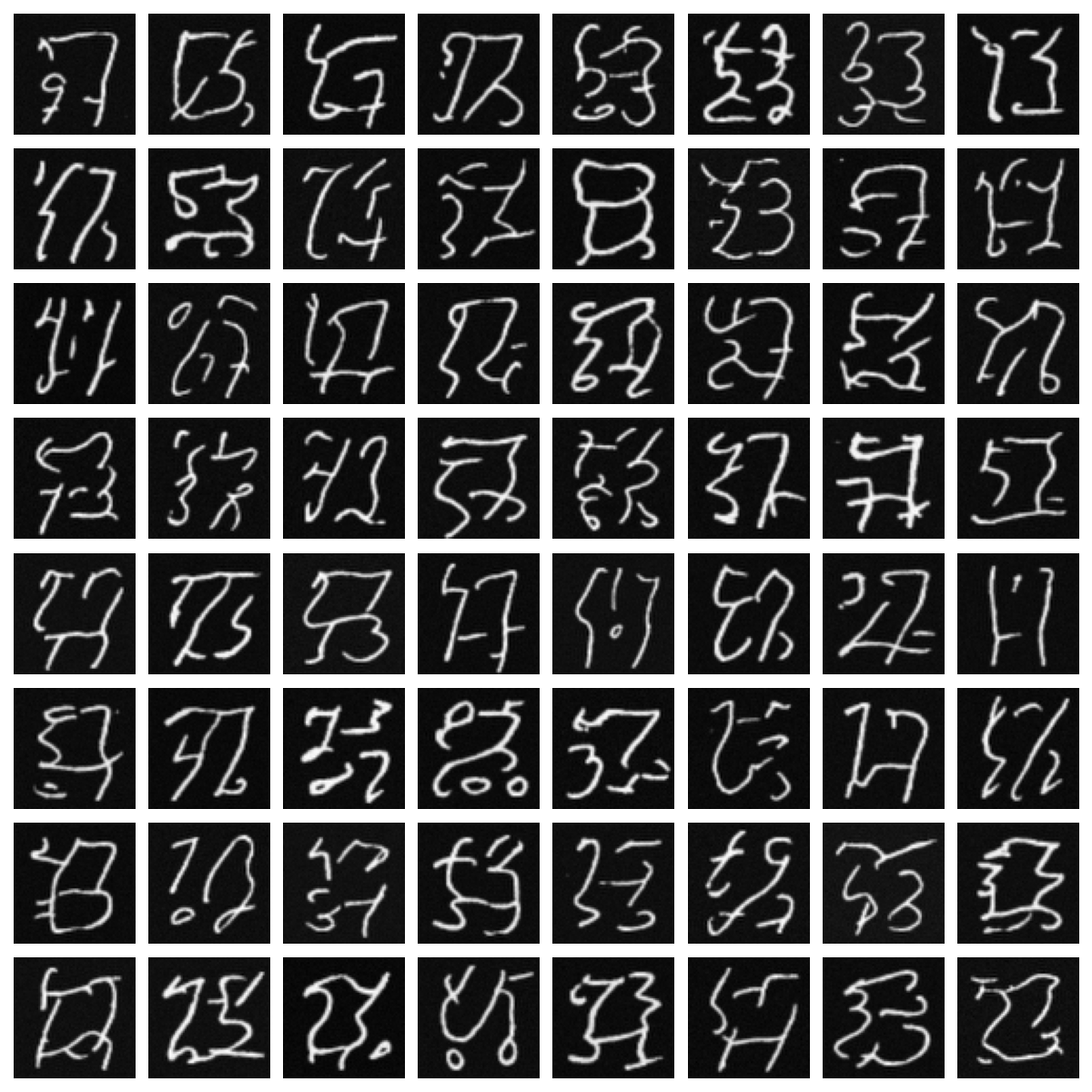}}
\caption{MNIST images generated by a U-Net trained with standard Gaussian prior at resolutions 32 at resolutions 32, 50, 64 (left to right).}
\label{fig:mnist_unet}
\end{figure}

\paragraph{Comparison of different priors} We train an FNO \cite{stuartNO} with the priors described in~\cref{sub:latent} using $32\times32$ images and assess the ability of the models to generate images at higher resolutions. We train all models for 200 epochs and evaluate their sliced Wasserstein distance every 10 epochs. We generate images using the model with the lowest DSM loss on a validation set at resolution 32. 
We visualize samples for the FNO models trained with different priors in the rows of \cref{fig:comp_prior_mnist} and provide loss curves at the coarse and fine resolution. The convergence plots indicate that coarse-level training with FNO also decreases the loss at higher resolutions.

The FNO and standard Gaussian show a much less diverse generation than the combined prior on resolutions 50 and 64. This can also be seen in \cref{fig:sw_mnist}, where we show how the sliced Wasserstein \cite{stein2023exposing} distance evaluated on 2000 samples against the test set progresses during training. We rescale the generated 64 images to $[0,1]$ before calculating the sliced Wasserstein \cite{stein2023exposing} metric. Here, we can see that the combined prior performs best and that the Laplacian and U-Nets do not generalize to higher dimensions. The Laplacian prior yields unsatisfactory images in MNIST experiments; see \cref{fig:comp_prior_mnist}. This can be explained by the quick decay of the singular values, and therefore, one might need a different norm for training rather than a standard $L^2$-norm, see, e.g., \cite{jakiw_diffusion} for a discussion of the use of the Cameron Martin norm. This is a special case of the Bessel prior \eqref{eq: bessel}, which performs well in our experiment described in \cref{5.2}.

To gauge the diversity of the samples, we compute the Vendi diversity \cite{friedman2022vendi} score for the models with the lowest DSM loss at training resolution for the one run. In our experiments, the combined prior (7.8), FNO prior (8.8), and standard Gaussian (7.05) yield similar results at resolution 64, where the FNO prior takes the lead. However, judging from the images, we can see a slight mode collapse in all models and that the FNO and standard Gaussian produce unreasonable images. In summary, the combined prior produces the visually most appealing images.

\begin{figure}
\centering
\subfigure[Inverse Laplacian resolution 32]{\includegraphics[width=0.2\textwidth]{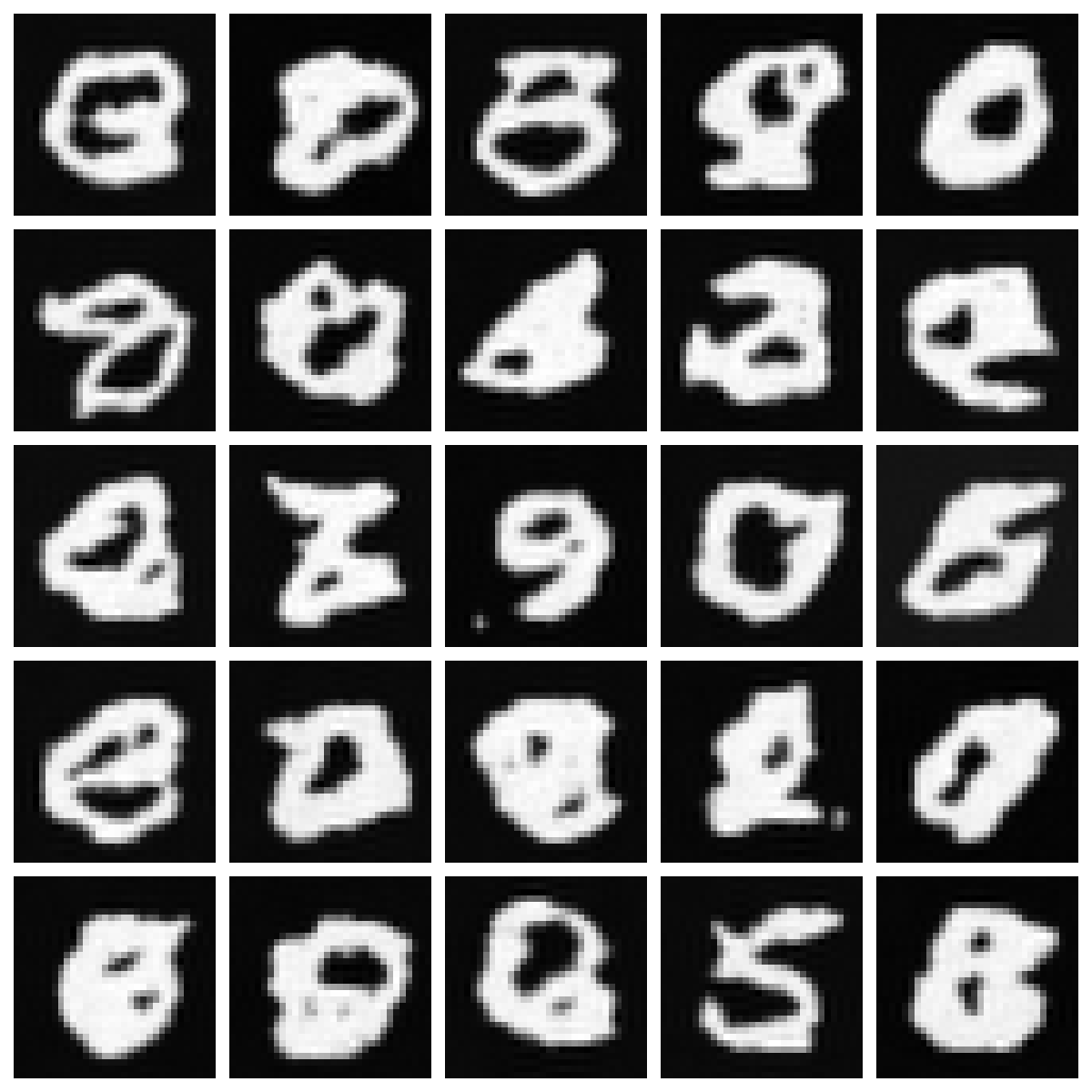}}
\subfigure[Inverse Laplacian resolution 50]{\includegraphics[width=0.2\textwidth]{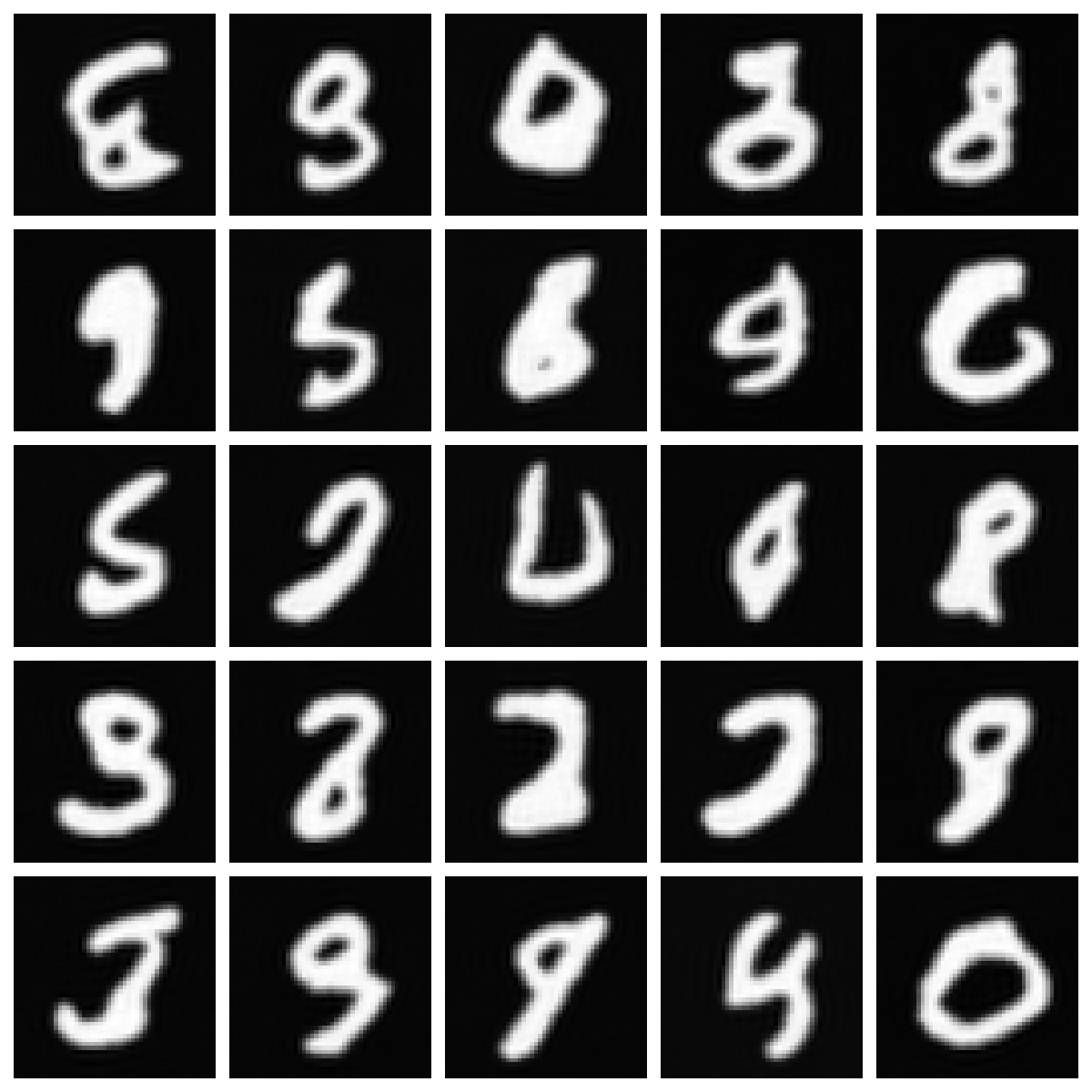}}
\subfigure[Inverse Laplacian resolution 64]{\includegraphics[width=0.2\textwidth]{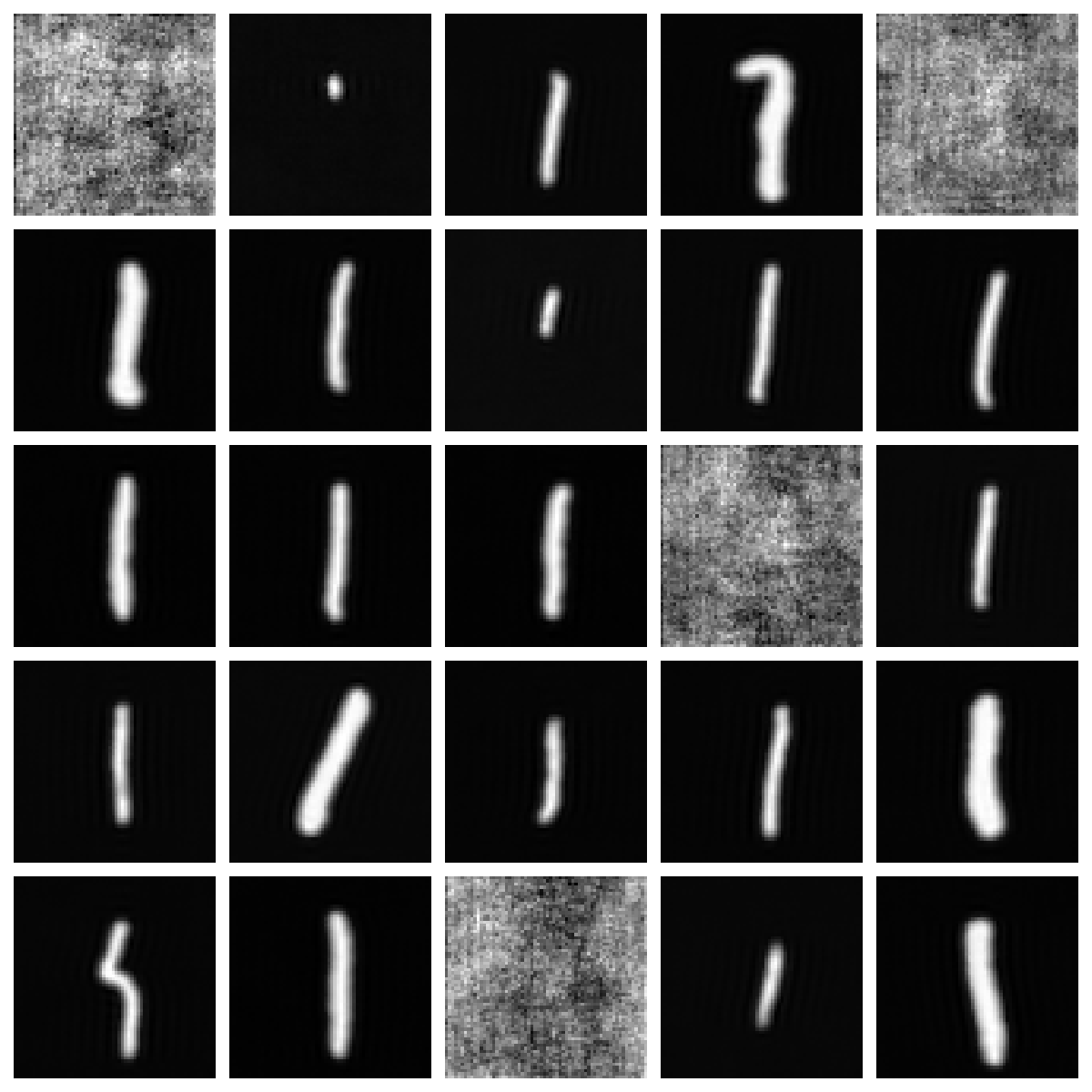}}
\subfigure[Laplacian Loss Curve]{\includegraphics[width=0.24\textwidth]{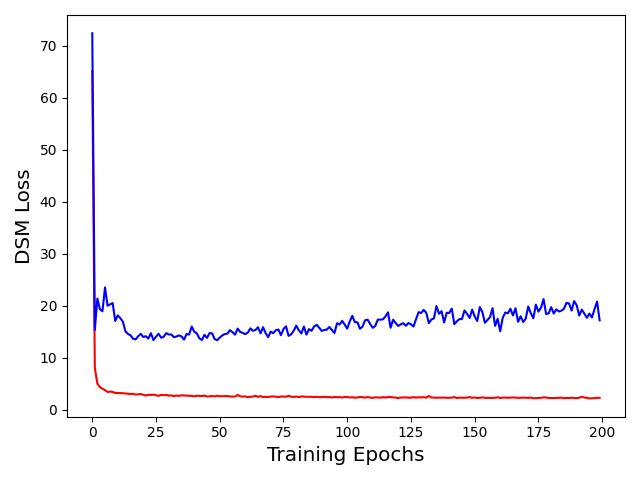}}

\subfigure[Standard resolution 32]{\includegraphics[width=0.2\textwidth]{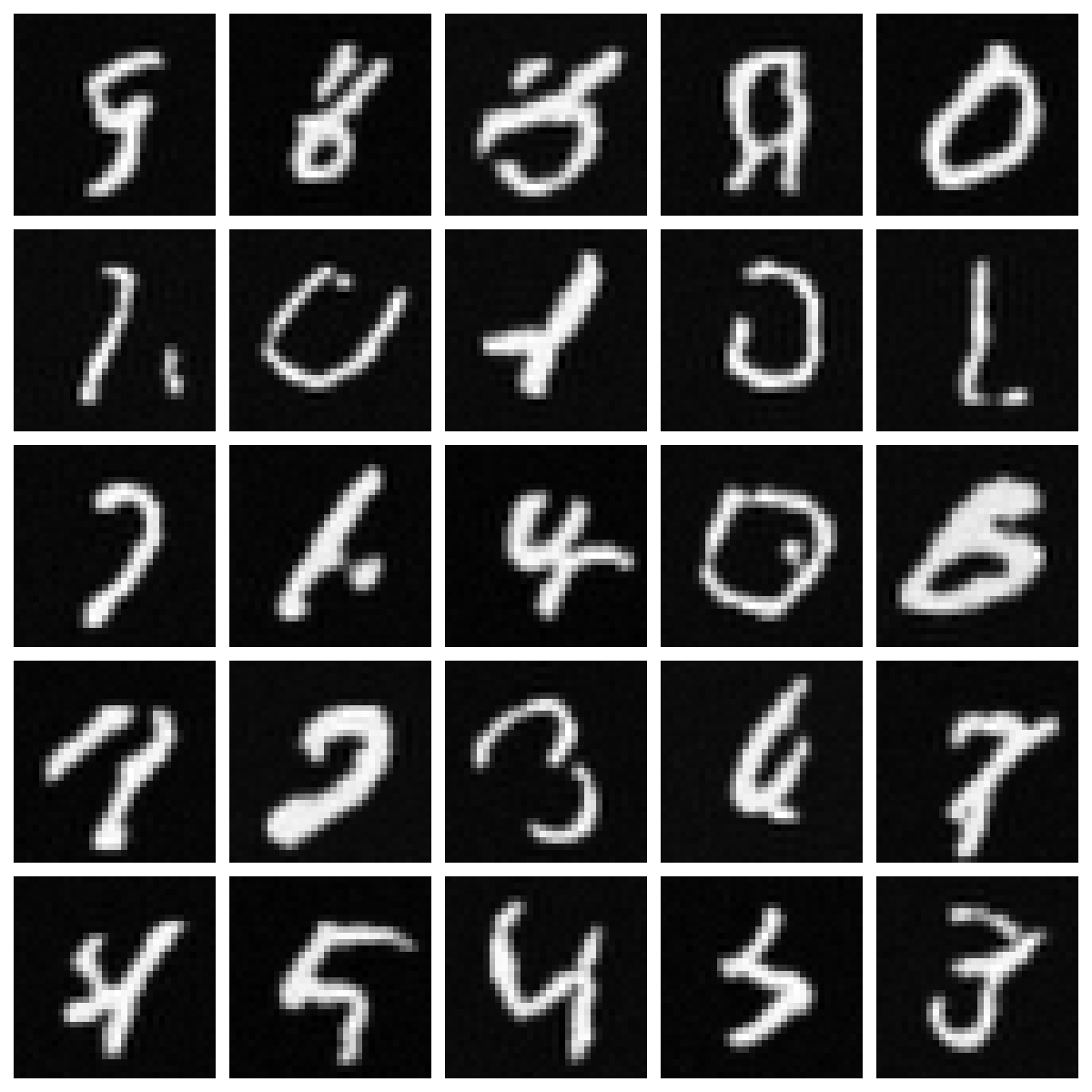}}
\subfigure[Standard resolution 50]{\includegraphics[width=0.2\textwidth]{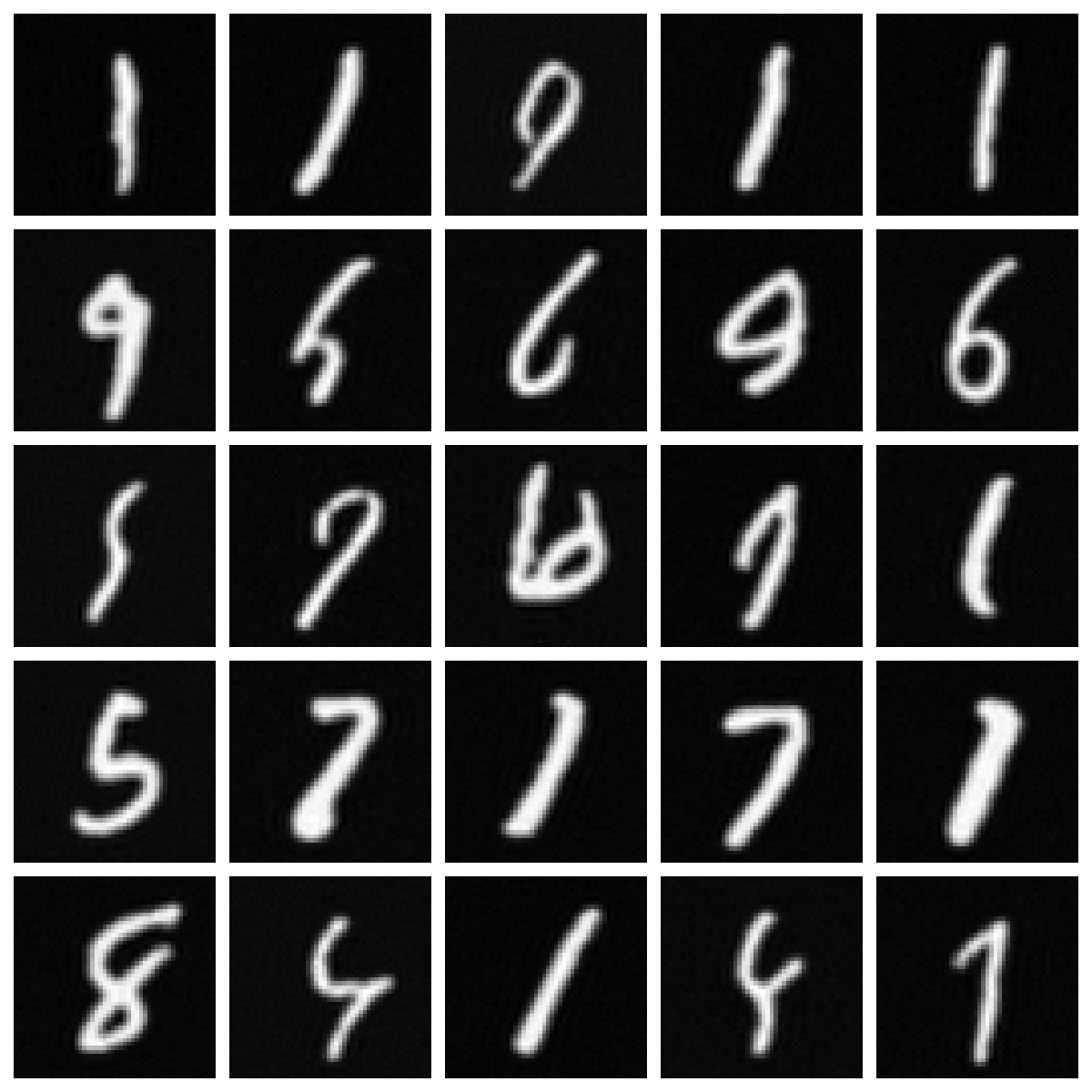}}
\subfigure[Standard resolution 64]{\includegraphics[width=0.2\textwidth]{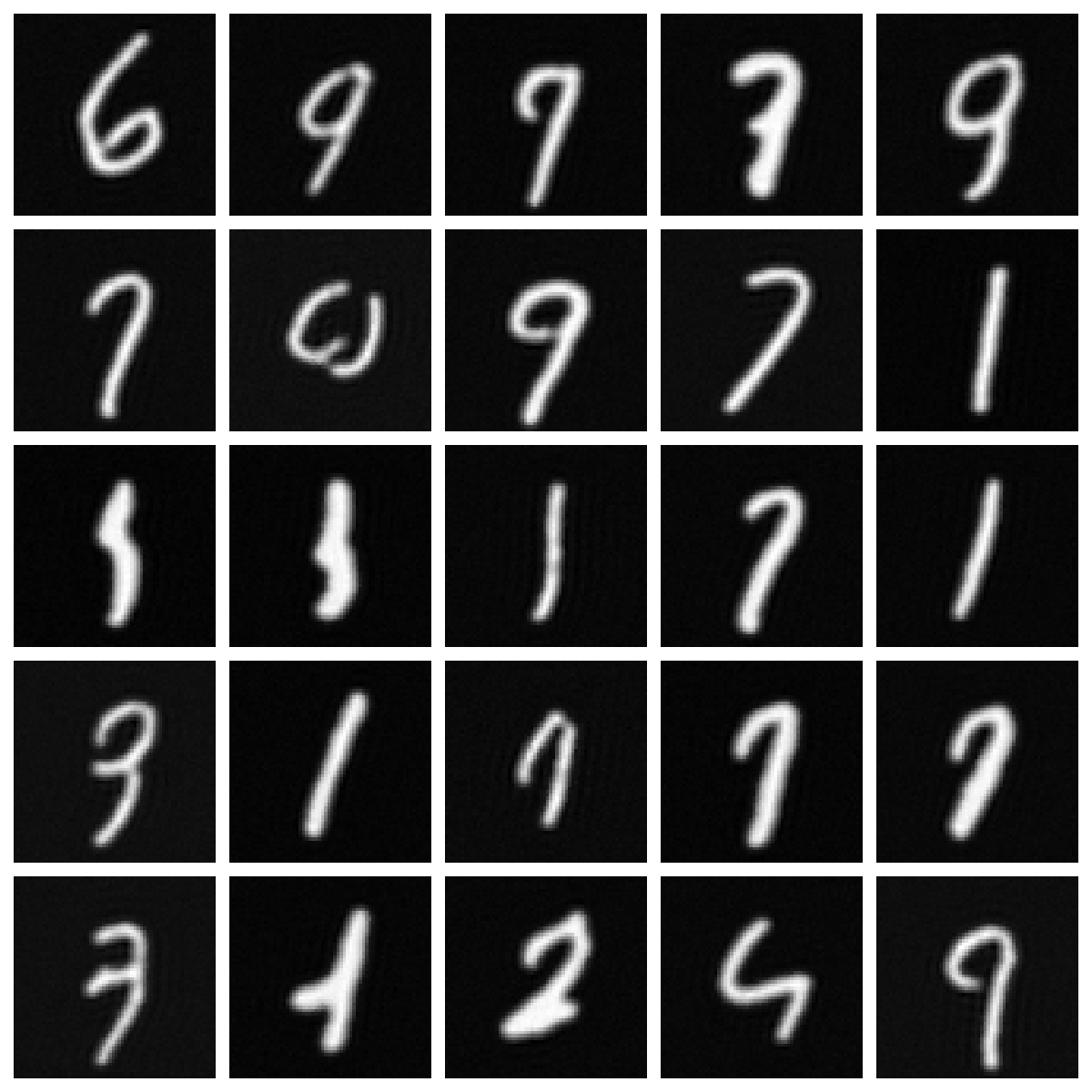}}
\subfigure[Standard Loss Curve]{\includegraphics[width=0.24\textwidth]{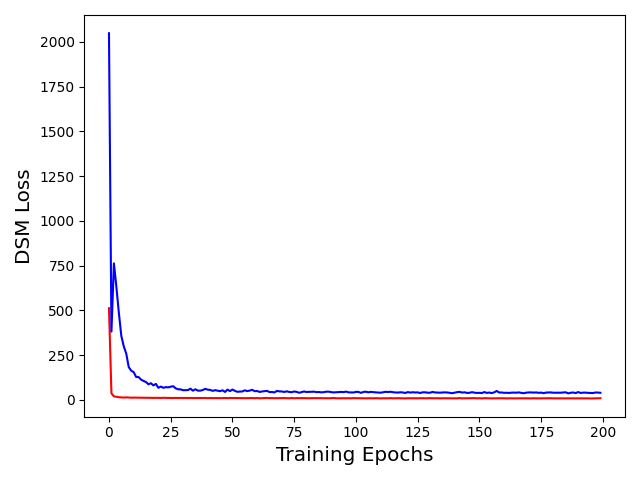}}

\subfigure[FNO resolution 32]{\includegraphics[width=0.2\textwidth]{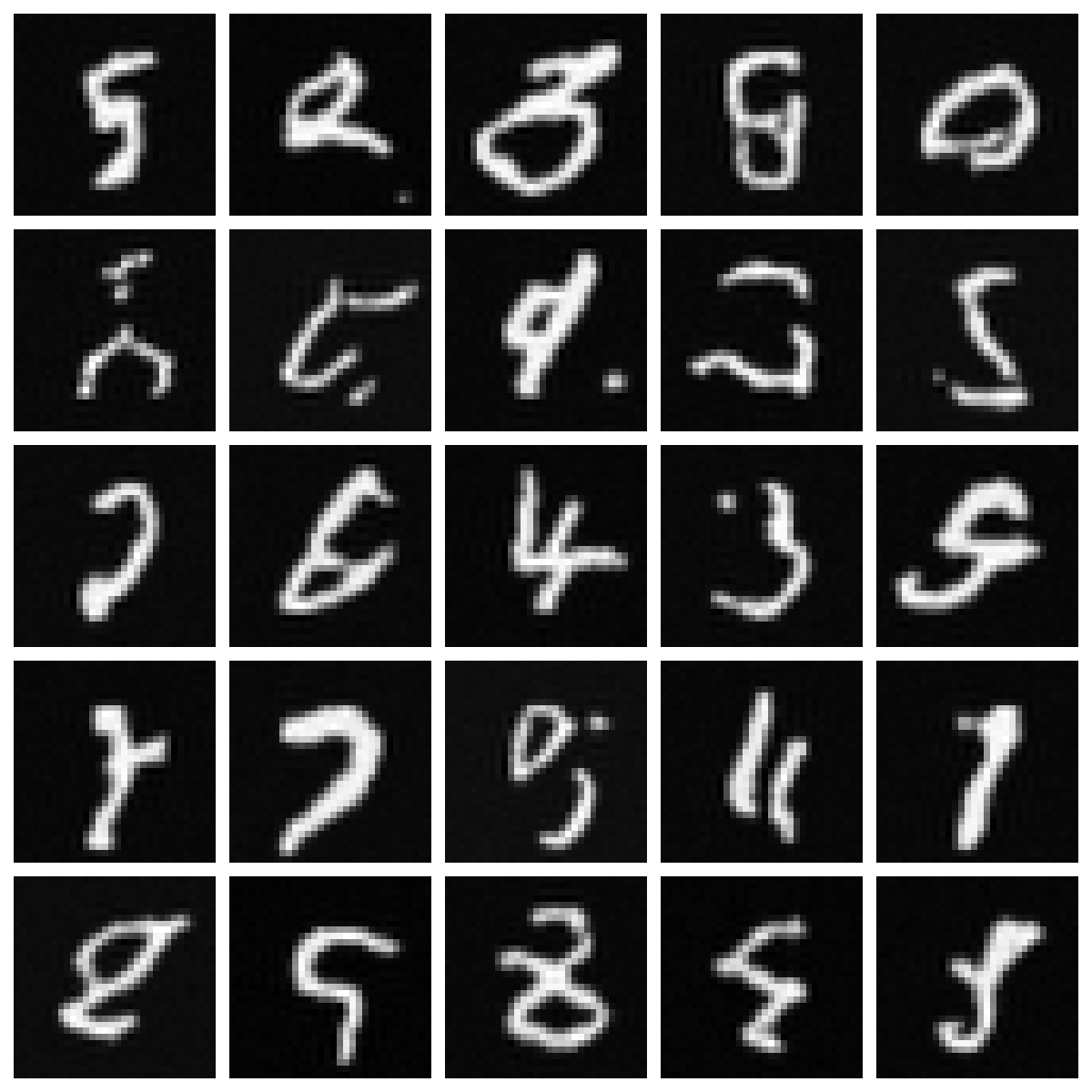}}
\subfigure[FNO resolution 50]{\includegraphics[width=0.2\textwidth]{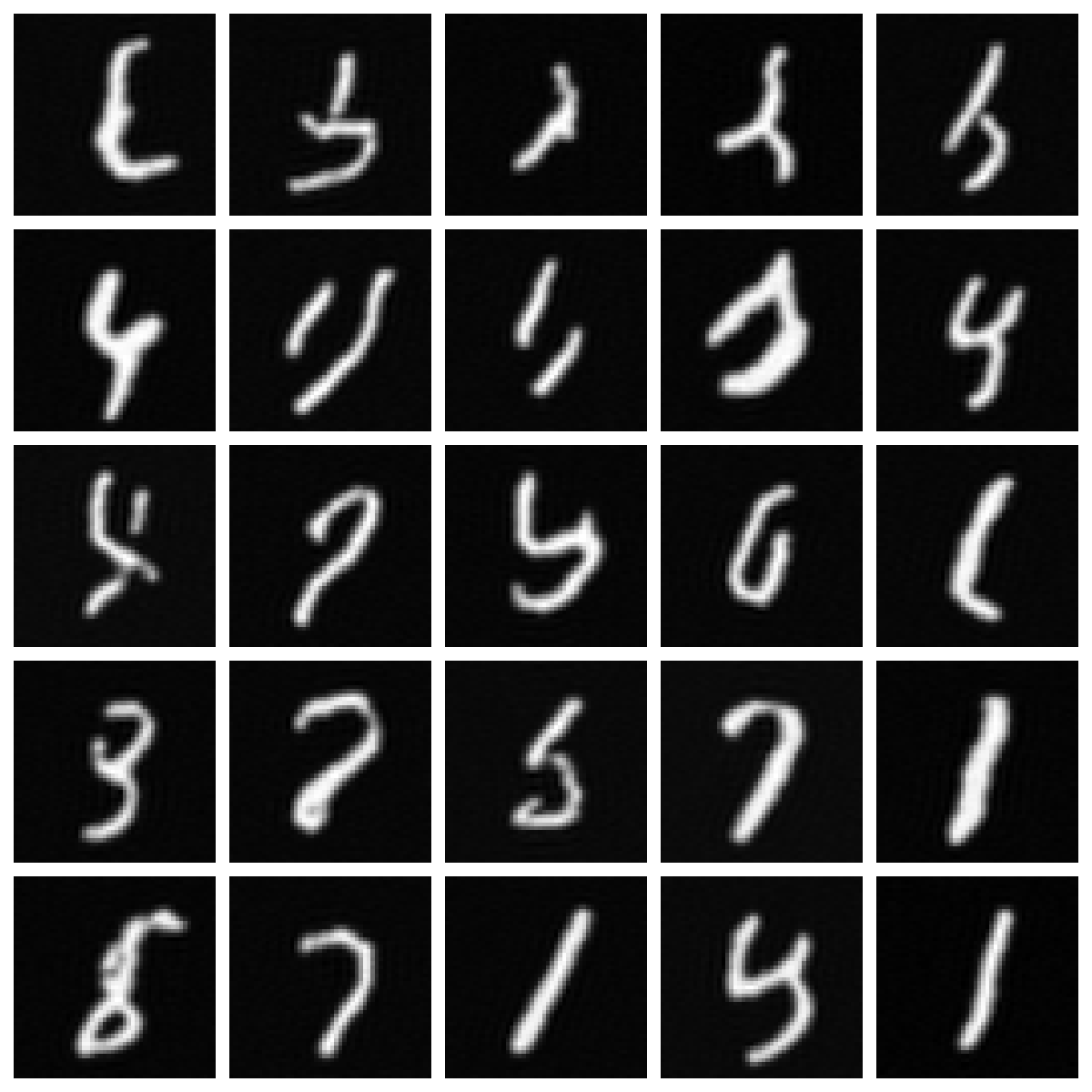}}
\subfigure[FNO resolution 64]{\includegraphics[width=0.2\textwidth]{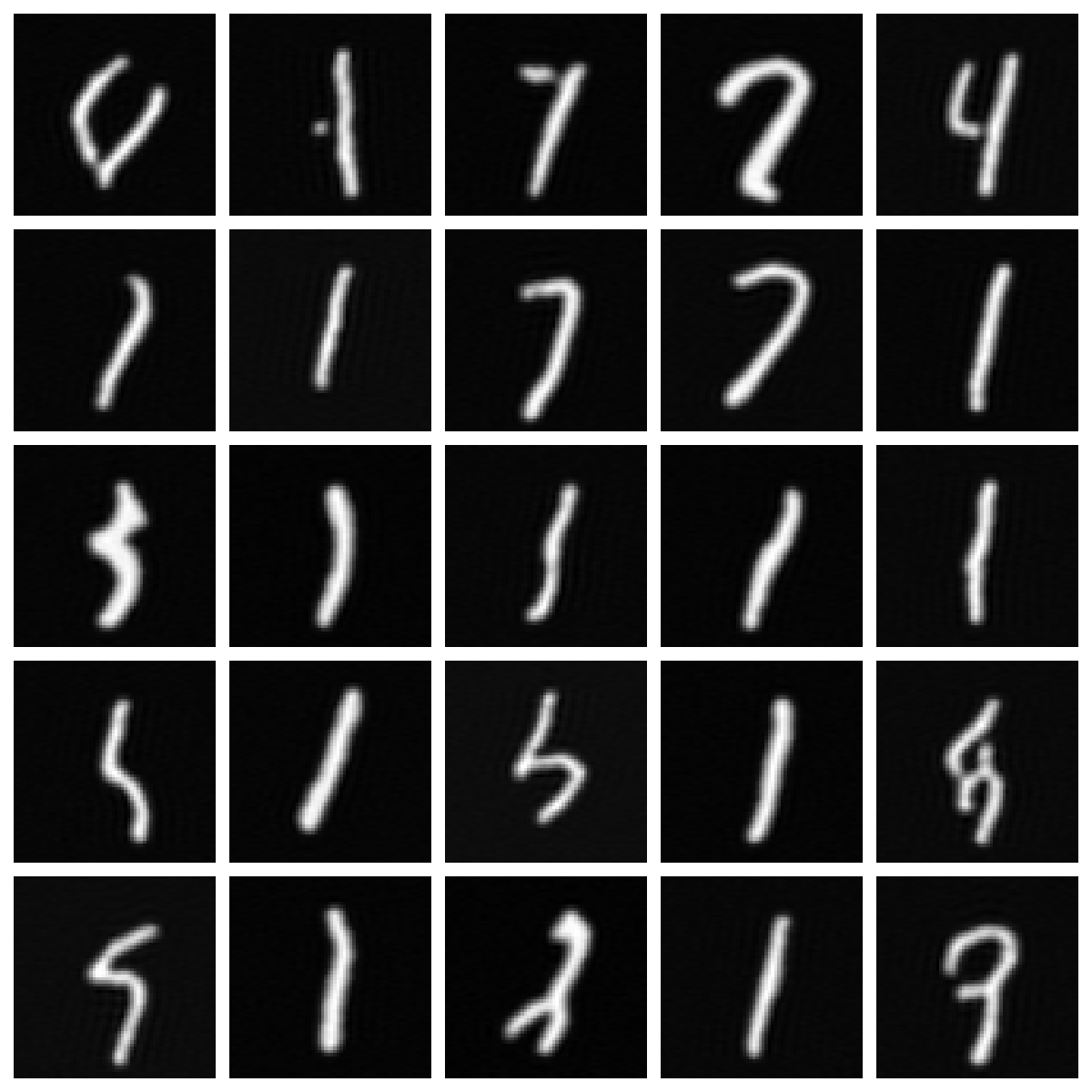}}
\subfigure[FNO Loss Curve]{\includegraphics[width=0.24\textwidth]{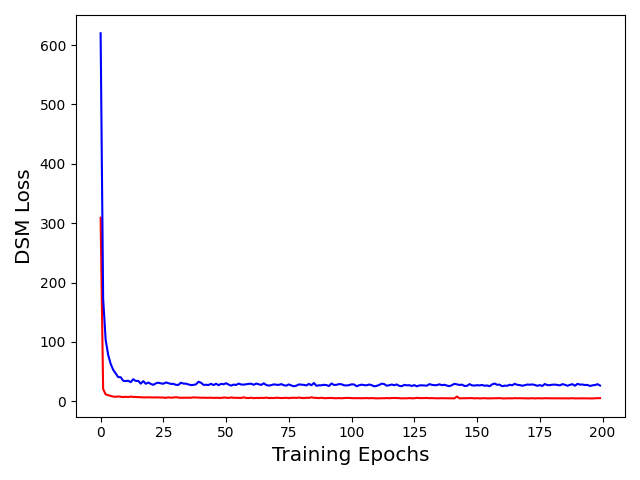}}

\subfigure[Combined resolution 32]{\includegraphics[width=0.2\textwidth]{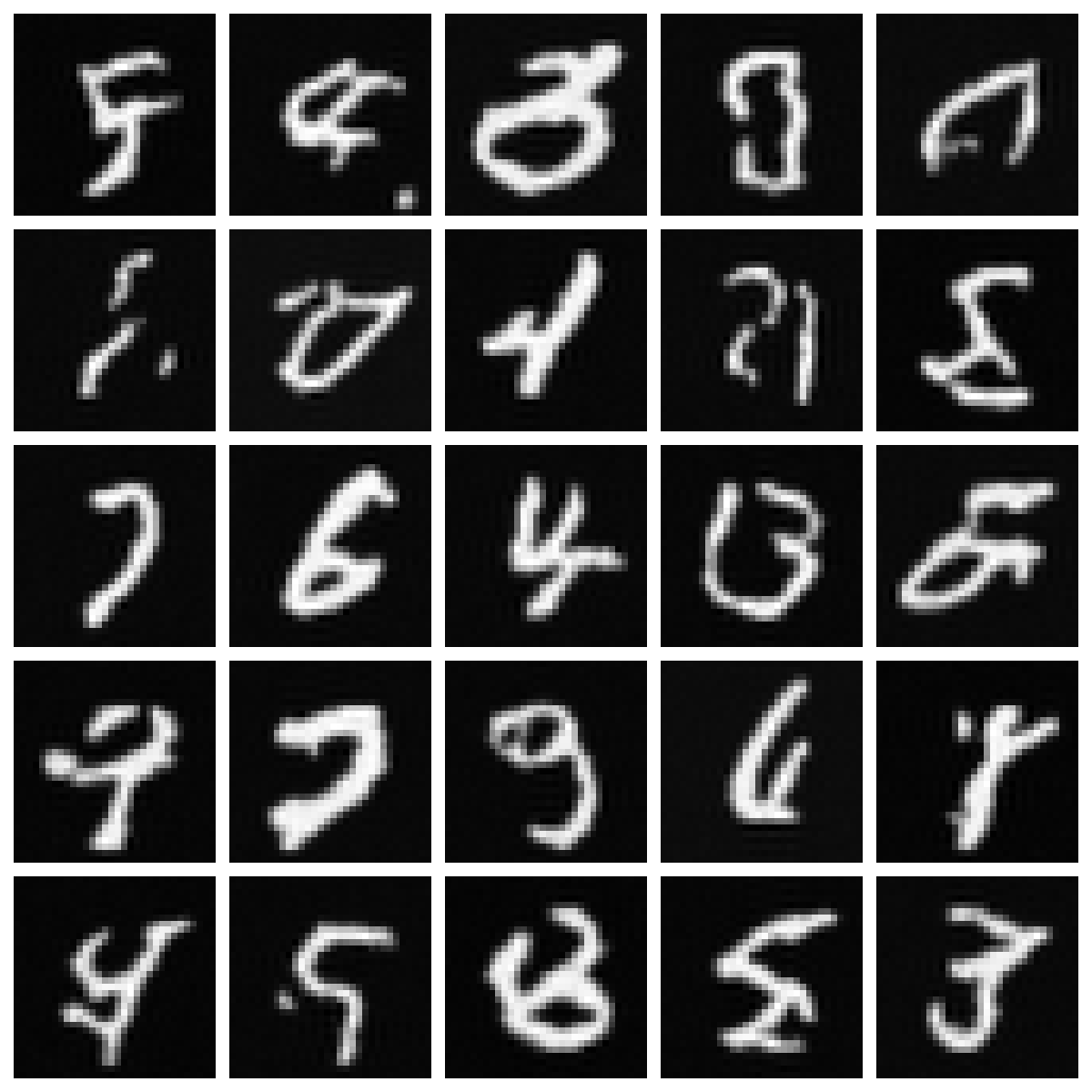}}
\subfigure[Combined resolution 50]{\includegraphics[width=0.2\textwidth]{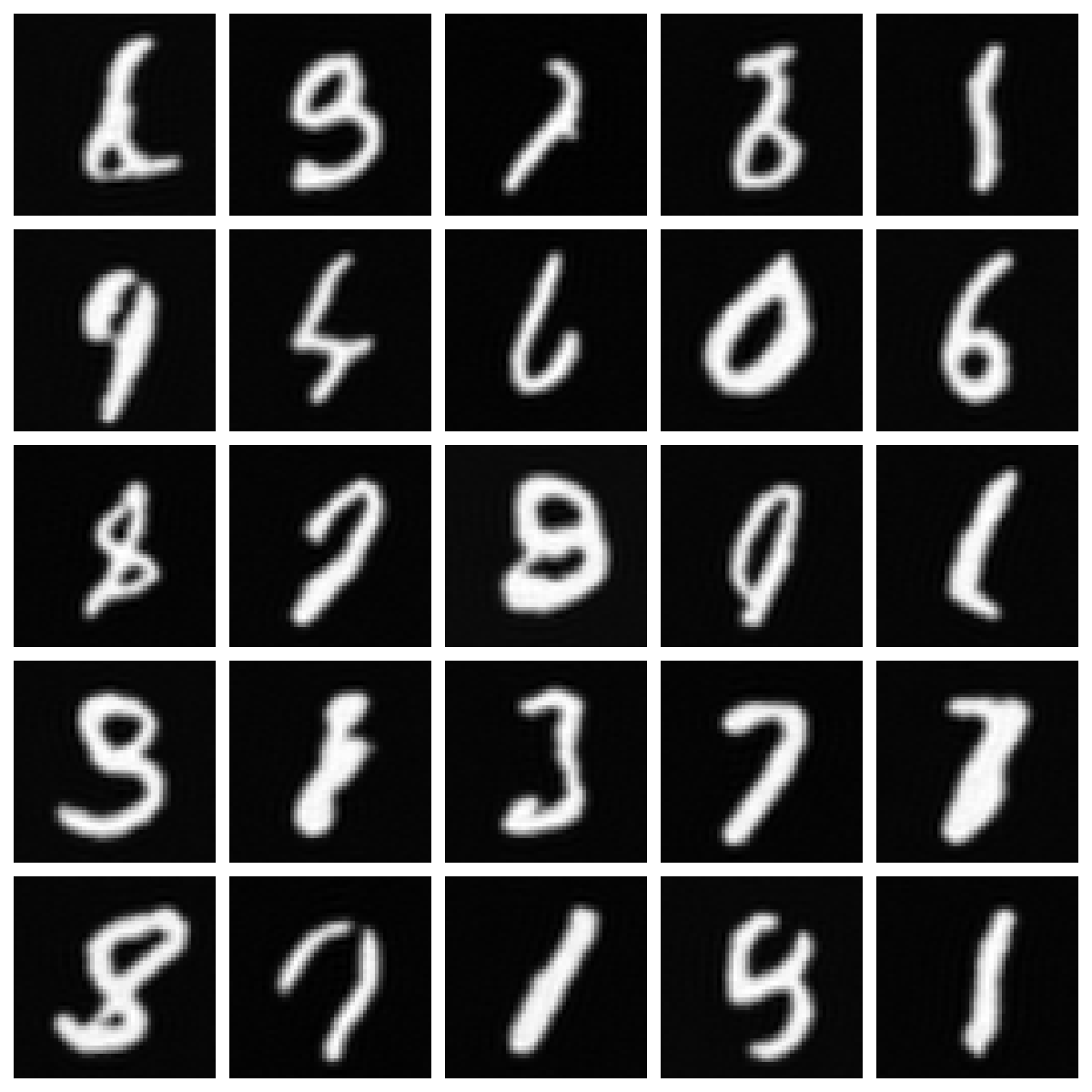}}
\subfigure[Combined resolution 64]{\includegraphics[width=0.2\textwidth]{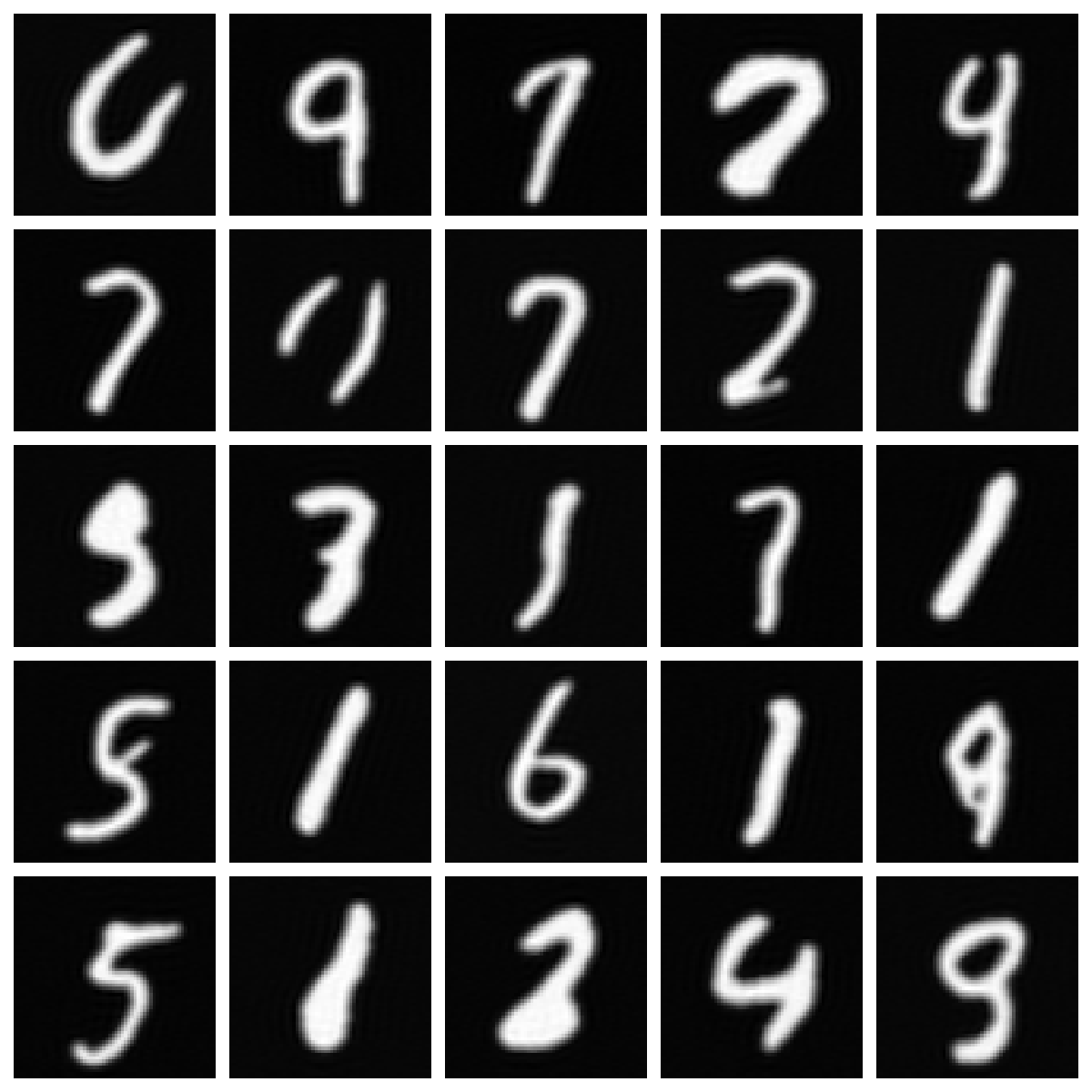}}
\subfigure[Combined Loss Curve]{\includegraphics[width=0.24\textwidth]{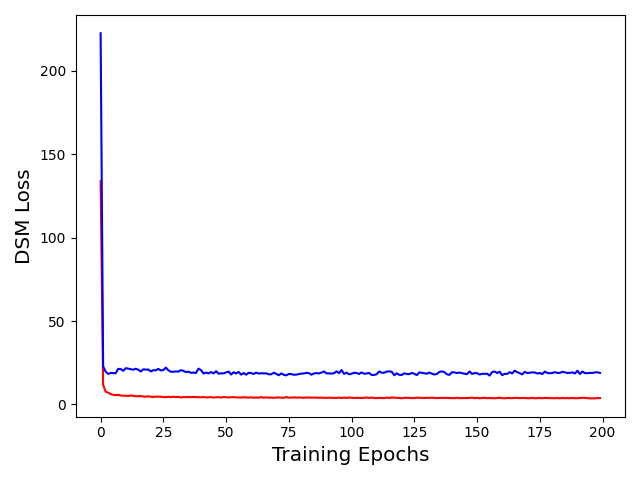}}

\caption{MNIST images generated by FNO and different priors at resolutions 32, 50, and 64 and loss curves for the MNIST example: training resolution loss (red) and loss at resolution 64 (blue). }
\label{fig:comp_prior_mnist}
\end{figure}

\begin{figure}
\center
\subfigure[Sliced Wasserstein losses for the different priors over the epochs.]{\includegraphics[width=0.4\textwidth]{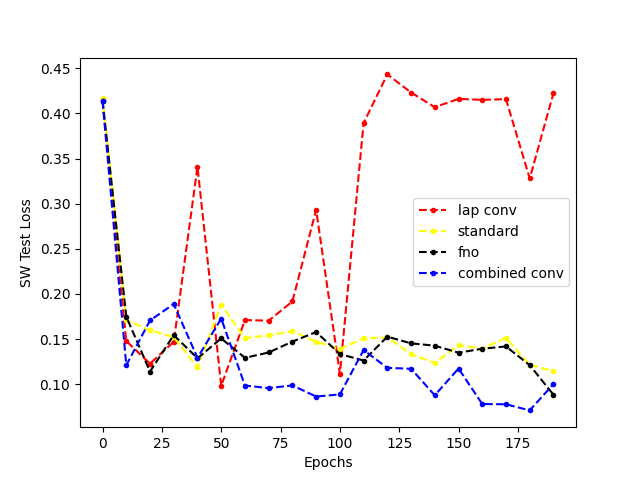}}
\subfigure[Sliced Wasserstein losses for the different priors over the epochs for a different seed.]{\includegraphics[width=0.4\textwidth]{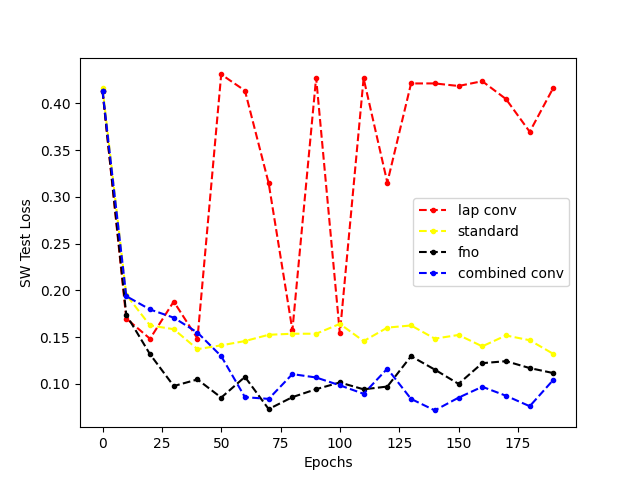}}
\caption{Sliced Wasserstein curves for the MNIST example.}
\label{fig:sw_mnist}
\end{figure}

\subsection{Multilevel training: MNIST}
\label{5.3}
We demonstrate that the weights trained on coarse images using our multilevel diffusion framework provide good initial guesses for fine-level training and lead to reductions in training time.
\begin{figure}
\includegraphics[width=0.47\textwidth]{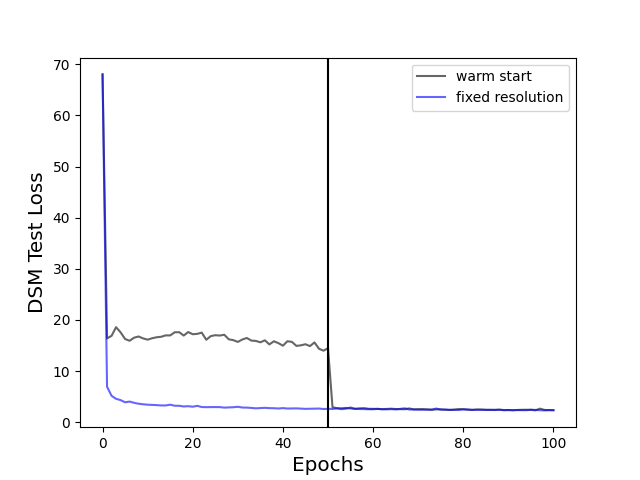}
\includegraphics[width=0.47\textwidth]{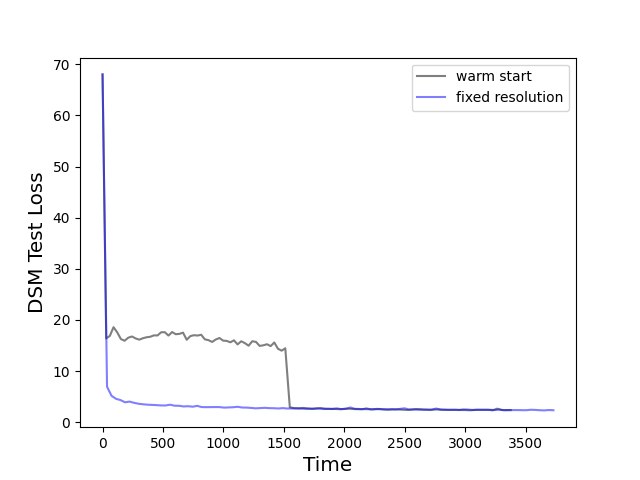}
\caption{DSM loss curve on the MNIST test set. Left: warm starting method versus the cold start in terms of epochs. Right: warm start vs cold start in terms of time. The black vertical line depicts the training at the finer resolution, i.e., from epoch 50 we train on $64 \times 64$.}
\label{fig:mnist_loss}
\end{figure}

To this end, we rescale MNIST to resolution $64 \times 64$ and compare training a given FNO on the $64\times 64$ images for 100 epochs to a coarse-to-fine training starting with  50 epochs on resolution $32 \times 32$ followed by 50 epochs on the full resolution.  Here we chose the combined prior with a smaller number of Fourier modes.  In \cref{fig:mnist_loss}, one can see the loss curves with a warm start and fixed resolution. The left plot shows that both strategies lead to comparable test losses while the multilevel scheme leads to a reduction of training times. In our example, fine-level training took about 9 percent longer than the multilevel strategy; see right subplot in \cref{fig:mnist_loss}. Here, we only account for training time (and not evaluation time).  
\begin{figure}
\center
\subfigure[Samples at resolution 128 at fixed resolution]{\includegraphics[width=0.3\textwidth]{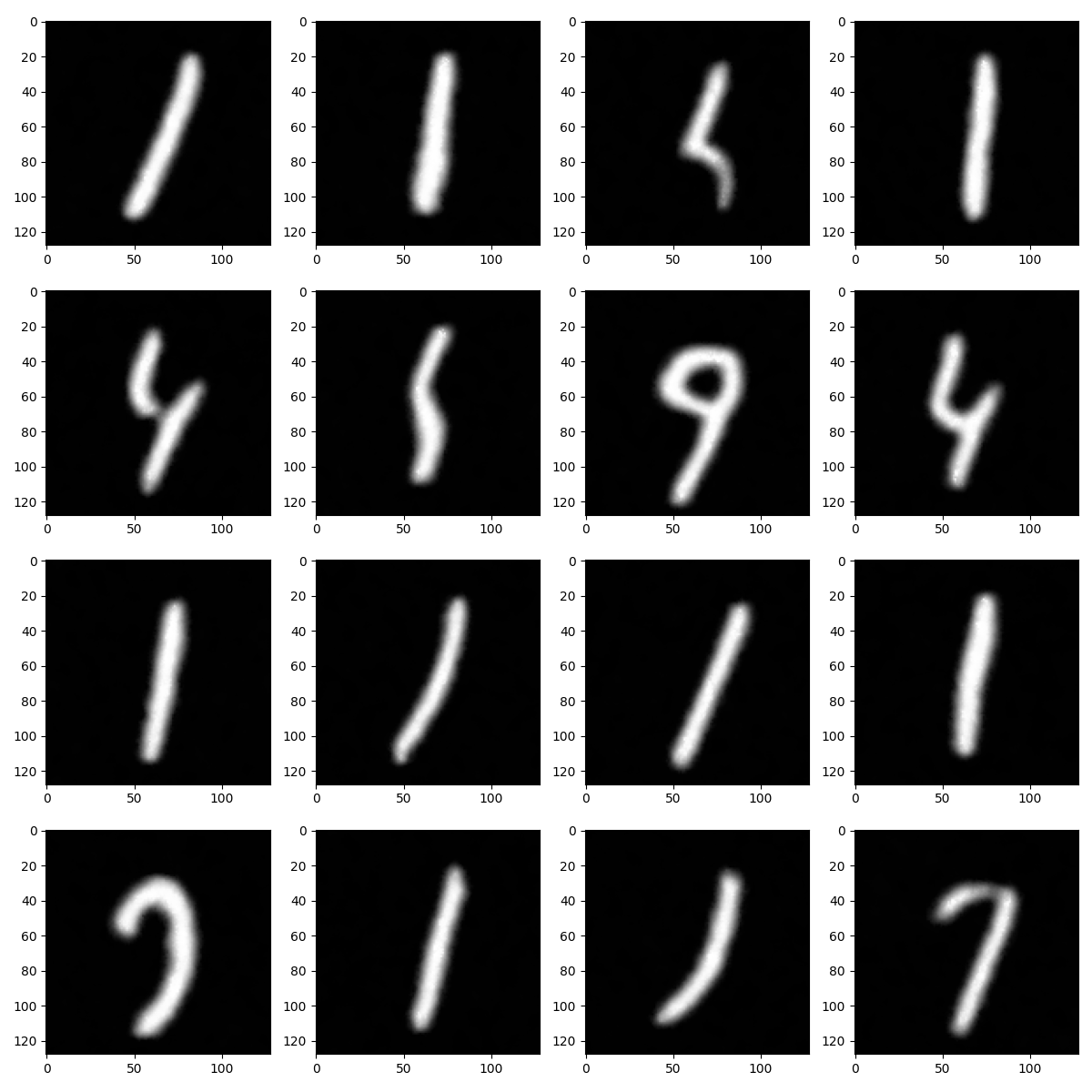}}
\subfigure[Samples at resolution 128 with multilevel warmstart]{\includegraphics[width=0.3\textwidth]{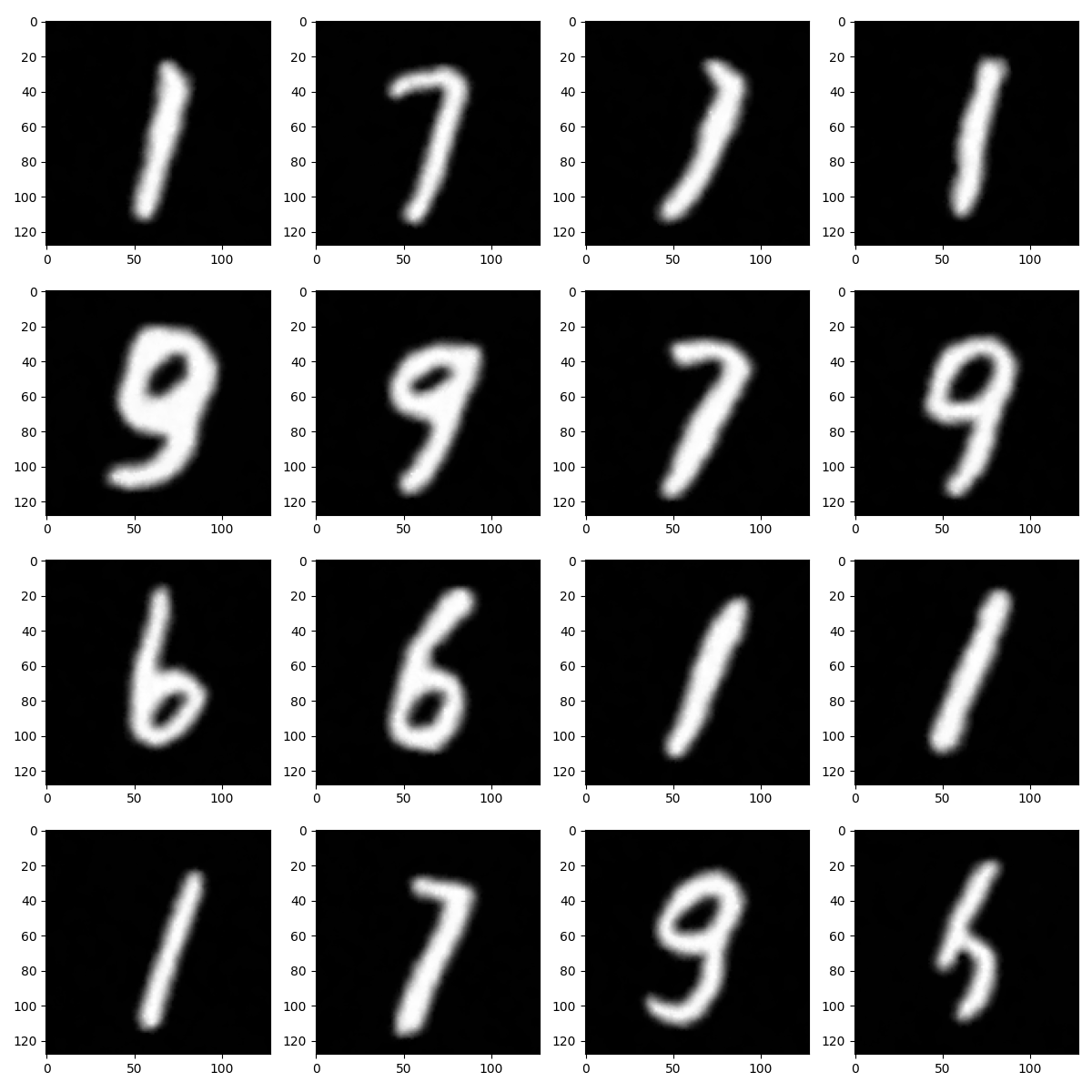}}
\caption{Samples generated using the model trained on $64\times64$ MNIST images (left) and the model trained with a multilevel strategy on resolutions $32\times32$ and $64\times 64$ (right). }
\label{fig:warmstart_mnist}

\end{figure}

Although both strategies achieve a similar test loss, visual inspection of samples generated at resolution $128\time128$ provided in \cref{fig:warmstart_mnist} suggests that the multilevel strategy improves sample quality in this experiment.

\subsection{2D Reaction-Diffusion Equation}
\label{5.2}

We corroborate our findings from the above experiments using the dataset obtained from a nonlinear PDE described in \cite[Section D.9]{takamoto2023pdebench}. 
The target distribution is the push-forward of a Gaussian through a 2D reaction-diffusion equation with Neumann boundary condition. 
Due to the nonlinear reaction function, the target distribution is non-Gaussian.
To obtain samples, we use a similar setup to \cite[Section D.9]{takamoto2023pdebench} and discretize $y_1 \in (-1,1)$, $y_2 \in (-1,1)$ using a regular mesh with 128 cells in each direction and 101 equidistant time steps for the temporal variable $t \in (0,5]$. The results presented in the section use the solution at the 50th time step.

\paragraph{Operator architecture}
Using the U-Net architecture to parameterize the score function for the PDE dataset leads to drawbacks similar to those above. 
In \cref{fig:pde3}, we compare the U-Net architecture to an FNO architecture trained on coarse images of size $32\times32$ with the same Bessel prior. 
The DSM losses evaluated using 300 validation images at the different resolutions decay more consistently with the FNO architecture than those of the U-Net. 
For the latter, the DSM loss on the finest resolution increases during training.  
Hence, we now focus on the FNO architecture.  See the supplementary materials for more details on the hyperparameter search. 

\begin{figure}[t]
\centering
\begin{tikzpicture}[scale=1]

\node[] at (-0.4,0.8){\includegraphics[width=.245\textwidth]{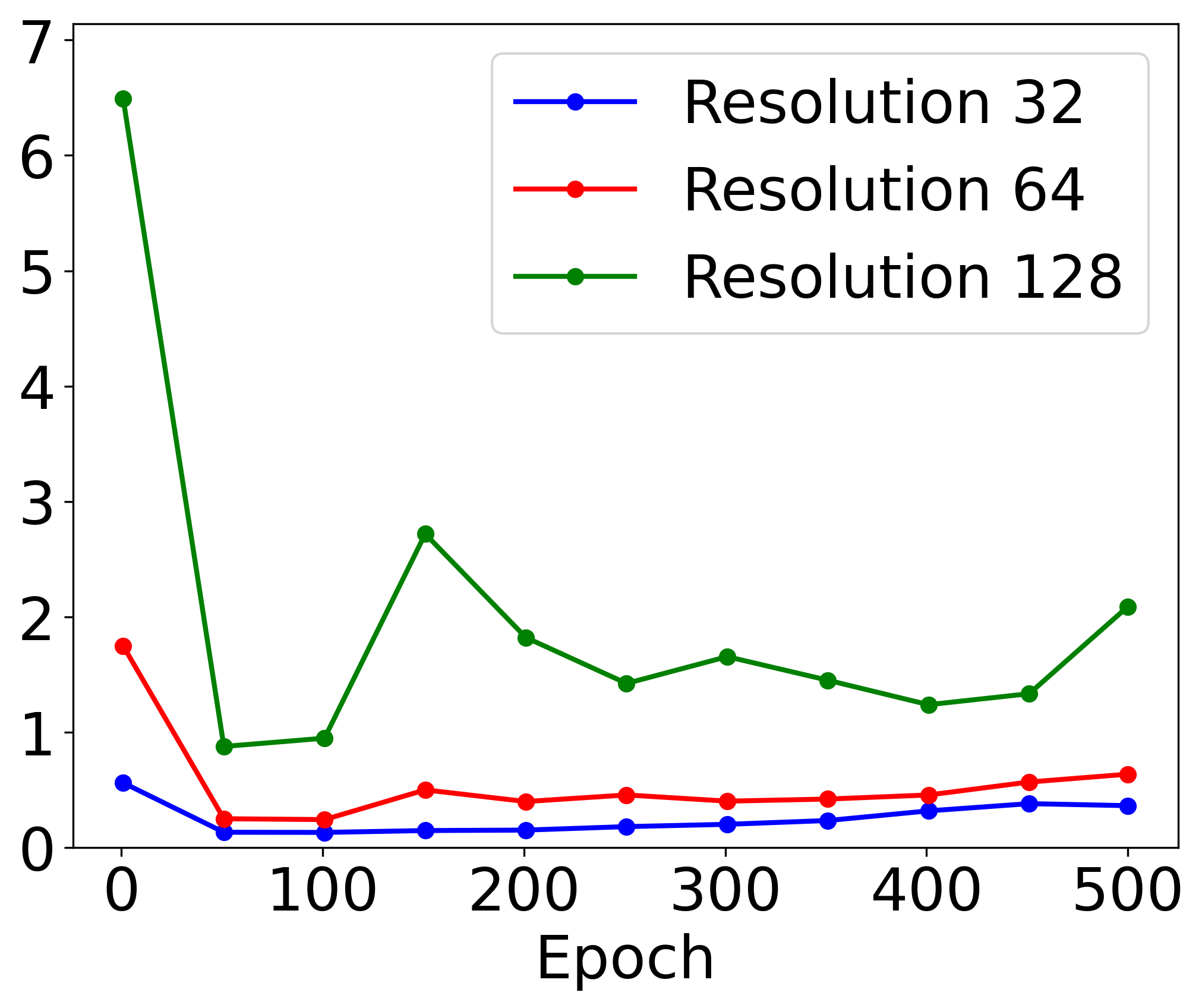}};
\node[label,rotate=90, font=\sffamily] at (-2.6,1.2){FNO Loss};
\node[] at (3.5,1){\includegraphics[trim={1.2cm 1.2cm 1.2cm 1.2cm},clip,width=.22\textwidth]{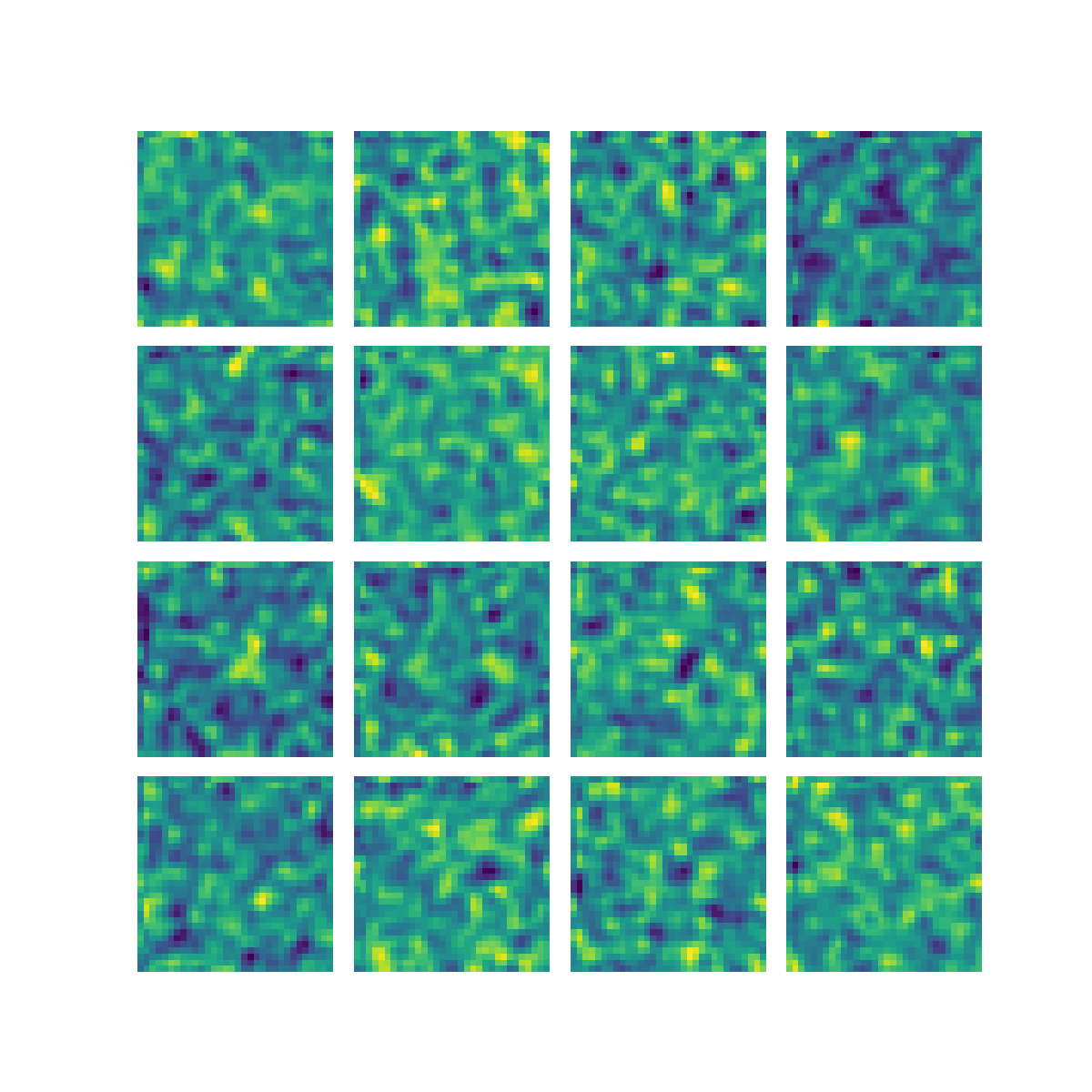}};

\node[] at (6.8,1) {\includegraphics[trim={1.2cm 1.2cm 1.2cm 1.2cm},clip, width=.22\textwidth]{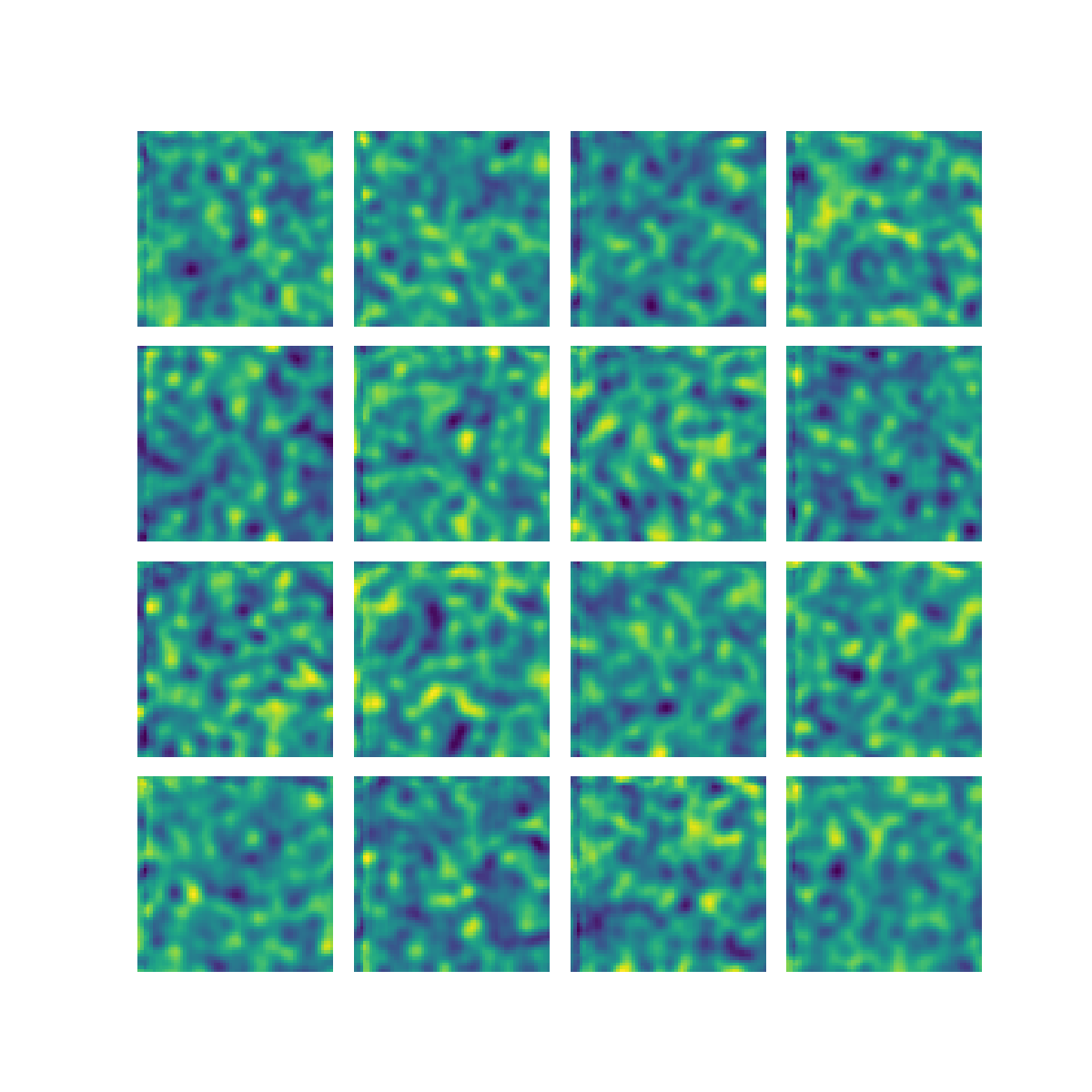}};

\node[] at (10.1,1) {\includegraphics[trim={1.2cm 1.2cm 1.2cm 1.2cm},clip, width=.22\textwidth]{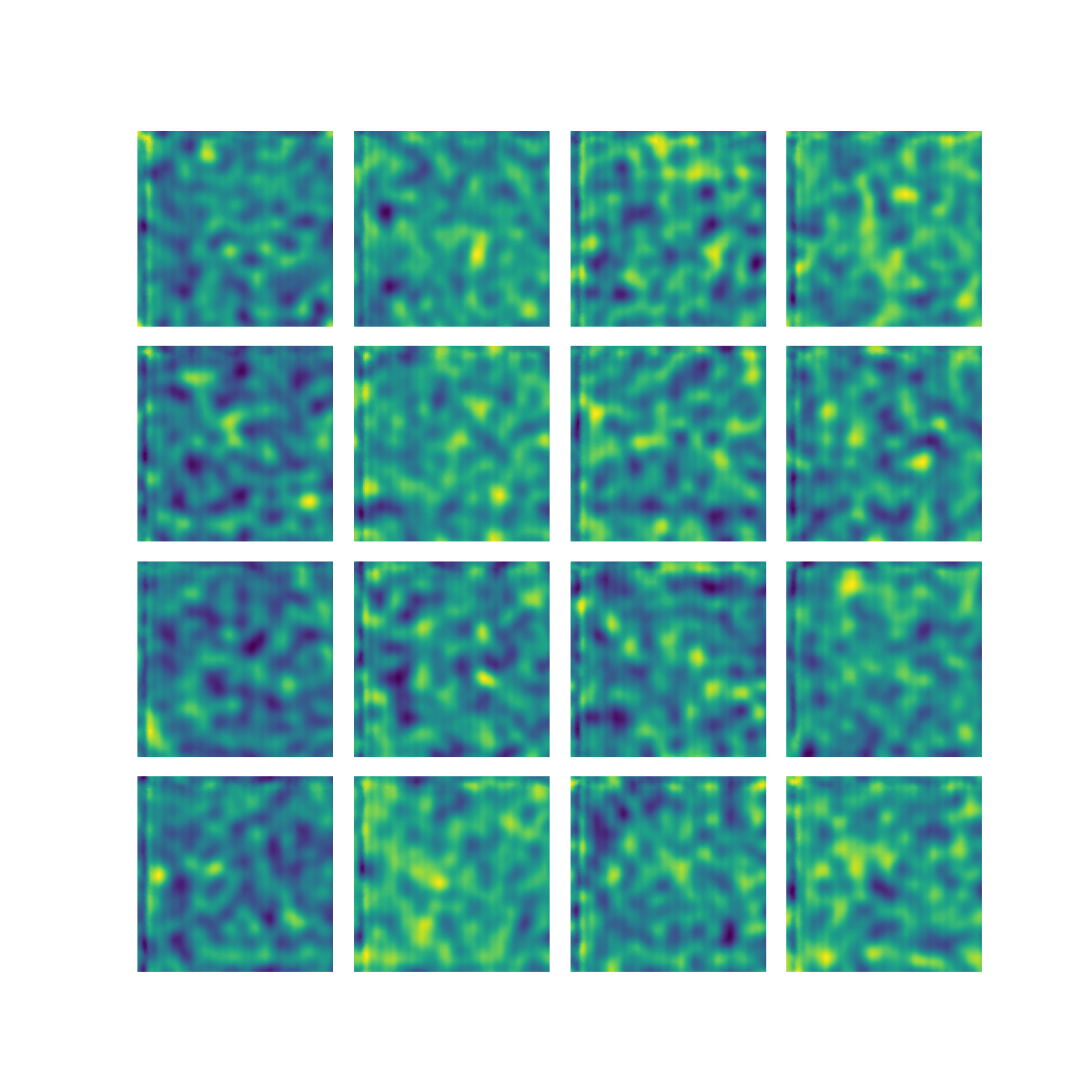}};

\node[] at (-0.4,-3.2){\includegraphics[width=.26\textwidth]{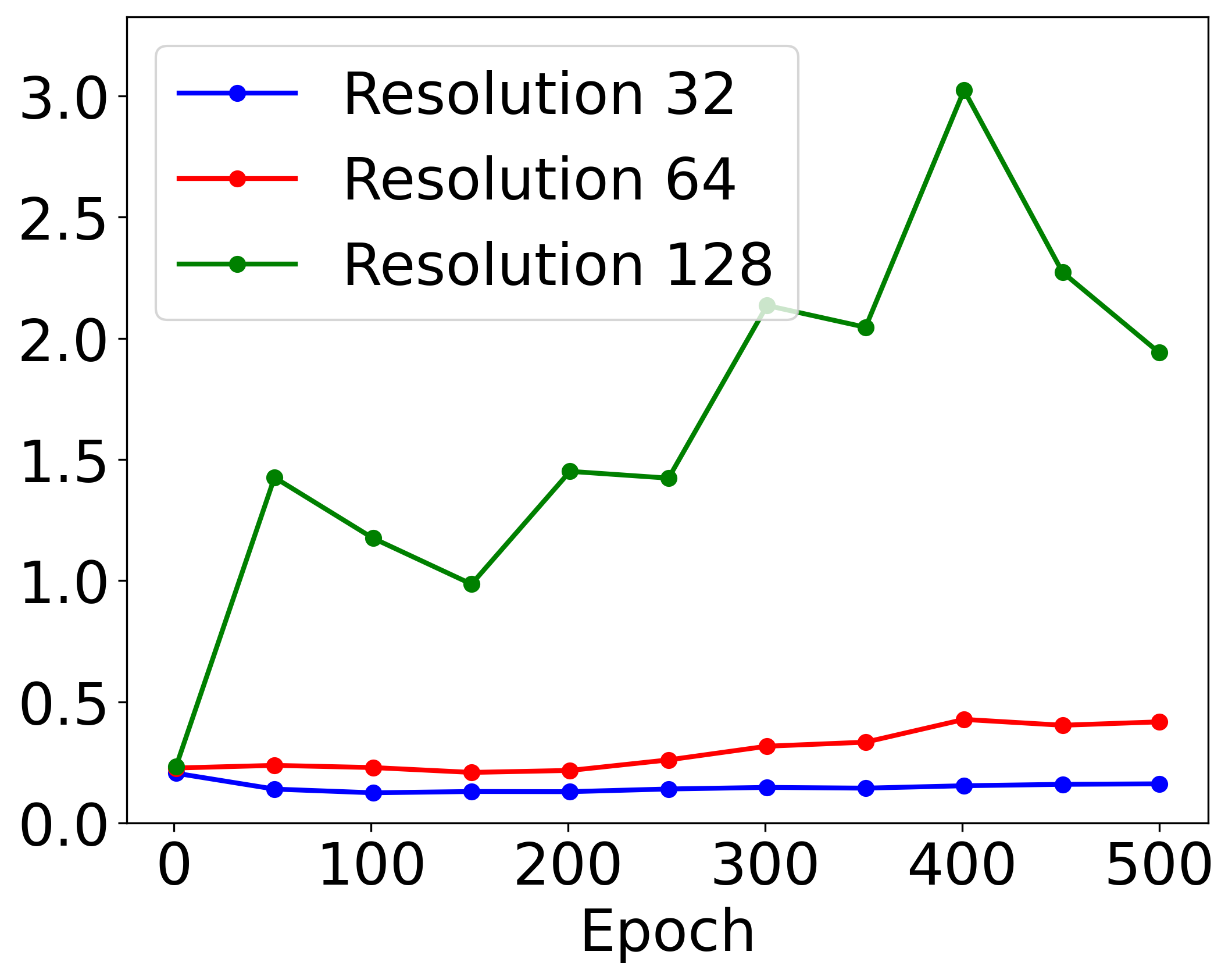}};
\node[label, font=\sffamily] at (3.5,-5){32};
\node[label,rotate=90, font=\sffamily] at (-2.6,-3.2){U-Net Loss};
\node[] at (3.5,-3){\includegraphics[trim={1.2cm 1.2cm 1.2cm 1.2cm},clip,width=.22\textwidth]{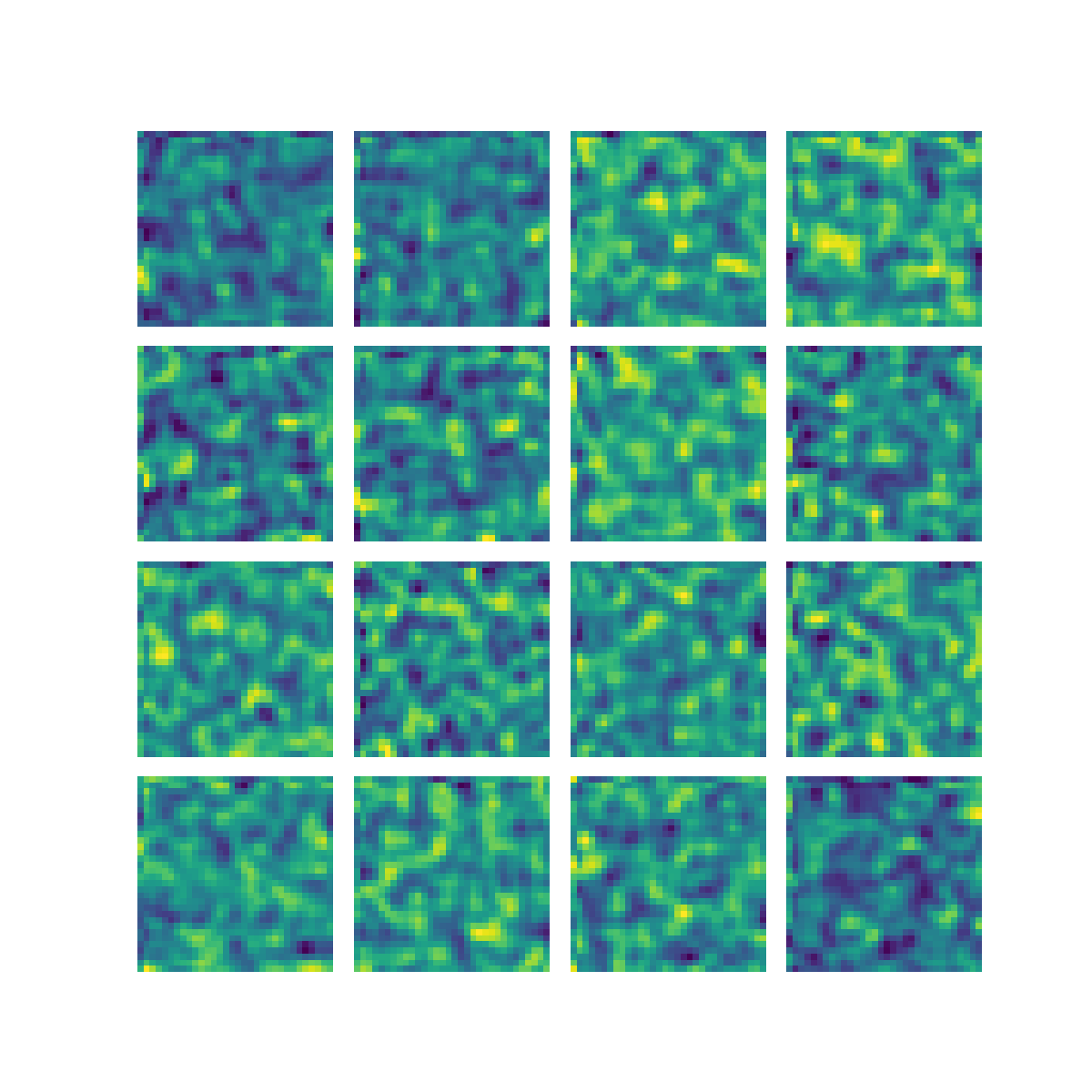}};
\node[label, font=\sffamily] at (3.5,-5){32};

\node[] at (6.8,-3) {\includegraphics[trim={1.2cm 1.2cm 1.2cm 1.2cm},clip, width=.22\textwidth]{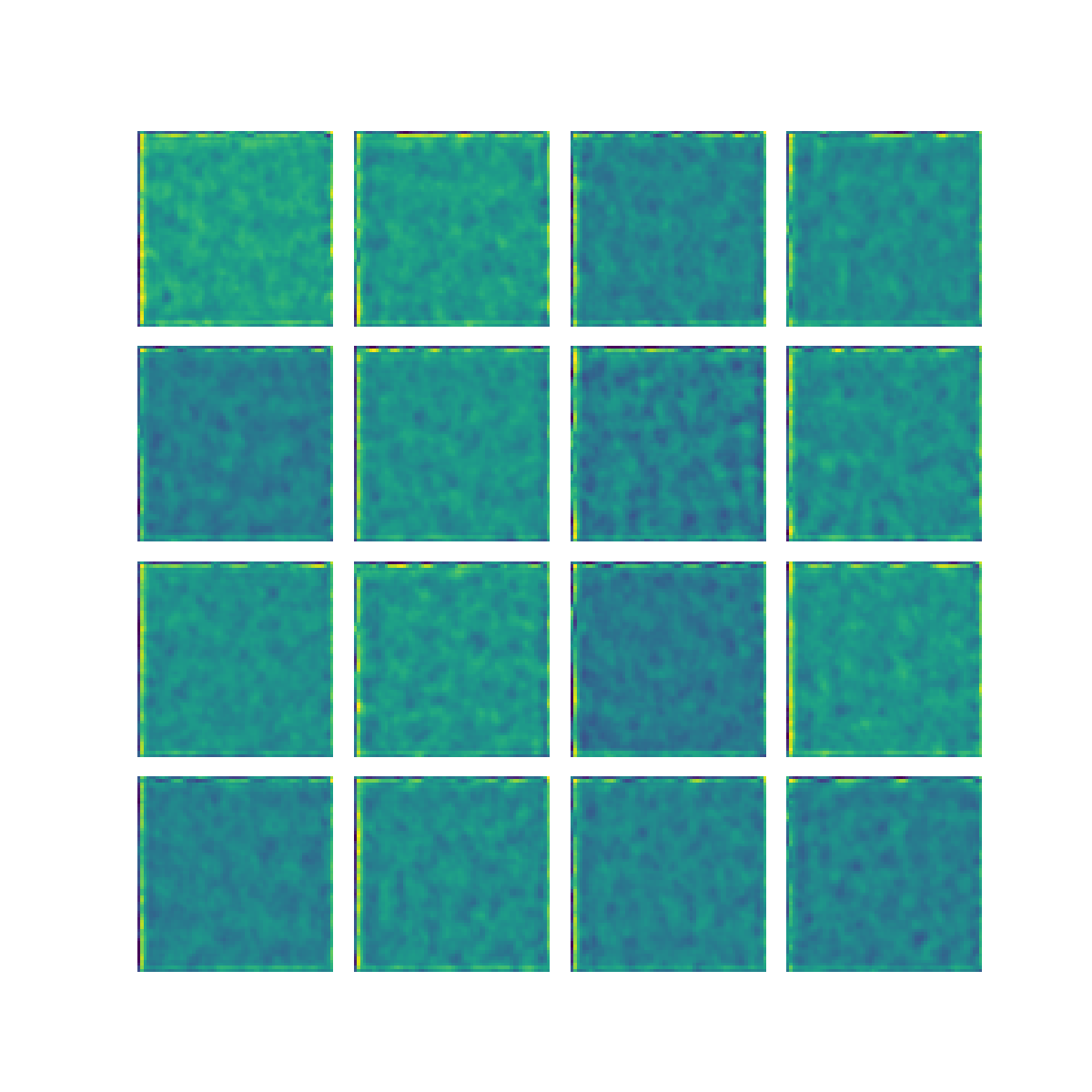}};
\node[label, font=\sffamily] at (6.8,-5){64};
\node[] at (10.1,-3) {\includegraphics[trim={1.2cm 1.2cm 1.2cm 1.2cm},clip, width=.22\textwidth]{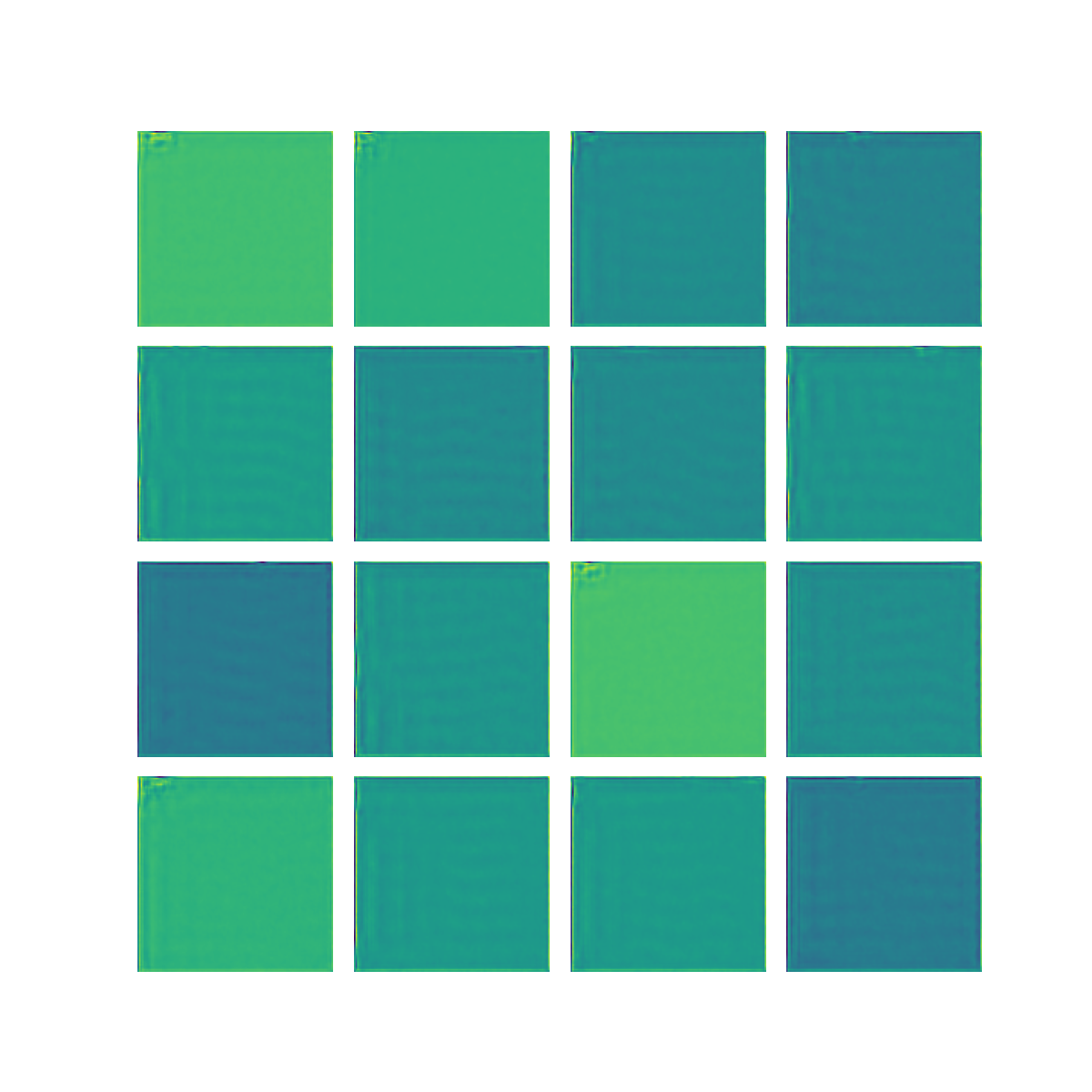}};
\node[label, font=\sffamily] at (10.2,-5){128};

\end{tikzpicture}

\caption{Results of FNO (Top row) and U-Net (Bottom row) architectures trained at resolution $32 \times 32$ using the Bessel prior \eqref{eq: bessel} with $\gamma_2 = 8$, $k = 1.1$. Left column: Loss curves evaluated at different resolutions from the FNO and U-Net models trained at resolution $32 \times 32$. Other columns: The generated image samples at  resolutions $32\times32$, $64\times64,$ and $128\times128$.
}
\label{fig:pde3}
\end{figure}
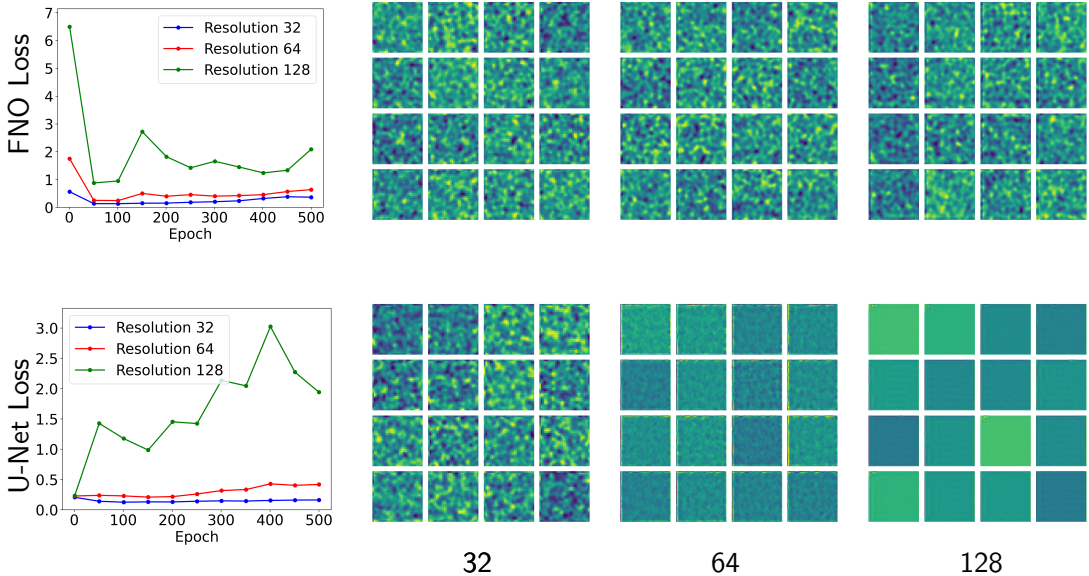 

\paragraph{Comparison of different priors}
We compare the performance of models trained with different priors on resolution $32\times 32$ to produce samples at higher resolutions. 
As for the other datasets, the standard Gaussian prior does not generalize well to higher resolutions; see, e.g., the samples of size $256 \times 256$ in \cref{fig:pde2}.
The trace-class covariance operators offer the flexibility of adjusting the eigenvalue decay and achieving a better intensity level in the generated images. In particular, we select suitable values for the scaling factor $\gamma_0$, $\gamma_1$, and $\gamma_2$, and the fractional power $k$ defined in \cref{sub:latent}, as we need appropriate choices to satisfy the regularity conditions in \cref{thm:main} and \cref{thm:multilevel}. These hyperparameters in the trace class operator influence the regularity of its associated Gaussian random field, and thus the smoothness of the SDE solutions. 
The combined prior with $\gamma_0 =0$ (i.e., the Laplacian prior) and the Bessel prior give better results in terms of sliced Wasserstein metric evaluated at resolution $128$.
 We include the hyperparameter search table and more details in the supplements. 
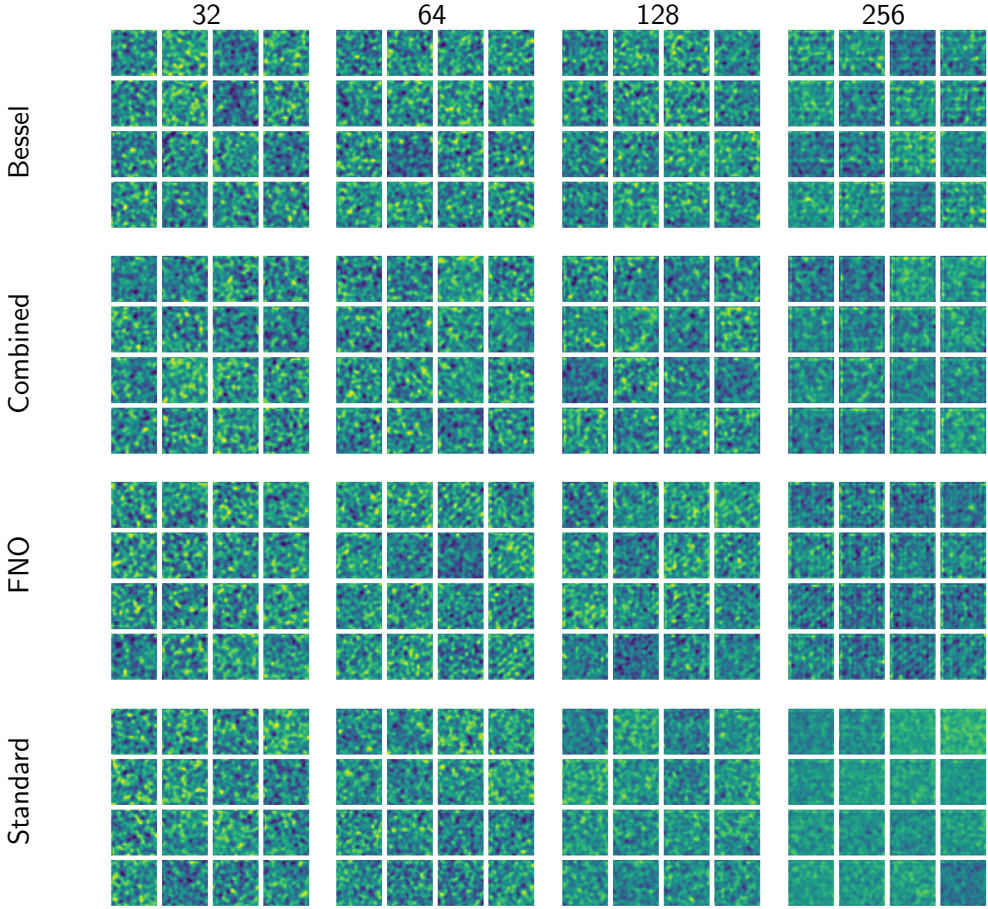
\begin{figure}[t]
\centering
\begin{tikzpicture}[scale=1]

\node[] at (3,1.2){\includegraphics[trim={1.2cm 1.2cm 1.2cm 1.2cm},clip,width=.2\textwidth]{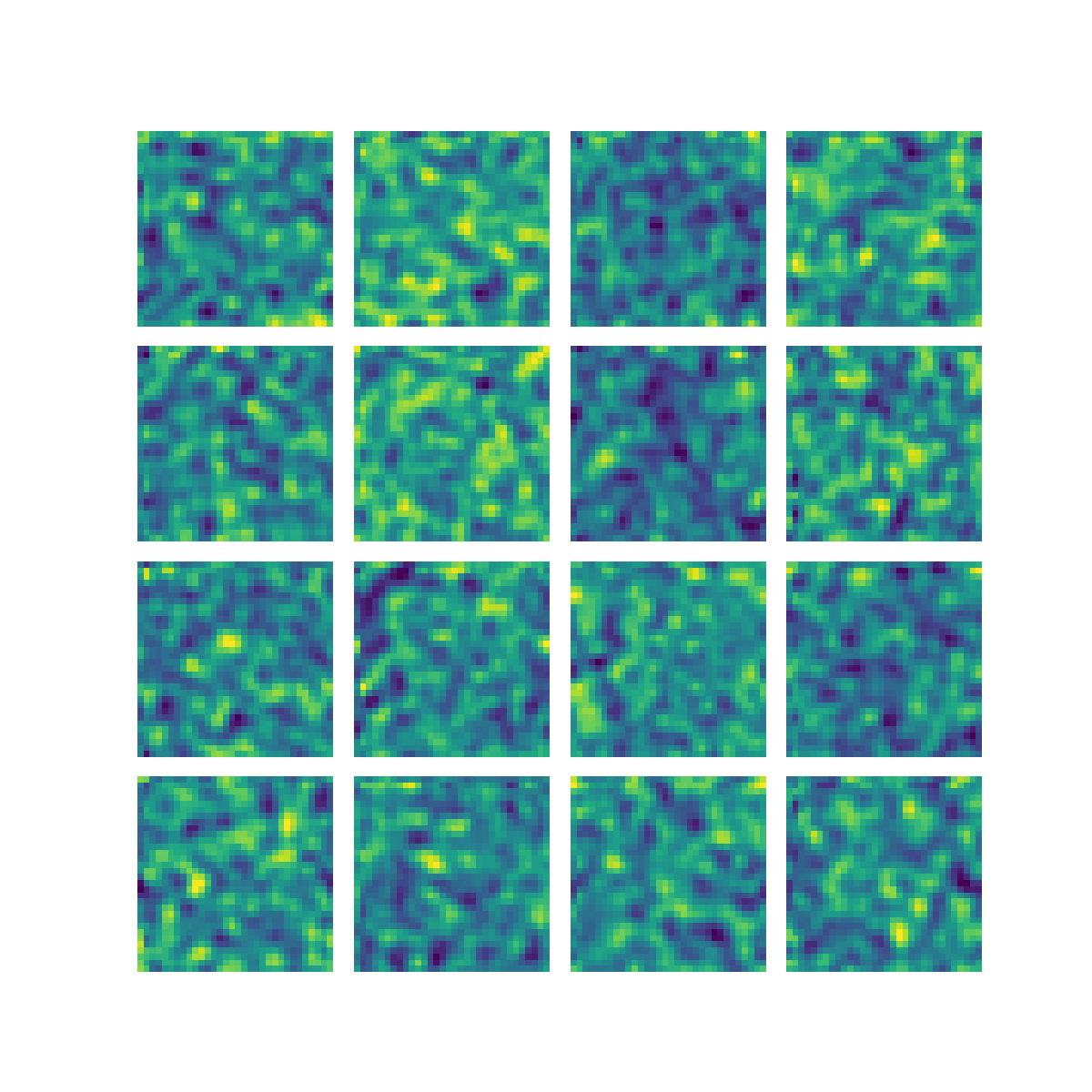}};
\node[label, font=\sffamily] at (3,2.7){32};

\node[] at (6,1.2) {\includegraphics[trim={1.2cm 1.2cm 1.2cm 1.2cm},clip, width=.2\textwidth]{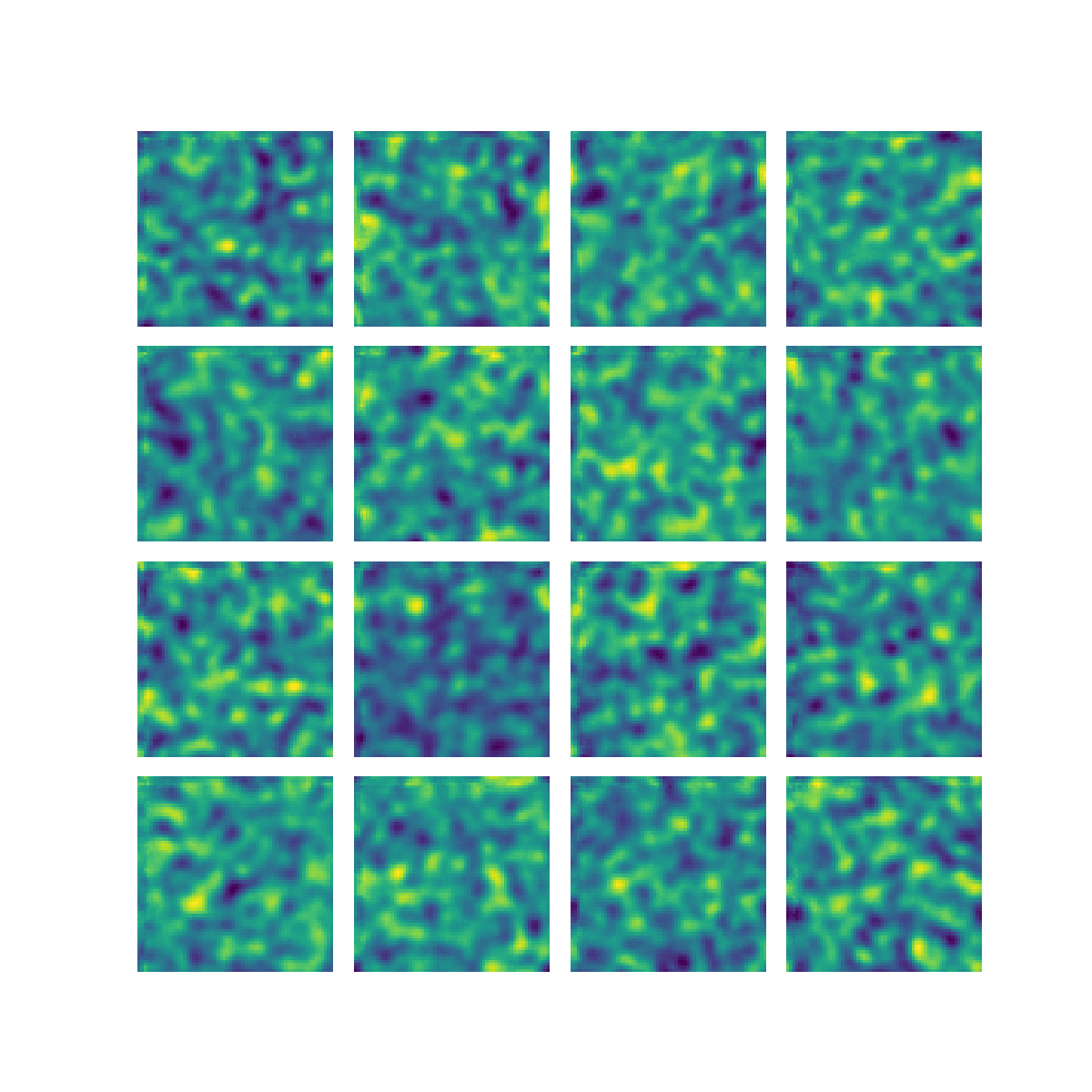}};
\node[label, font=\sffamily] at (6,2.7){64};

\node[] at (9,1.2) {\includegraphics[trim={1.2cm 1.2cm 1.2cm 1.2cm},clip, width=.2\textwidth]{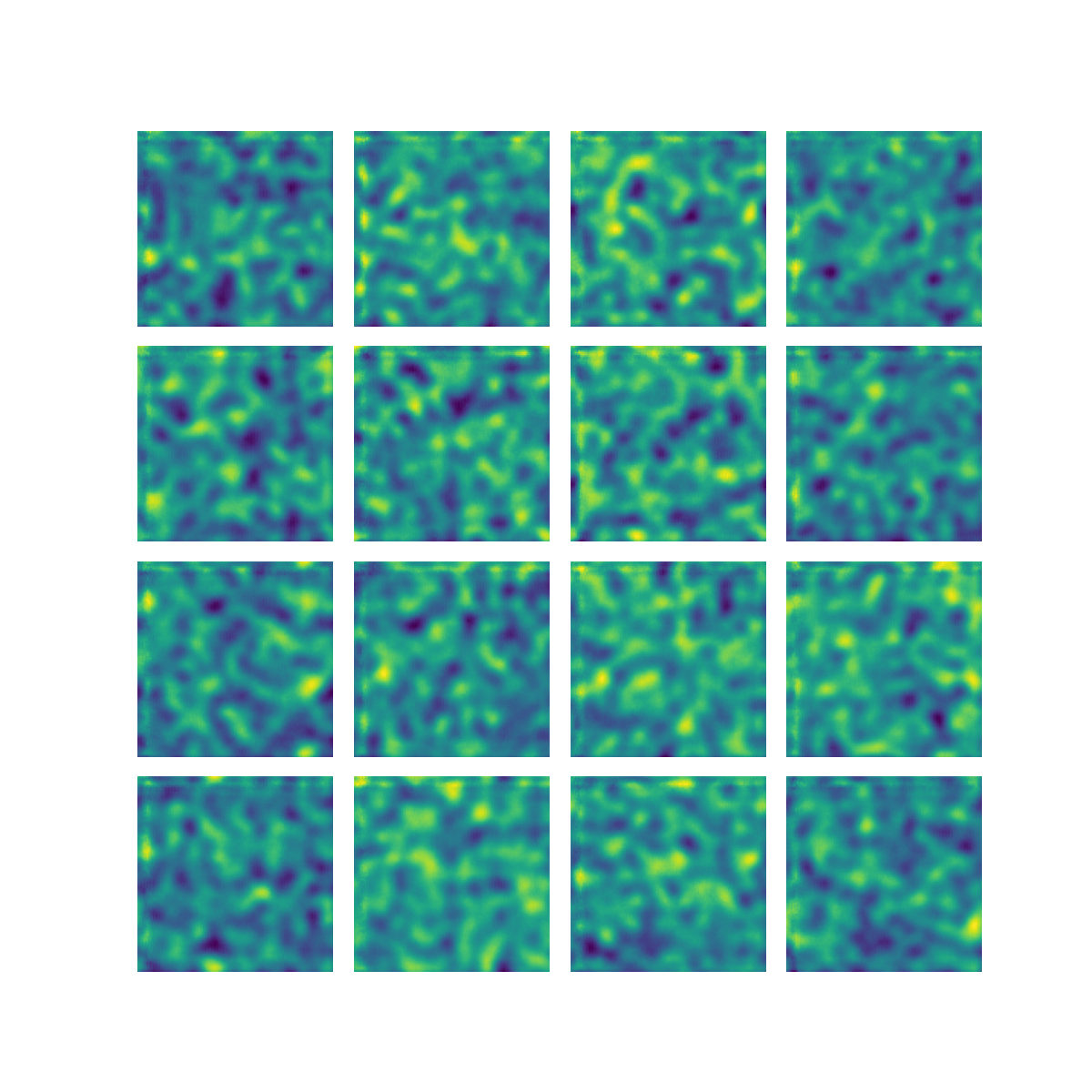}};
\node[label, font=\sffamily] at (9,2.7){128};

\node[] at (12,1.2) {\includegraphics[trim={1.2cm 1.2cm 1.2cm 1.2cm},clip, width=.2\textwidth]{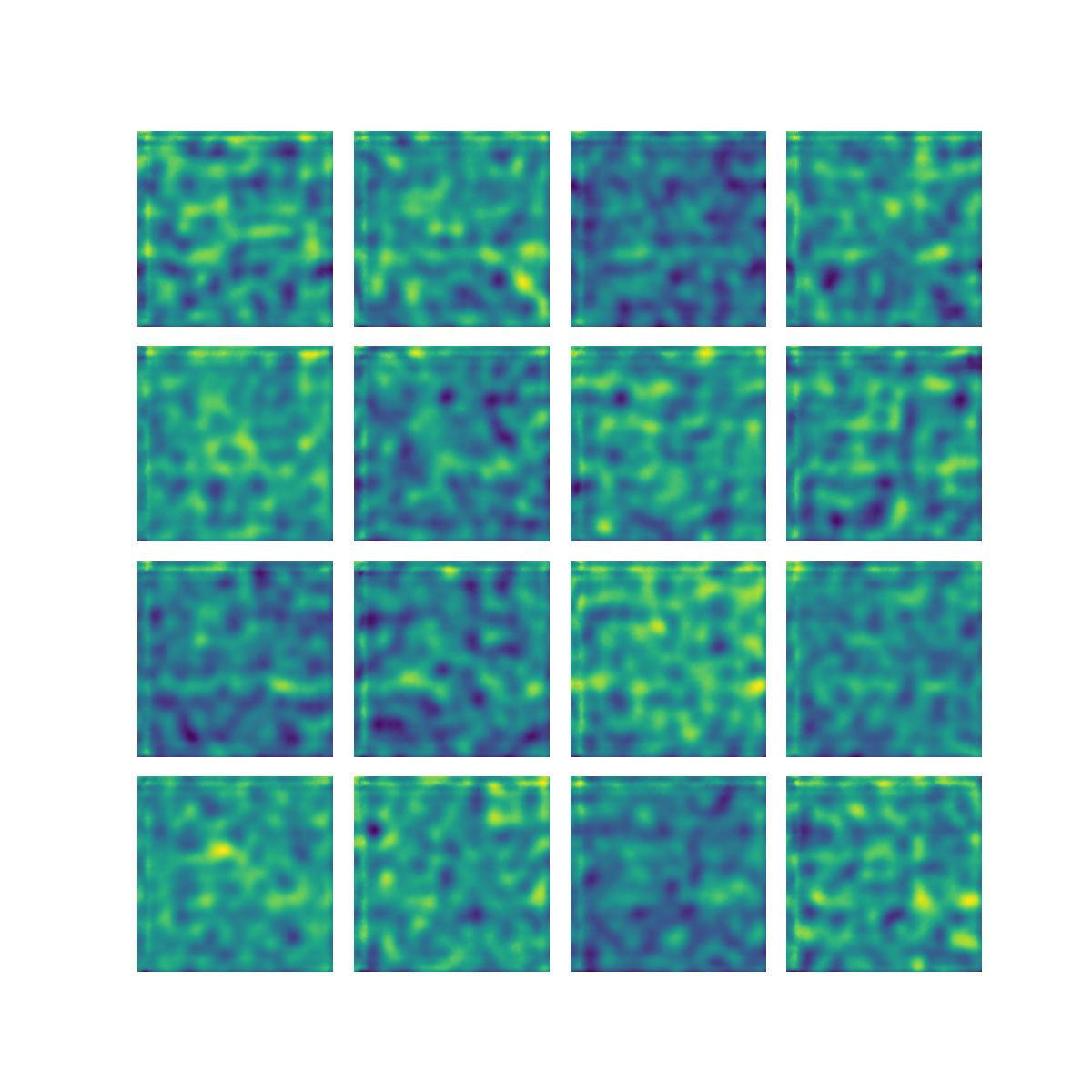}};
\node[label, font=\sffamily] at (12,2.7){256};

\node[label,rotate=90, font=\sffamily] at (0.5,1){Bessel};

\node[] at (3,-1.8){\includegraphics[trim={1.2cm 1.2cm 1.2cm 1.2cm},clip,width=.2\textwidth]{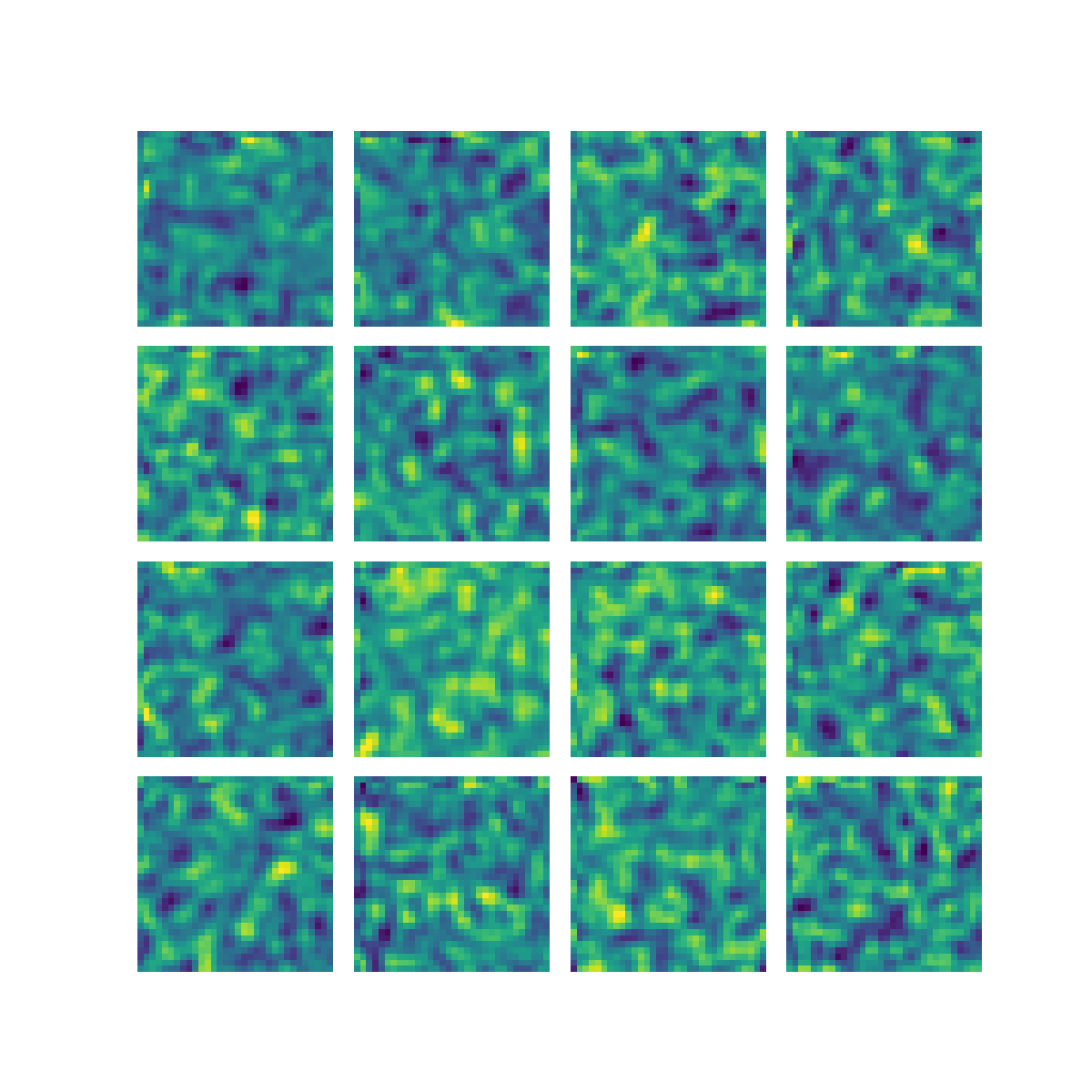}};

\node[] at (6,-1.8) {\includegraphics[trim={1.2cm 1.2cm 1.2cm 1.2cm},clip, width=.2\textwidth]{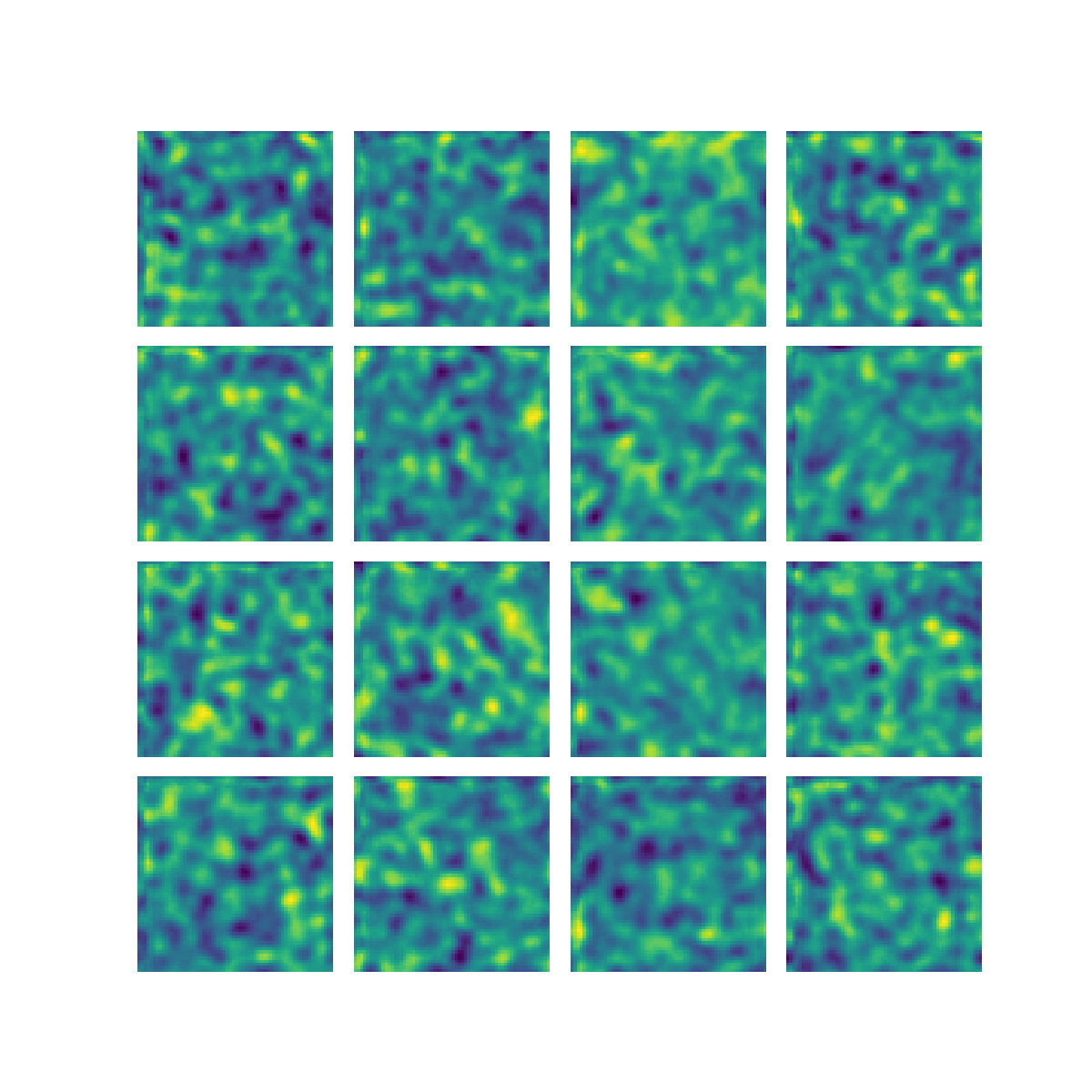}};

\node[] at (9,-1.8) {\includegraphics[trim={1.2cm 1.2cm 1.2cm 1.2cm},clip, width=.2\textwidth]{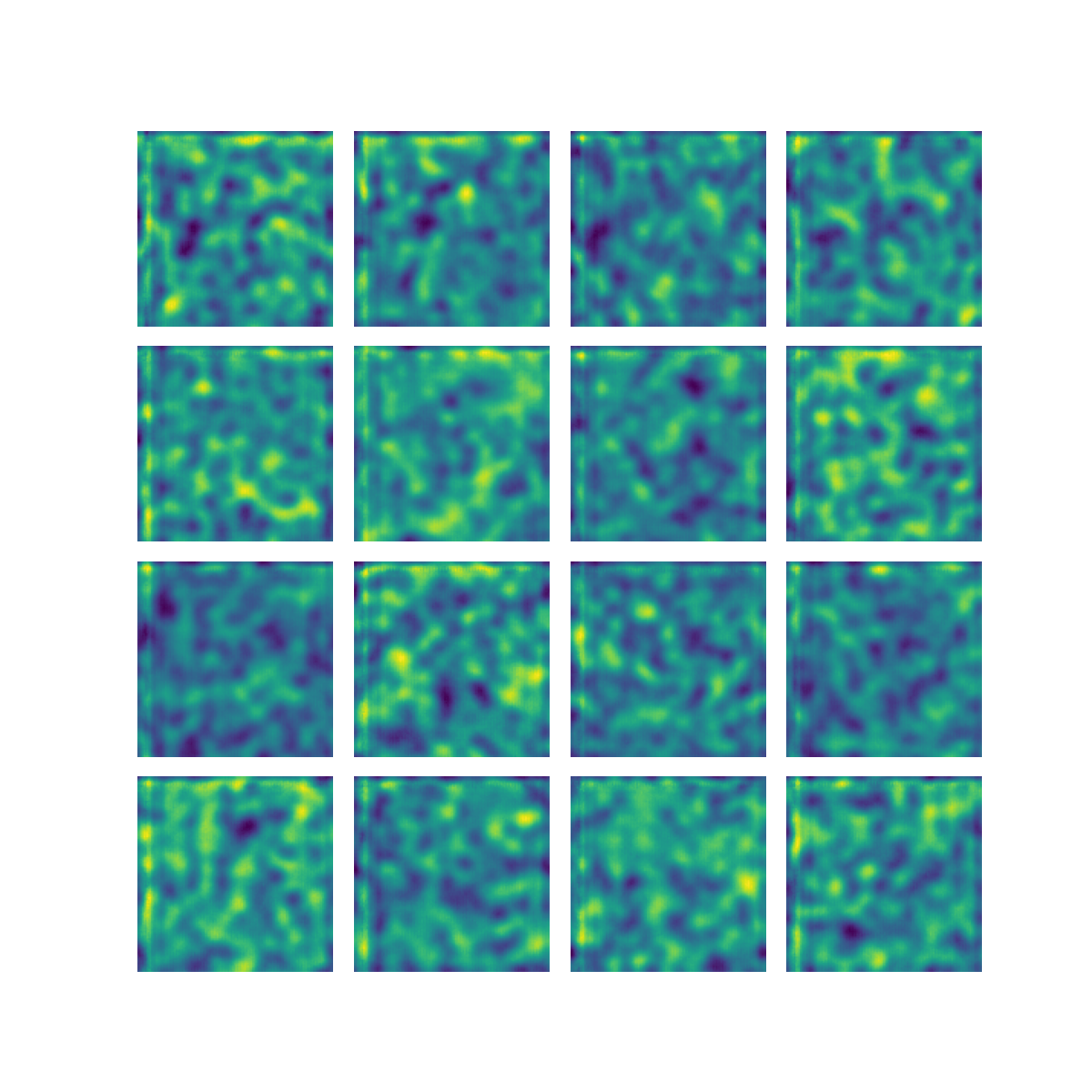}};

\node[] at (12,-1.8) {\includegraphics[trim={1.2cm 1.2cm 1.2cm 1.2cm},clip, width=.2\textwidth]{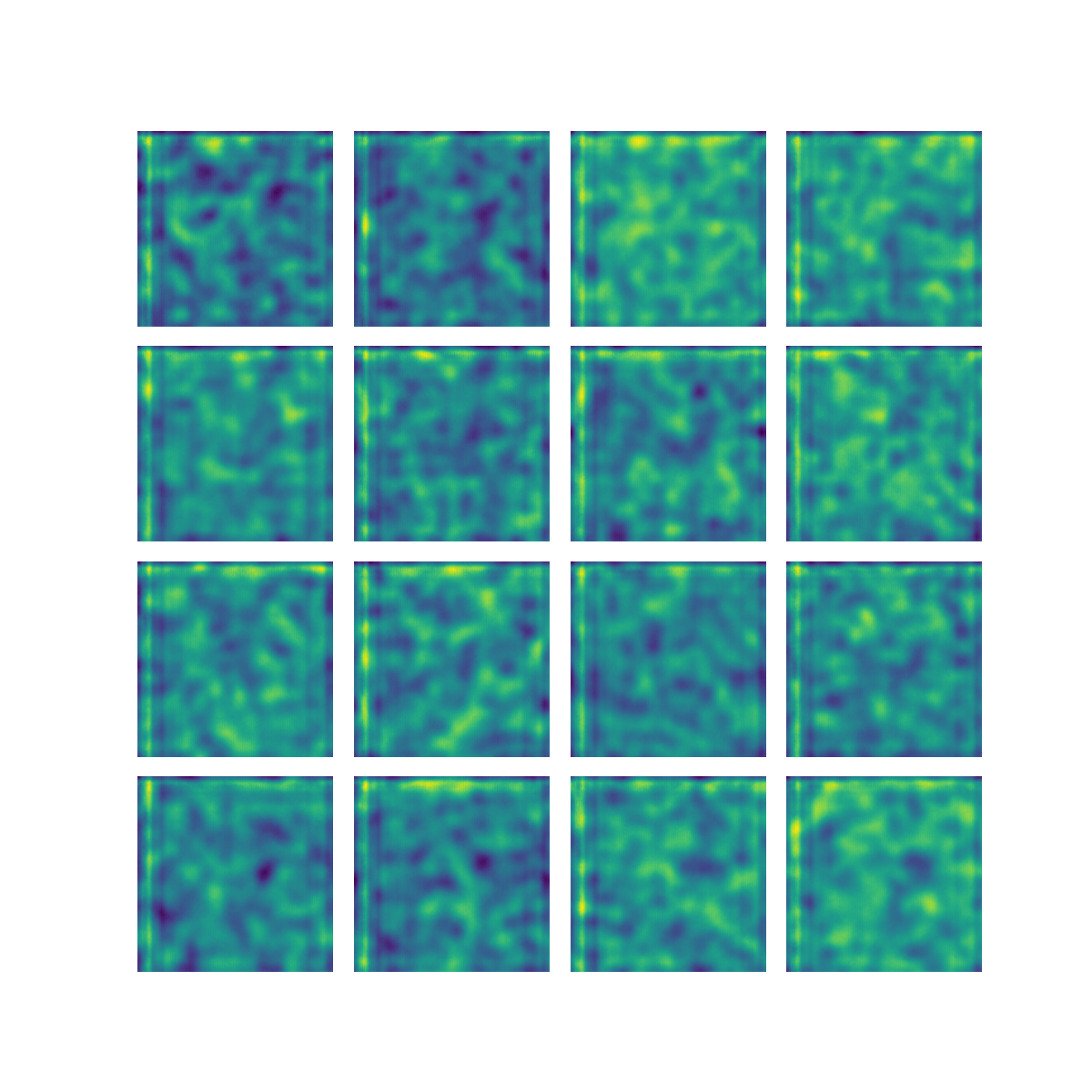}};

\node[label,rotate=90, font=\sffamily] at (0.5,-1.8){Combined};

\node[] at (3,-4.8){\includegraphics[trim={1.2cm 1.2cm 1.2cm 1.2cm},clip,width=.2\textwidth]{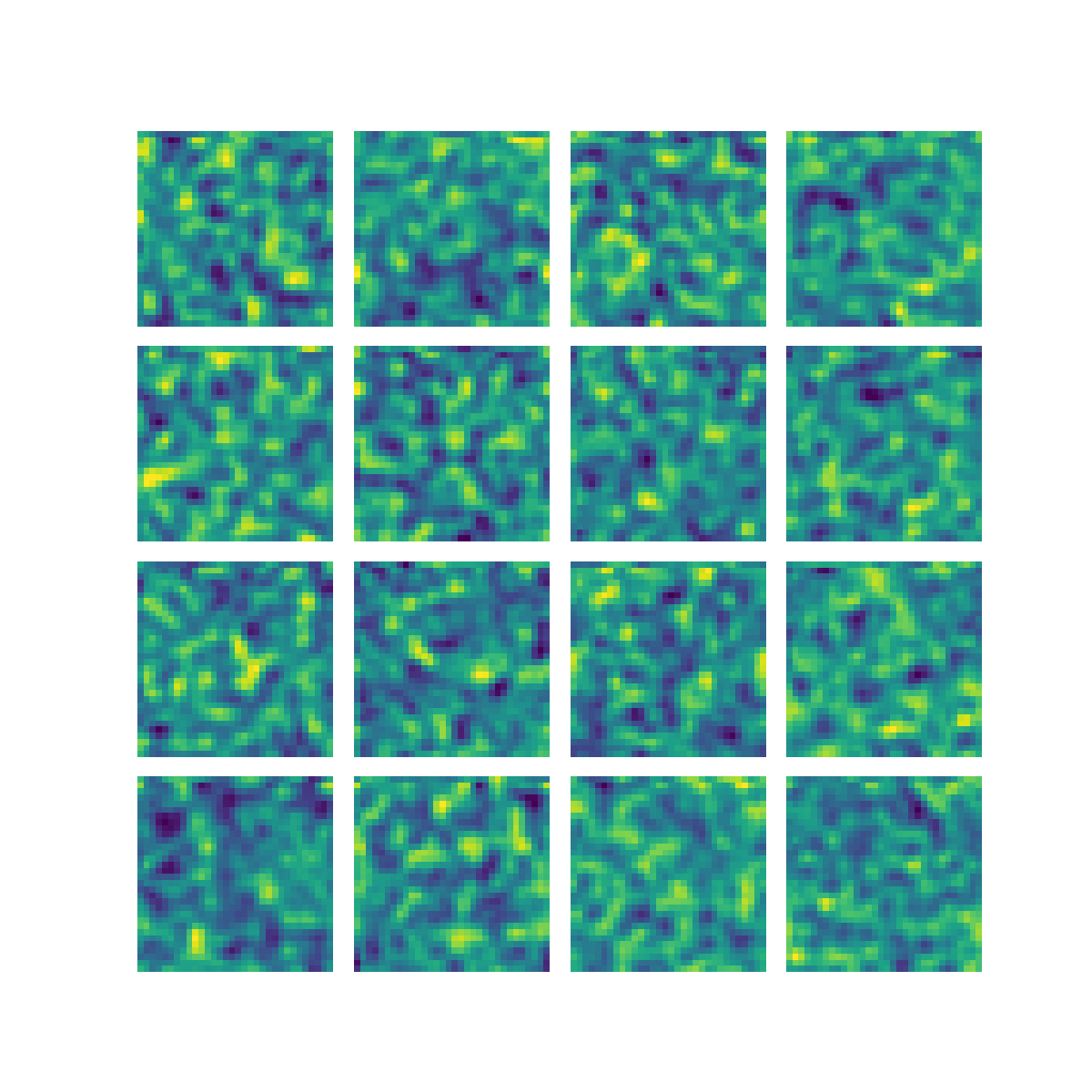}};

\node[] at (6,-4.8) {\includegraphics[trim={1.2cm 1.2cm 1.2cm 1.2cm},clip, width=.2\textwidth]{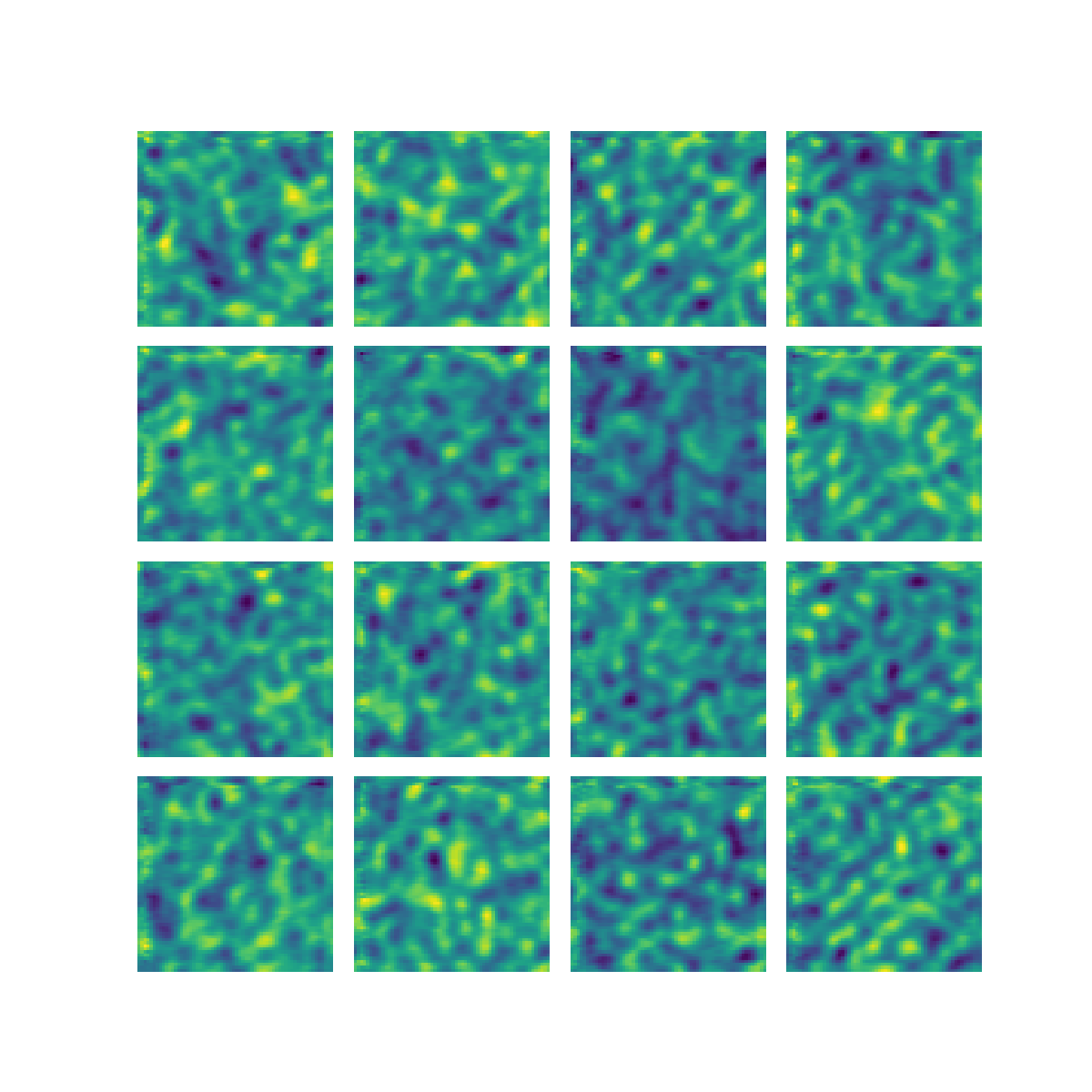}};

\node[] at (9,-4.8) {\includegraphics[trim={1.2cm 1.2cm 1.2cm 1.2cm},clip, width=.2\textwidth]{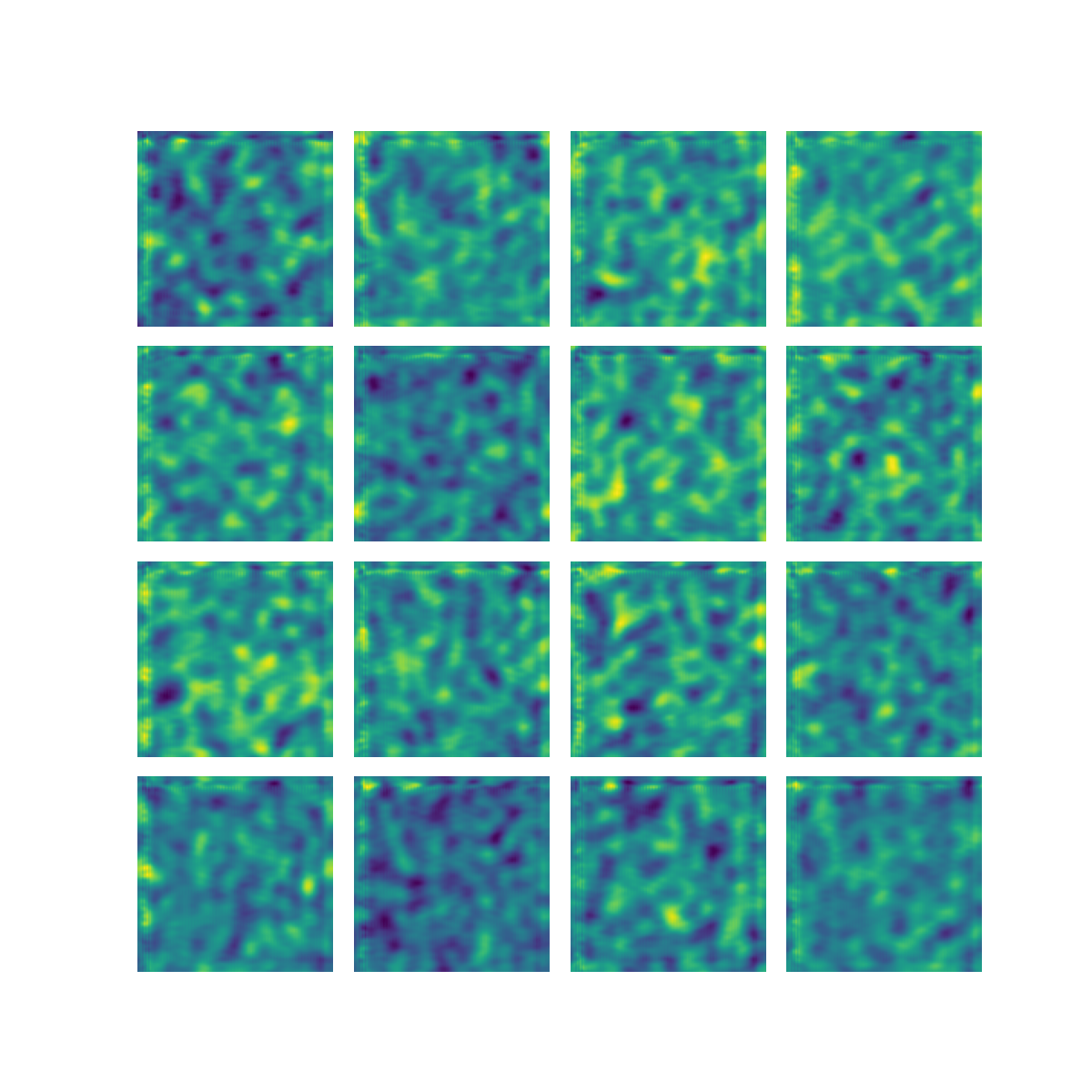}};

\node[] at (12,-4.8) {\includegraphics[trim={1.2cm 1.2cm 1.2cm 1.2cm},clip, width=.2\textwidth]{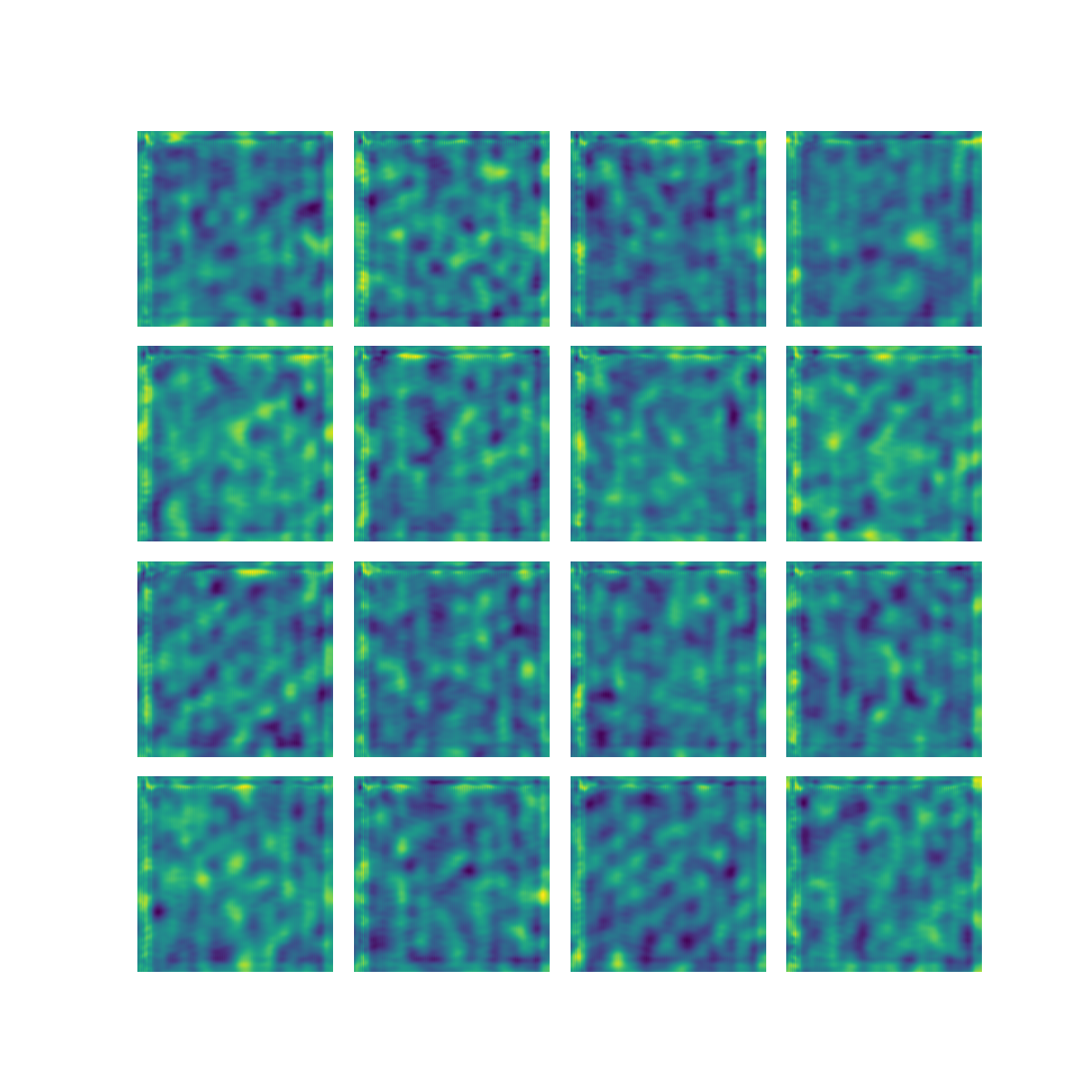}};

\node[label,rotate=90, font=\sffamily] at (0.5,-4.6){FNO};

\node[] at (3,-7.8){\includegraphics[trim={1.2cm 1.2cm 1.2cm 1.2cm},clip,width=.2\textwidth]{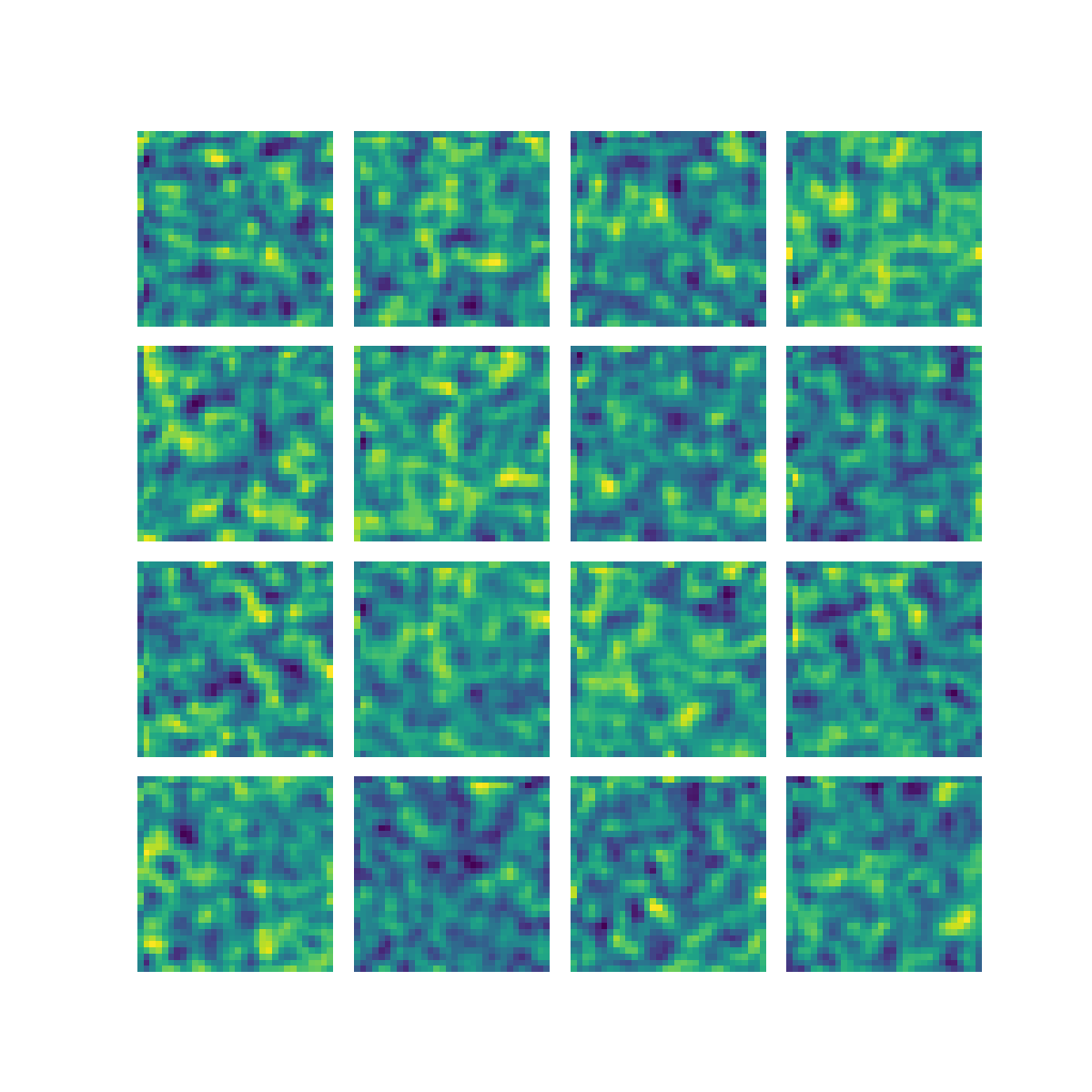}};

\node[] at (6,-7.8) {\includegraphics[trim={1.2cm 1.2cm 1.2cm 1.2cm},clip, width=.2\textwidth]{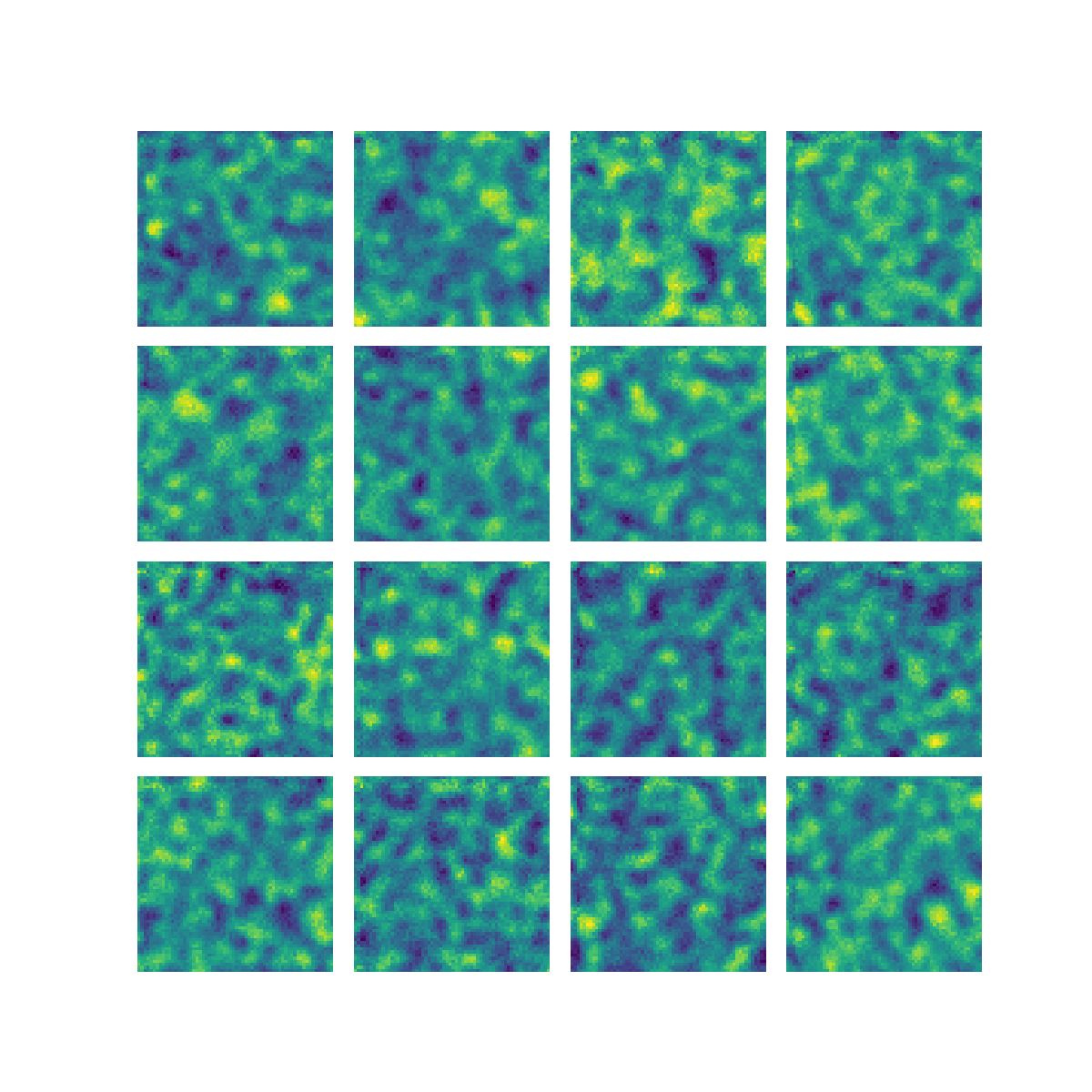}};

\node[] at (9,-7.8) {\includegraphics[trim={1.2cm 1.2cm 1.2cm 1.2cm},clip, width=.2\textwidth]{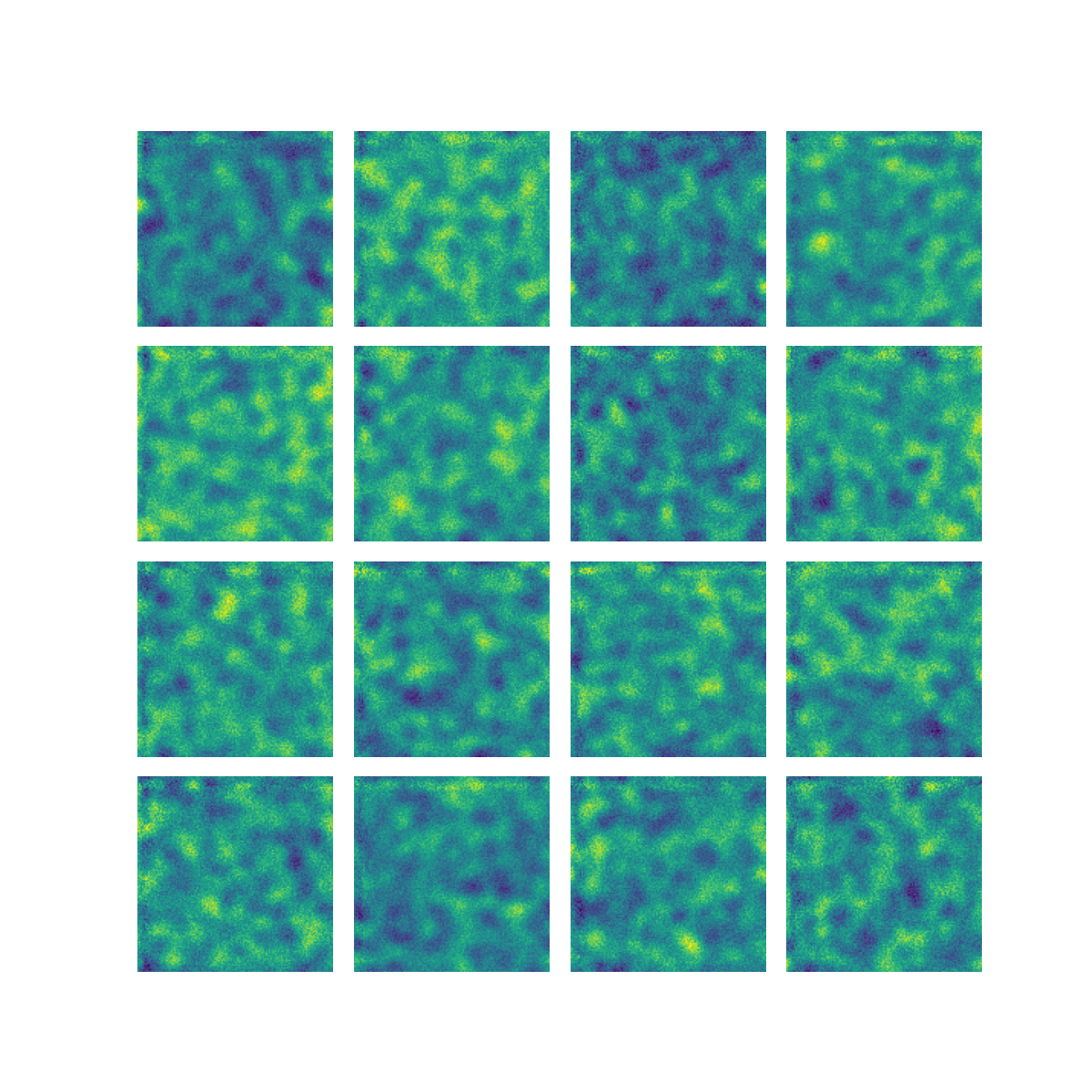}};

\node[] at (12,-7.8) {\includegraphics[trim={1.2cm 1.2cm 1.2cm 1.2cm},clip, width=.2\textwidth]{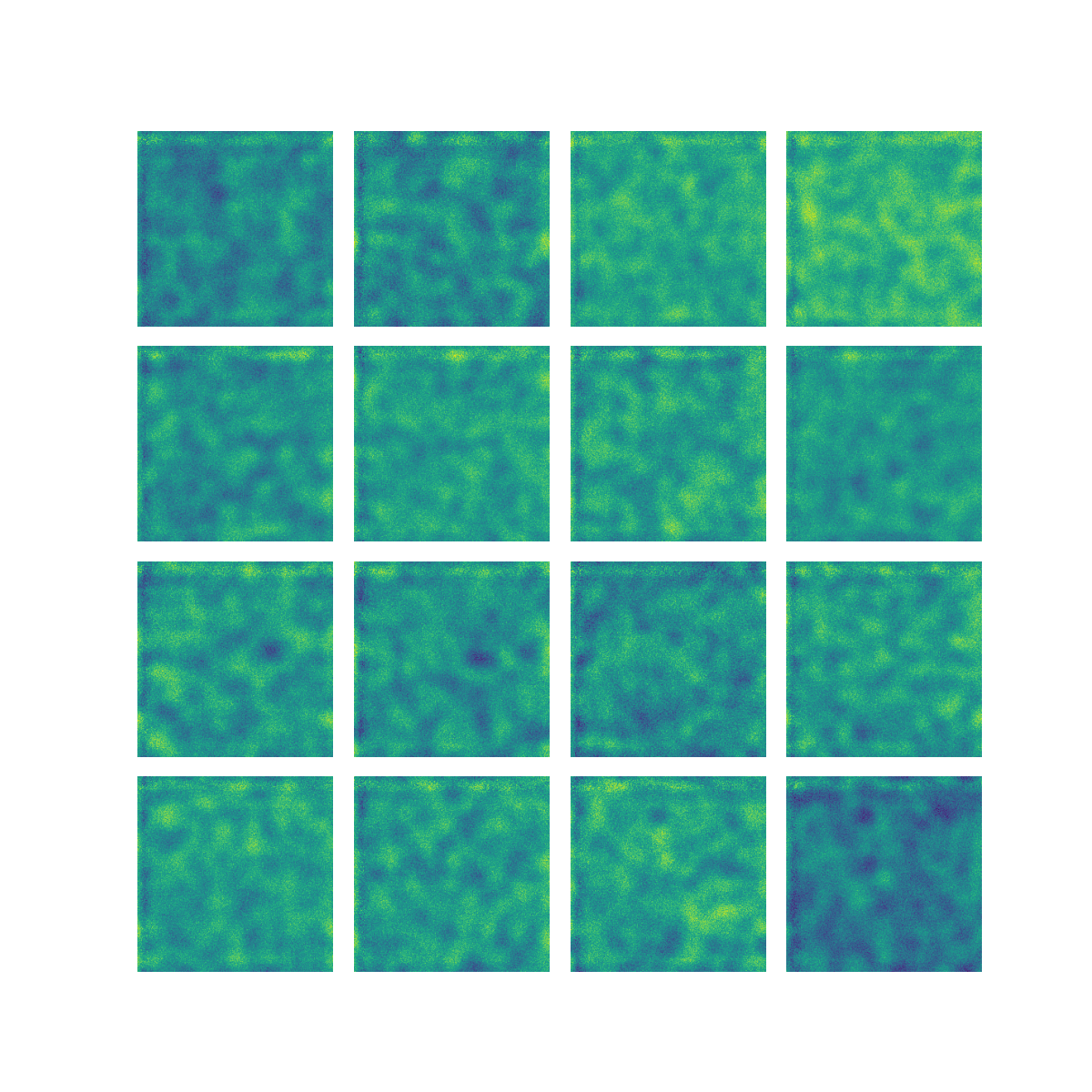}};

\node[label,rotate=90, font=\sffamily] at (0.5,-7.6){Standard};

\end{tikzpicture}

\caption{PDE solution images trained from FNO generative model with different priors at resolution 32. Illustration of generated images at resolution $32 \times 32$, $64 \times 64$, $128 \times 128$ and $256 \times 256$. We record the sliced Wasserstein distance (SW) between learned and true distribution, evaluated at $128 \times 128$ every 50 epochs, and take the average of the best sliced Wasserstein distance over four runs for the four priors, respectively.
Top: Bessel prior ($\gamma_2 = 8$, $k=1.$1). SW$= 0.044$. 
Second row: Combined prior ($\gamma_0 = 0$,$\gamma_1 = 8$, $k=1.1$). SW$= 0.027$. 
Third row: FNO prior, SW$= 0.123$. 
Bottom: Standard Gaussian prior. SW$= 0.053$. 
}
\label{fig:pde2}
\end{figure} 

\paragraph{Multilevel training}
We experimentally show that weights obtained from coarse-level training on resolution $32\times 32$ can be used to warmstart fine-level training at resolutions $64\times 64$ and $128\times128$.
In \cref{fig:pdewarm}, we compare the loss values and sliced Wasserstein distances during fine-level training when initializing the weights using coarse-level results or random weight initialization.
While the loss curves appear similar and ultimately achieve similar values, using the coarse-level weights leads to lower sliced Wasserstein distances, especially in the first few epochs.
\begin{figure}[t]
\centering
\begin{tikzpicture}[scale=1]

\node[] at (-0.6,-3){\includegraphics[width=.24\textwidth]{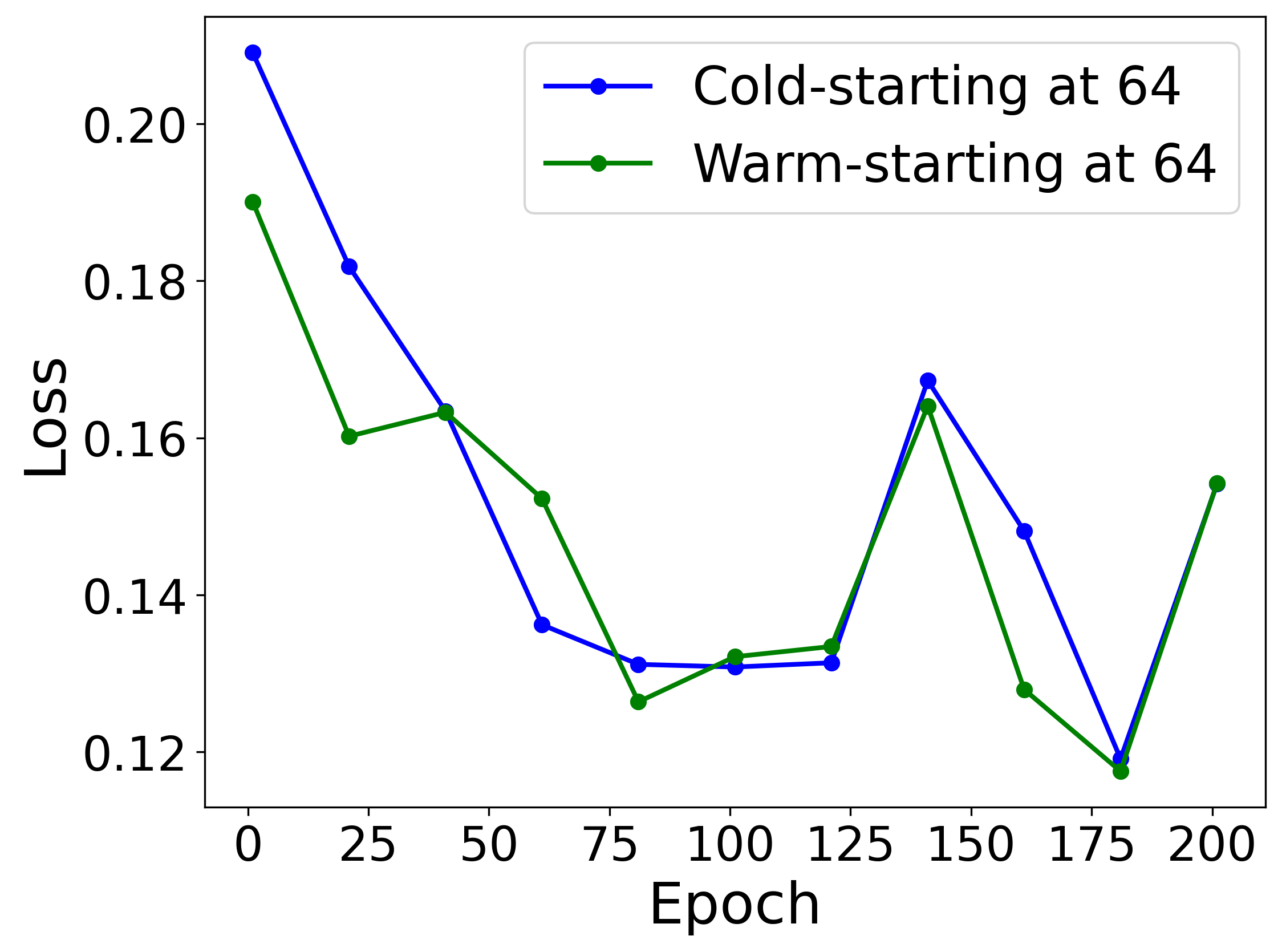}};
\node[] at (3.2,-3){\includegraphics[width=.24\textwidth]{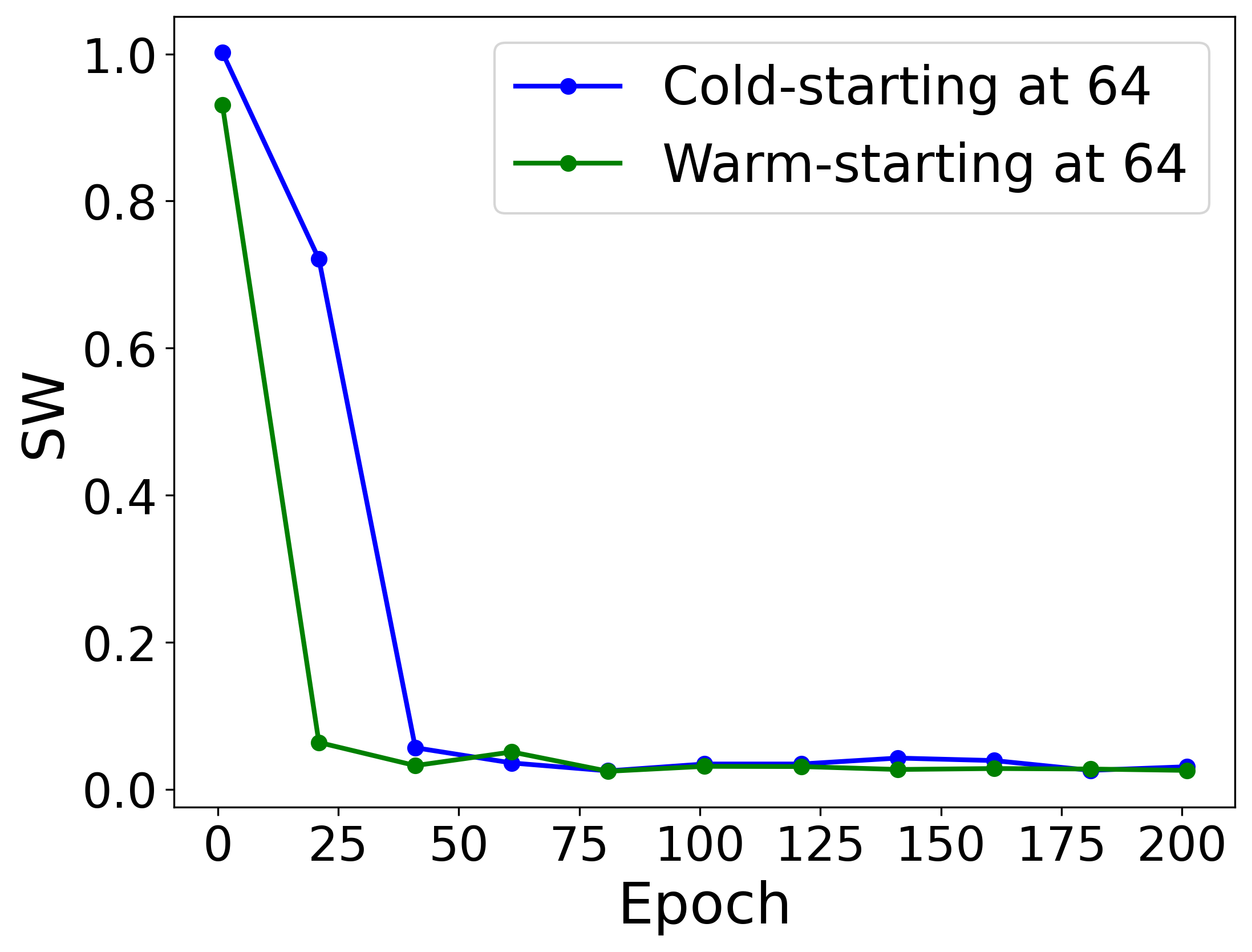}};

\node[] at (7,-3) {\includegraphics[width=.24\textwidth]{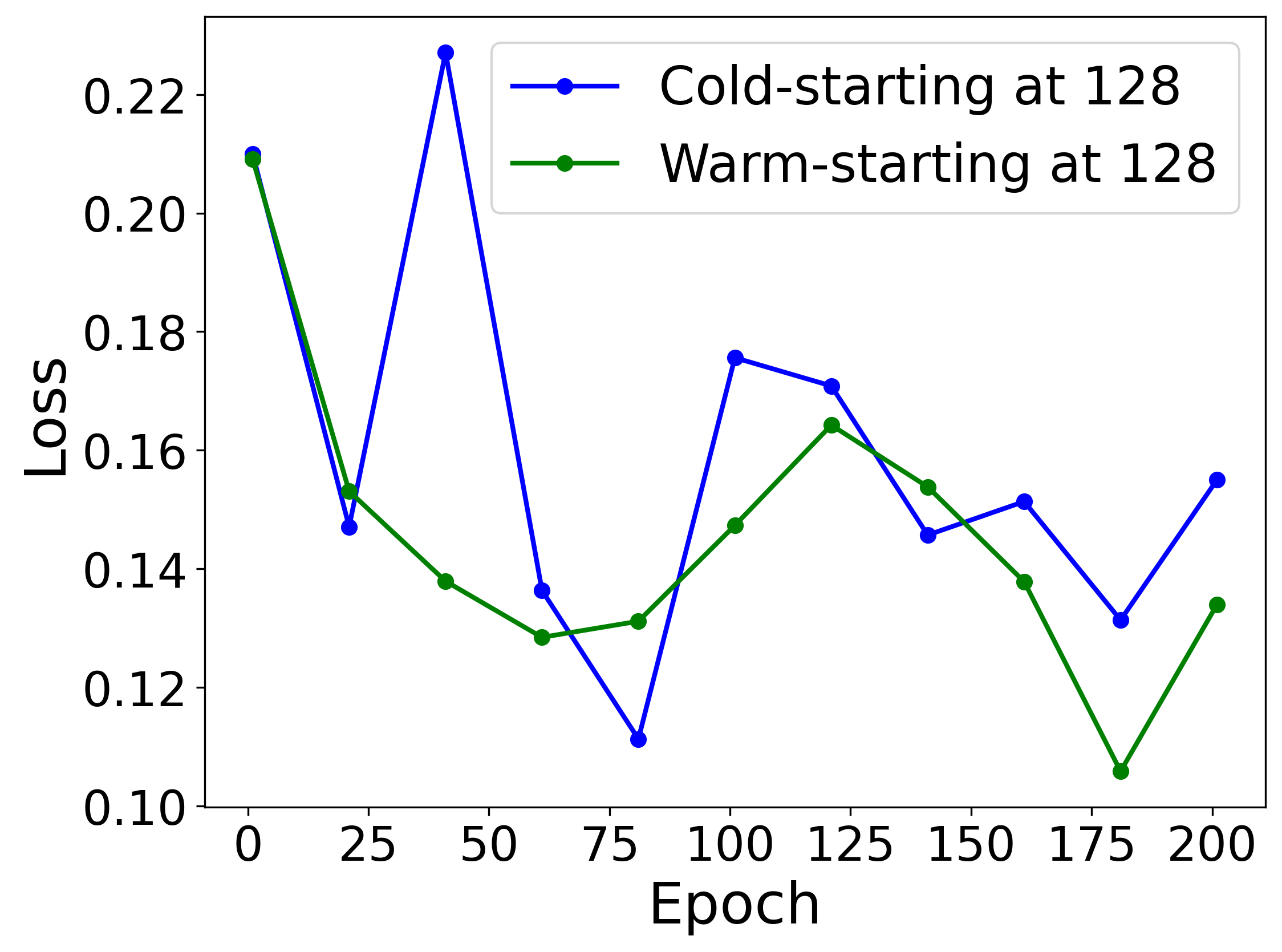}};
\node[] at (10.8,-3) {\includegraphics[width=.24\textwidth]{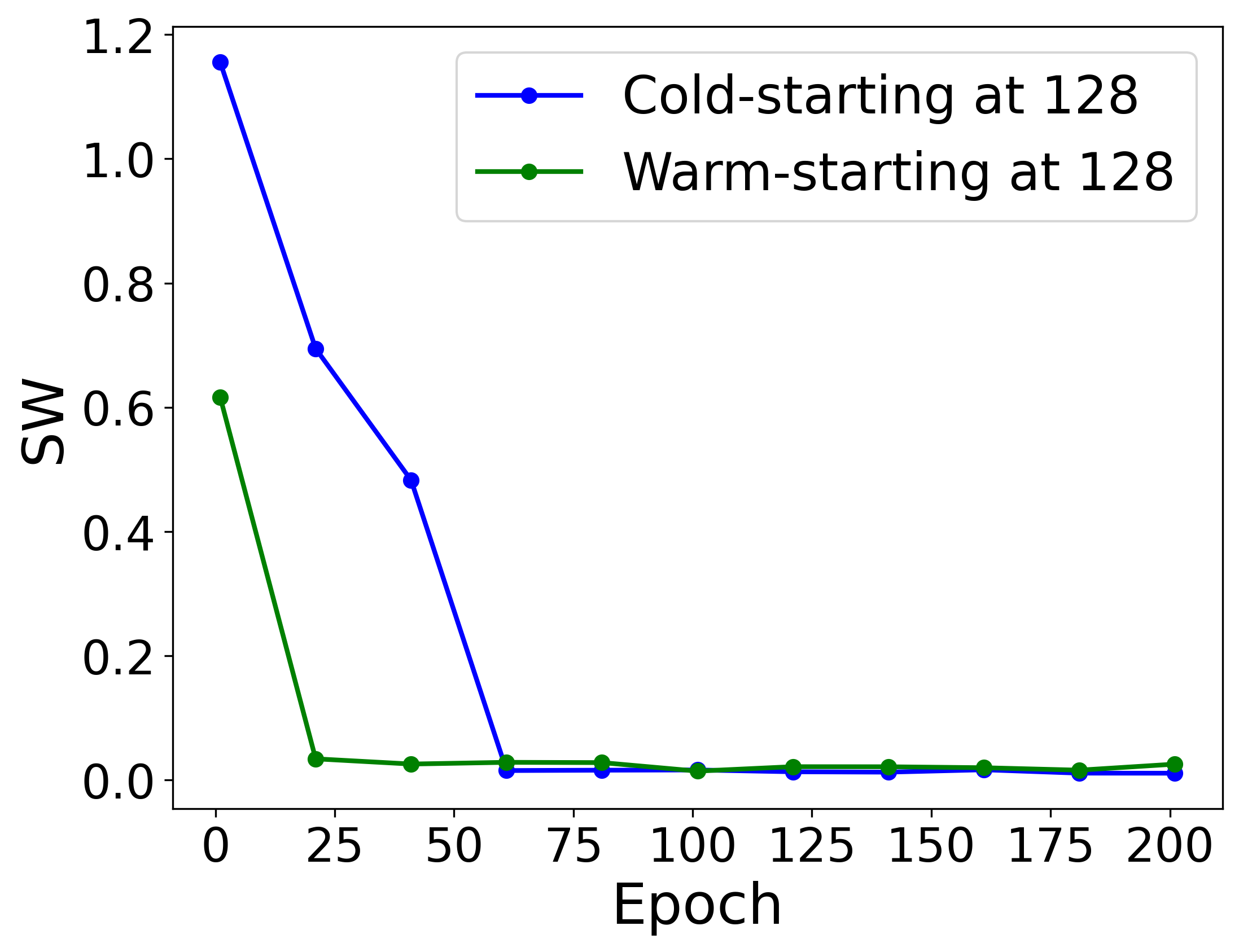}};

\end{tikzpicture}

\caption{Comparison of warm-starting training and cold-starting training (from scratch) in losses and sliced Wasserstein values. The warm-starting training at higher resolutions is initialized with the weights trained at $32 \times 32$, where the weights are saved according to the best-sliced Wasserstein distance. The prior used here is Bessel prior \eqref{eq: bessel} with $\gamma_2 = 8$, $k = 1.1$.
}
\label{fig:pdewarm}
\end{figure} 

\paragraph{Coarse-level hyperparameter search}
We demonstrate that formulating the diffusion model in function space enables one to use hyperparameters tuned on coarse images for fine-level training on the PDE dataset.
To this end, we train the FNO parametrization on resolution $32 \times 32$, $64 \times 64$, $128 \times 128$, with the same setup of hyperparameters, the same number of parameters (2.07M), and the same seed for random number generator. The hyperparameters are obtained from training the models for 500 epochs at resolution $32 \times 32$. In particular, we choose network mode cutoff level $k_{\max} = 12$ and 32 channels in the convolutional layers. The sampling with Euler-Maruyama method is done over the (pseudo-)time horizon $[0, 1]$, with 200 time steps. 
The generated images and the performance metrics are presented in \cref{fig:pde1}. The training losses are recorded for the FNO model trained at three different resolutions, and the sliced Wasserstein distances are evaluated for the test set on resolution $128 \times 128$, and are recorded in the plot every 50 epochs. The images presented in the bottom row are generated by the saved model at different resolutions with the smallest sliced Wasserstein distance. We also experimented U-Net architecture with 32 channels and 4 residual blocks, with a total of 2.4M parameters. When we use the same standard Gaussian prior and the same model setup to train on higher resolutions, the best sliced Wasserstein distance is 0.11 for resolution 64, and 1.42 for resolution 128. On the othere hand, we can use the same setup and a relatively small amount of hyperparameters tuned from coarse resolution to train high resolution models, justified by the quantitative evaluation in sliced Wasserstein distances and visual quality in \cref{fig:pde1}, even at a high resolution. 
\begin{figure}[t]
\centering
\begin{tikzpicture}[scale=1]

\node[] at (5.5,1){\includegraphics[width=.34\textwidth]{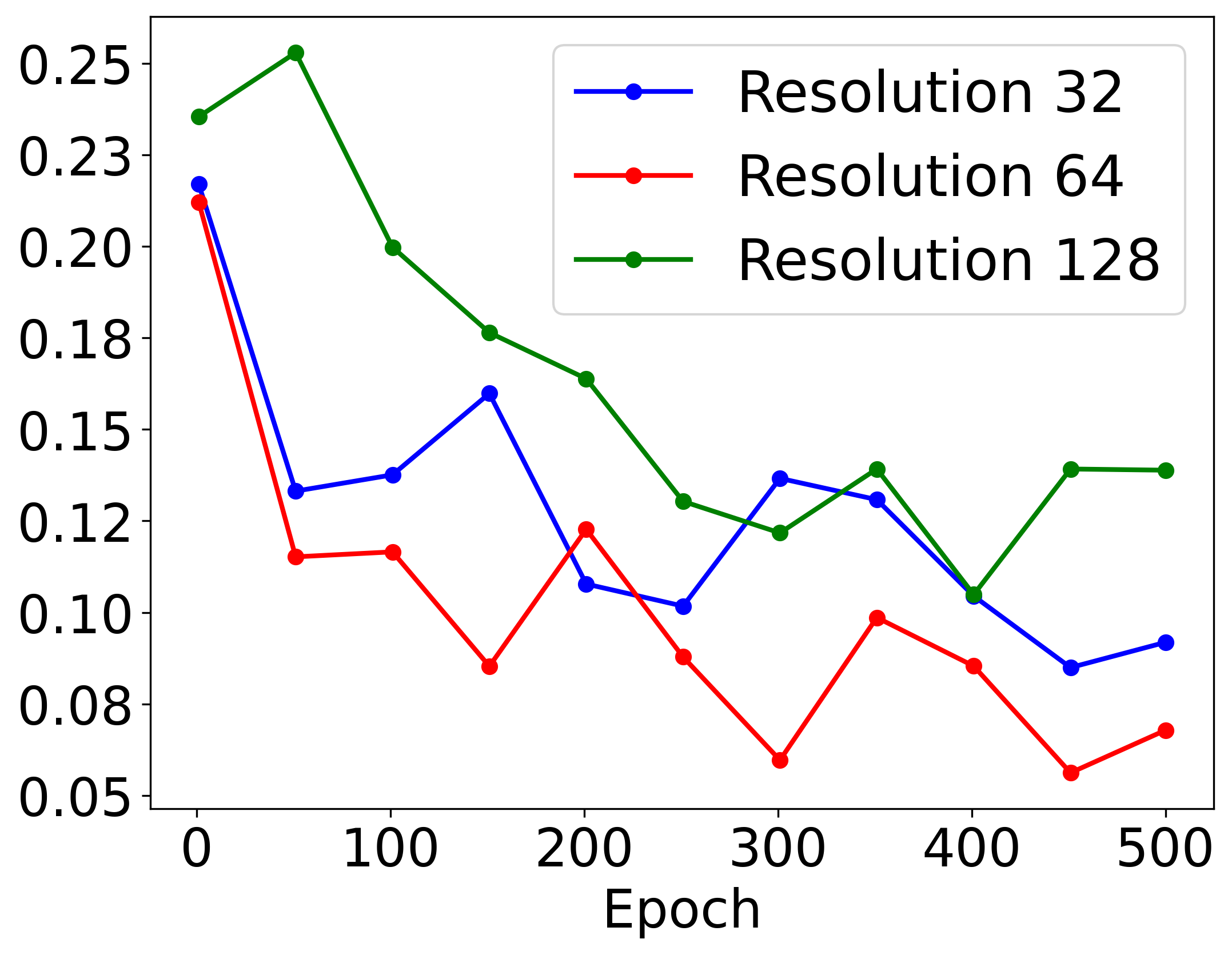}};

\node[label,rotate=90, font=\sffamily] at (2.5,1.5){Loss};

\node[] at (12,1) {\includegraphics[width=.34\textwidth]{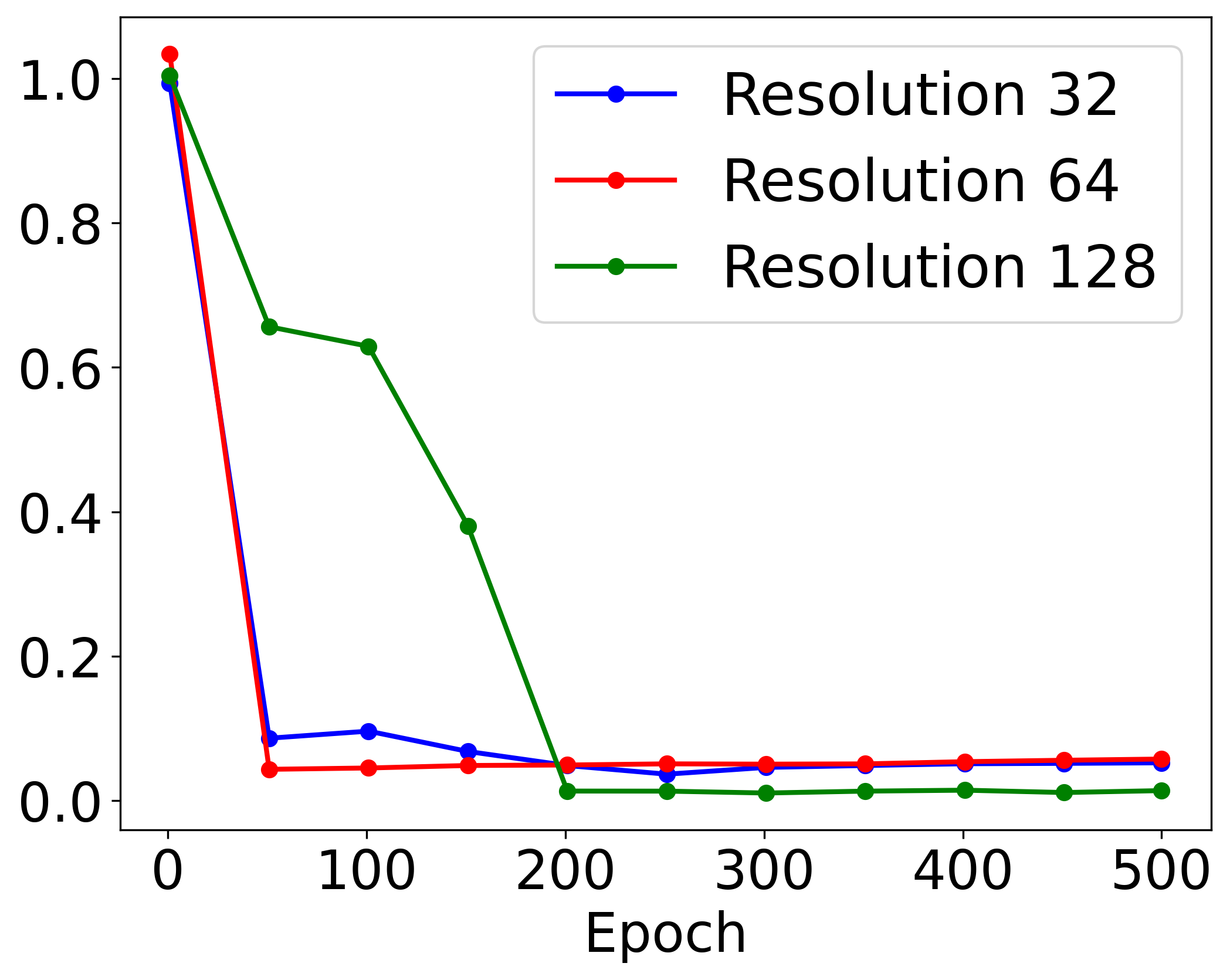}};
\node[label,rotate=90, font=\sffamily] at (9,1.5){S. Wasserstein};

\node[] at (3.5,-3){\includegraphics[trim={1cm 1cm 1cm 1cm},clip,width=.24\textwidth]{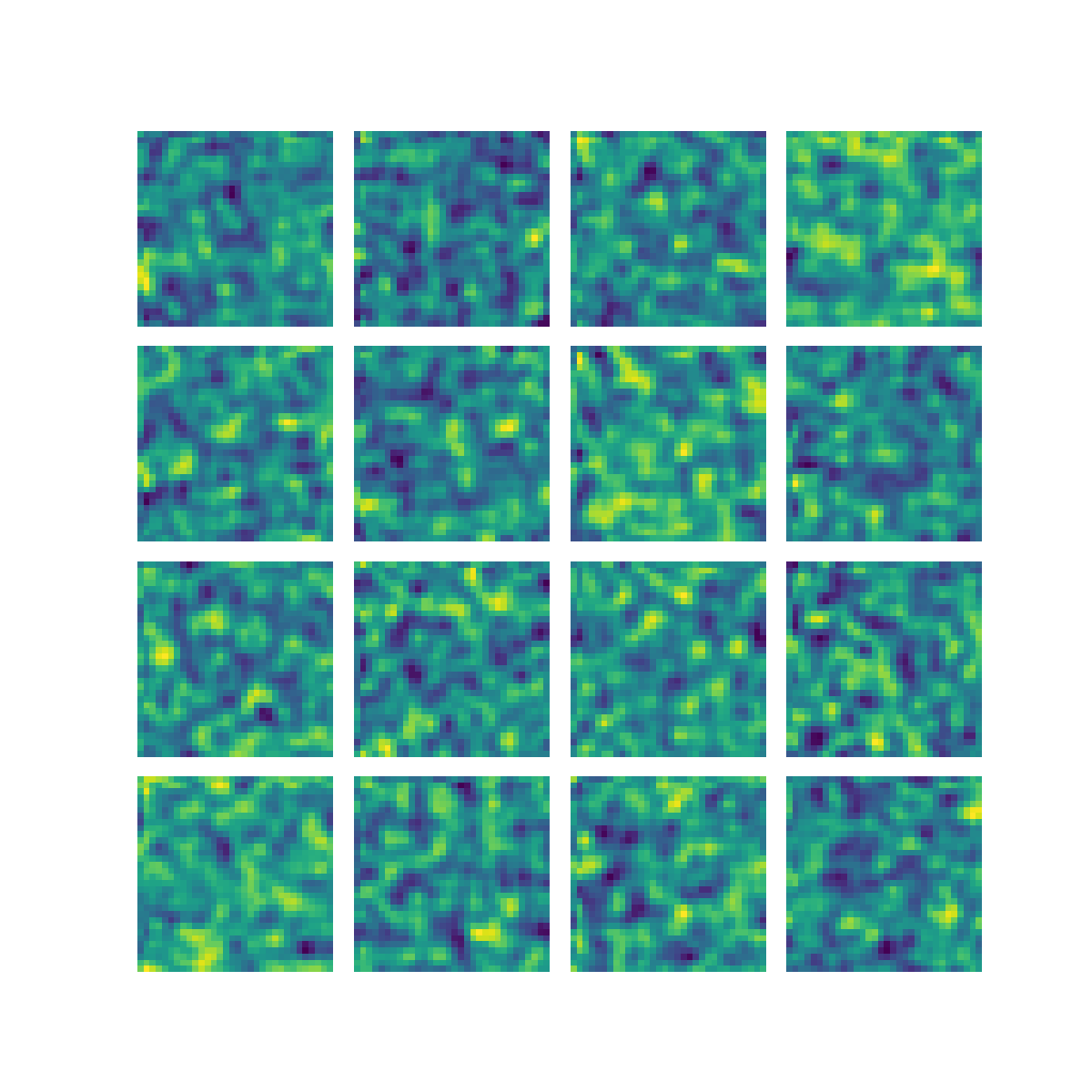}};
\node[label, font=\sffamily] at (3.5,-5){32};

\node[] at (7,-3) {\includegraphics[trim={1cm 1cm 1cm 1cm},clip, width=.24\textwidth]{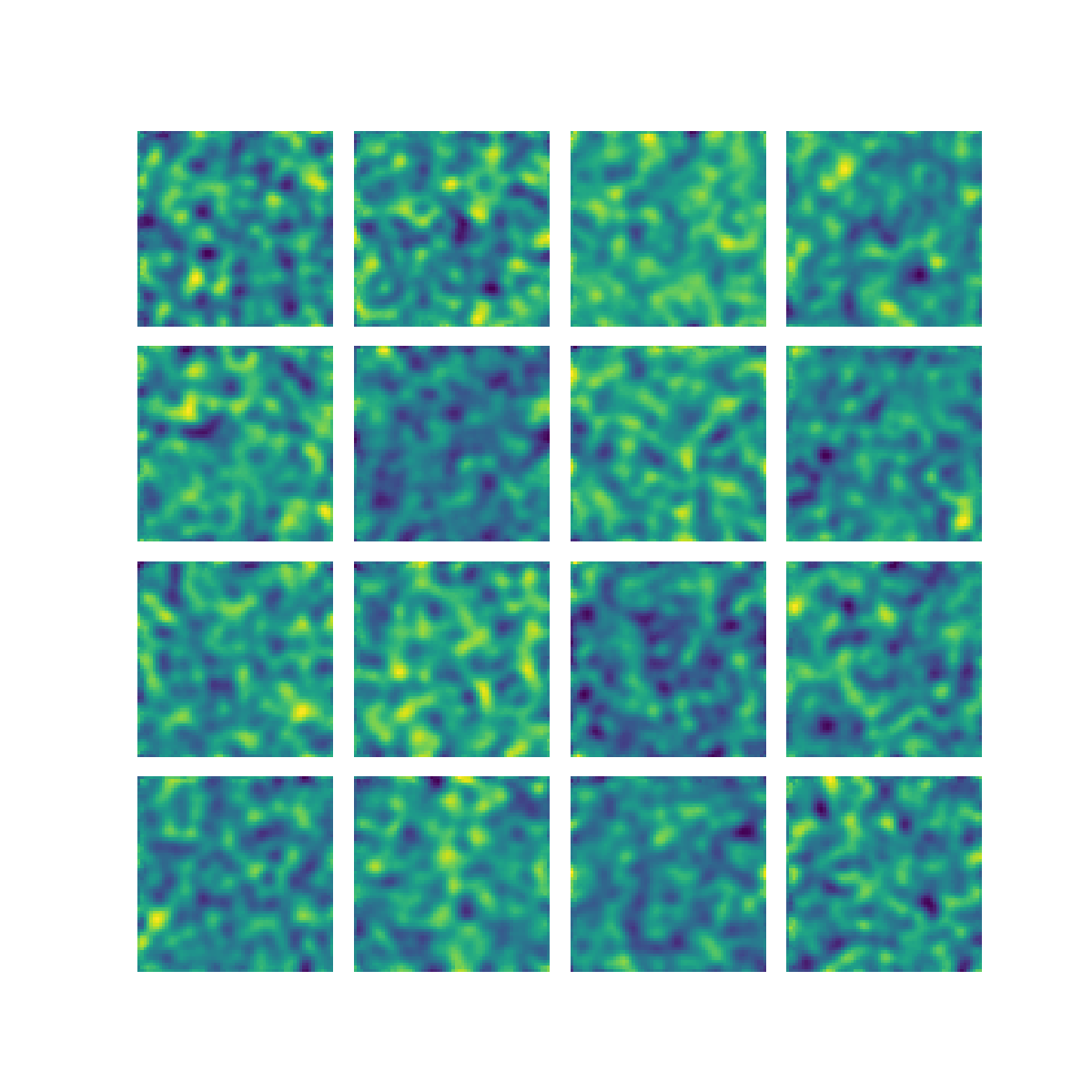}};
\node[label, font=\sffamily] at (7,-5){64};
\node[] at (10.5,-3) {\includegraphics[trim={1cm 1cm 1cm 1cm},clip, width=.24\textwidth]{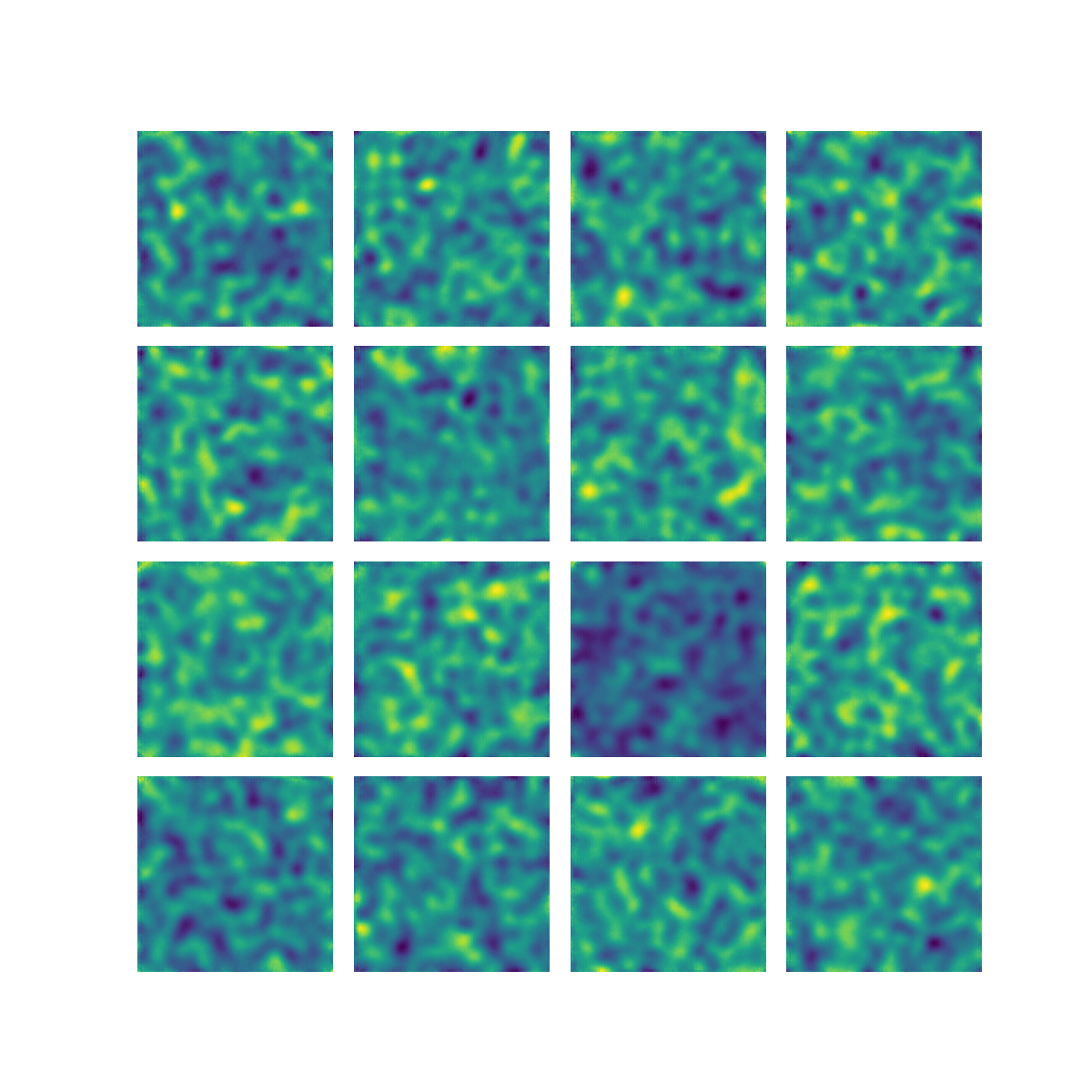}};
\node[label, font=\sffamily] at (10.5,-5){128};
\node[] at (14,-3) {\includegraphics[trim={1cm 1cm 1cm 1cm},clip, width=.24\textwidth]{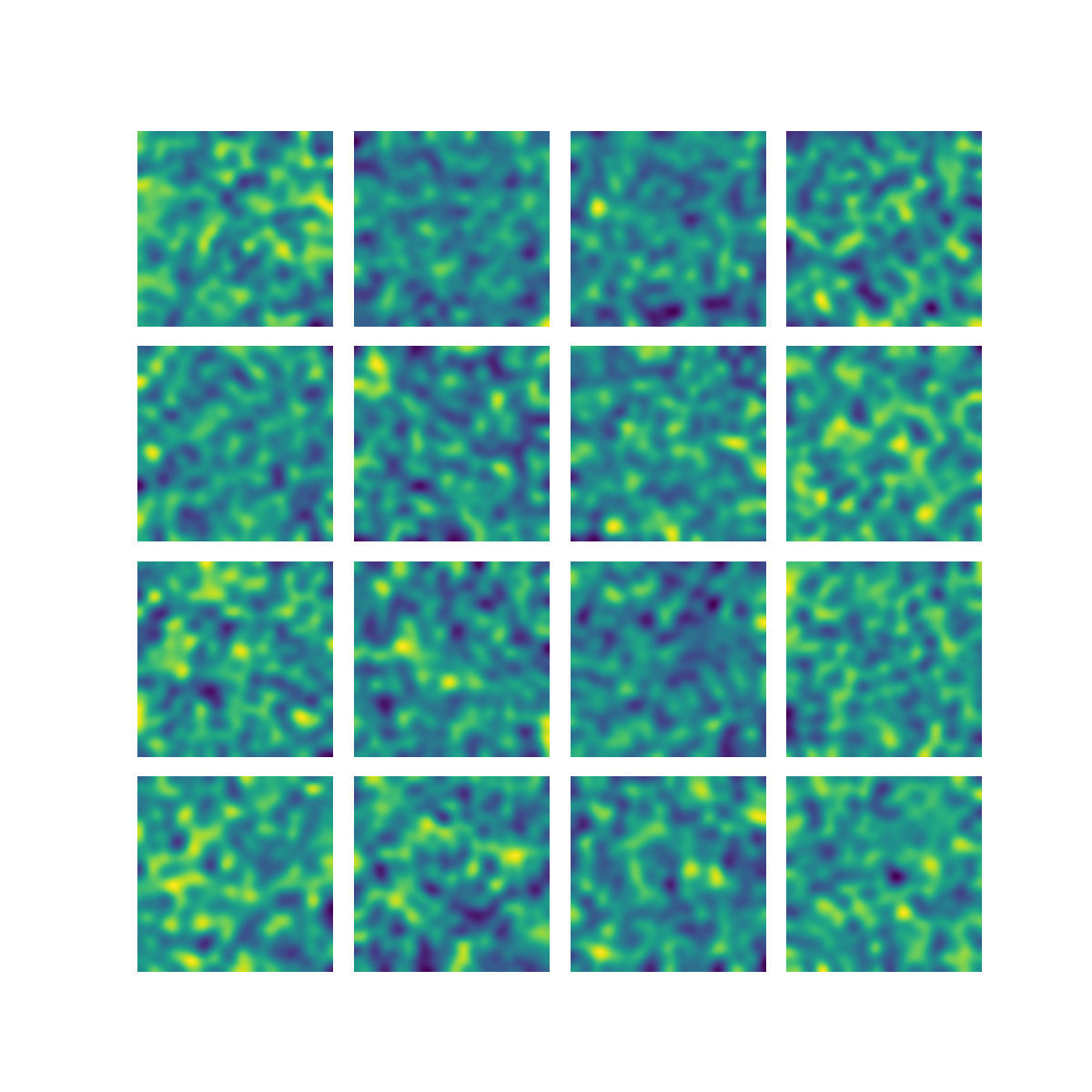}};
\node[label, font=\sffamily] at (14,-5){True};
\end{tikzpicture}

\caption{Results from FNO models with the same parameter count, trained at resolution $32 \times 32$, $64 \times 64$, $128 \times 128$, respectively. The prior used here is Bessel prior \eqref{eq: bessel} with $\gamma_2 = 8$, $k = 1.1$. Top: Training loss curves and Sliced Wasserstein distance curves from the FNO models trained at different resolutions, evaluated at resolution $128 \times 128$. Bottom: The generated image samples from the FNO models trained at different resolutions, and the true image samples.}

\label{fig:pde1}
\end{figure}
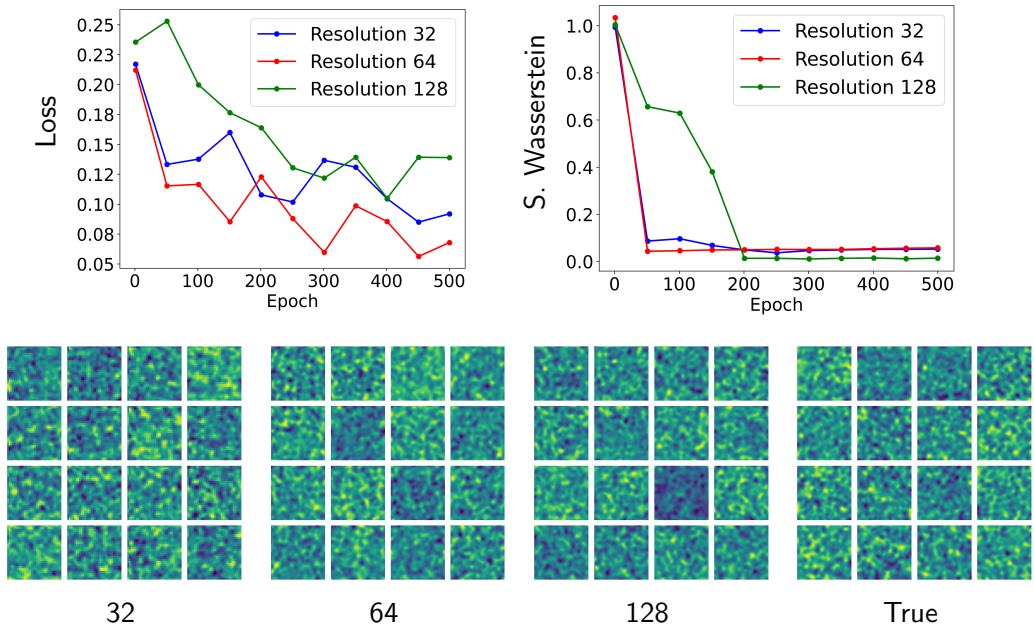 

\section{Discussion and Future Work}\label{eq:discuss}
We analyze infinite-dimensional score-based generative models (SBDMs) for image generation and develop multilevel algorithms.  Our theory underscores the importance of changing the latent distribution in existing SBDMs to trace class Gaussians.
With this choice, we show the well-posedness of the resulting forward and reverse processes. 
We show that the discretizations of both processes converge to their continuous counterparts.
We also present an explicit error estimate on the Wasserstein paths, which bounds the distance to the data distribution. In our numerical experiments, we compare different covariance operators and network architectures using the MNIST dataset.
While our work and related papers~\cite{ func_diff,kerrigan,nvidia_func_diff,jakiw_diffusion} are important first steps towards well-posed infinite-dimensional SBDMs, more work is needed to achieve state-of-the-art results.

In our experiments, we outlined the need for function-valued operator architectures to achieve good multilevel training results. 
In our experience, the FNO model behaves more consistently across resolutions than the commonly-used U-Net. The low-frequency representation implied by the FNO does not necessarily need to be true for image distribution since it cannot perfectly capture sharp edges in images.
As possible alternatives, one may consider Wavelet neural operators \cite{gupta2021multiwaveletbased}, the U-Net neural operators \cite{UNO, havrilla2024dfu}, or modifying the U-Net architecture to enable multilevel training such as done in \cite{williams2023unified} to leverage the strong inductive bias of U-Nets. Furthermore,  \cite{Habring2023}  argues that CNNs imply at least continuity of infinite-dimensional images, which also opens the possibility of more specialized theoretical analysis. 

Even within our infinite-dimensional framework, training SBDMs remains a complex endeavor, and the performance relies on the interplay of architecture, downsampling, and prior distributions. Hence, we expect additional improvements in practice by better understanding and selecting each of these components. For instance, the FNO architecture has the mode cut off, which makes it viable for these experiments, but it also contains "skip" connections, which propagate more information so that it does not perfectly satisfy our theory. We see our paper as an important step in this direction and see tremendous potential if one gets all these things perfectly correct. 
Another open question in the training arises from the behavior of stochastic approximation methods, such as Adam, in the infinite-dimensional setting. Compared to the numerical optimization techniques in multilevel approaches for other variational problems (see, e.g., ~\cite{Modersitzki2009} for applications in image registration), those methods are not descent methods. Hence, depending on the learning rate, even a perfect initialization from coarse mesh training may not lead to convergence to a global minimum at the fine level. 

\section{Acknowledgements}
PH thanks Chin-Wei Huang, Rafael Orozco, Jakiw Pidstrigach, and Shifan Zhao, and NY thanks Tomoyuki Ichiba and Nils Detering for helpful discussions. We also thank the two anonymous referees for their constructive feedback and many helpful suggestions.

\bibliographystyle{abbrv}
\bibliography{main}

\section{Hilbert Space Valued Wiener Processes}
\label{sec:wiener}
This section is based on 
\cite{da2014stochastic,hairer2009introduction,scheutzowLectureNotes,  liu2015stochastic}. 
Let $H$ be a separable Hilbert space and
$\mathcal B(H)$ its Borel $\sigma$-algebra.
By $L(H)$ we denote the space of bounded linear operators and 
by $K(H)$ the subspace of compact operators mapping from $H$ to $H$.
Further, let $H^*$ be the dual space of bounded linear functionals from $H$ to $\R$.
By Riesz' representation theorem, there exists an isomorphism between $H^*$ and $H$, more precisely
every linear bounded functional corresponds to an element $v \in H$ by
$\ell_v(u) = \langle u,v \rangle$ for all $u \in H$.
Let $(e_k)_{k\in \mathbb N}$ be an orthonormal basis of $H$.
A self-adjoint, positive semidefinite operator $Q \in L(H)$ which satisfies
$$
\text{tr} (Q) \coloneqq \sum_{k=1}^\infty \langle Q e_k,e_k \rangle < \infty
$$
is called \emph{trace class} (operator) with trace $\text{tr} (Q)$. The definition is independent of the chosen basis.
Trace class operators are in particular compact operators and Hilbert-Schmidt operators.
Recall that for two separable Hilbert spaces $H_1$ and $H_2$, a linear bounded operator $T:H_1 \to H_2$ is called a \emph{Hilbert-Schmidt operator} 
if 
$$
\|T \|_{L_2(H_1,H_2)} \coloneqq \Big( \sum_{k=1}^\infty \langle \|T e_k\|_{H_2}^2 \rangle \Big)^\frac12 < \infty,
$$
for an orthonormal basis $(e_k)_{k\in \mathbb N}$ of $H_1$. The space $L_2(H_1,H_2)$ of all such operators with the above norm is again a Hilbert space.

Let $(\Omega, \mathcal F,\mathbb P)$ be a probability space where $\Omega$ is a nonempty set and $\mathcal F$ is a $\sigma$-algebra of subsets of $\Omega$. For a Banach space $E$, we denote by $L^2(\Omega, \mathcal F,\mathbb P;E)$, the Bochner space of measurable functions $f:\Omega \to E$ with norm
$$
\|f\|_{L^2(\Omega, \mathcal F,\mathbb P;E)} \coloneqq \Big(\int_{\Omega} \|f\|_E^2 \, \rm{d} \mathbb P \Big)^\frac12<\infty .
$$ 
Often, we consider the case $E= \mathcal C([0,T],H)$,  which we always equip with the uniform norm.

A \emph{measure} $\mu$ on $(H,\mathcal B(H))$ 
is called \emph{Gaussian} if 
$\ell _\# \mu \coloneqq \mu \circ \ell^{-1}$
is a Gaussian measure on $\R$ for all $\ell \in H^*$.
An $H$-valued \emph{random variable} $X:\Omega \to H$ is called \emph{Gaussian} if 
$\mu = X_\# \mathbb P$ is a Gaussian measure on $H$.
By the following theorem \cite[Theorem 3.12]{scheutzowLectureNotes}, there is a one-to-one correspondence between Gaussian measures and pairs $(m,Q)$, where $m \in H$ and $Q \in K(H)$ is trace class.

\begin{theorem}\label{thm:Gau}
A \emph{measure} $\mu$ on $(H,\mathcal B(H))$  is Gaussian if and only if there exists some $m \in H$ and some trace class operator
$Q \in K(H)$ 
such that
\begin{equation} \label{star}
\hat \mu (u) \coloneqq \int _H {\rm e}^{{\rm i} \langle u,v\rangle} \, {\rm d}\mu(v) 
= 
{\rm e}^{{\rm i} \langle m,u\rangle - \frac12\langle Qu,u \rangle}, \quad u \in H.
\end{equation}
Then $m \in H$ is called the {\rm mean} of $\mu$ and $Q \in K(H)$ {\rm covariance} (operator) of $\mu$ and
we write $\mathcal N(m,Q)$ instead of $\mu$ and $X \sim \mathcal N(m,Q)$ for the corresponding random variable.
The Gaussian measure determines $m$ and $Q$ uniquely, and conversely, for each $(m,Q)$, $m\in H$, $Q\in K(H)$ trace class,
there exists a Gaussian measure which fulfills \eqref{star}. 
Furthermore, the following relations hold true for all $v,\widetilde v \in H$:
\begin{itemize}
\item[i)] 
$\int_H \langle u,v\rangle \, {\rm d} \mu(u) = \langle m,v \rangle$,
\item[ii)] 
$\int_H \left( \langle u,v\rangle - \langle m,v \rangle \right) \left( \langle u,\widetilde v\rangle - \langle m,\widetilde v \rangle \right)
\, {\rm d} \mu(u) = \langle Q \widetilde v, v \rangle$,
\item[iii)]
$\int_H  \|u-m\|^2\, {\rm d} \mu(u) = {\rm tr} (Q)$,
\end{itemize} 
and similarly for the corresponding random variable
\begin{itemize}
\item[iv)]
$\E(\langle X,v \rangle ) = \langle m,v \rangle$,
\item[v)]
${\rm Cov}\left(\langle X,v \rangle, \langle X, \widetilde v \rangle \right) = \langle Q \widetilde v,v \rangle$,
\item[vi)]
$\E(\|X-m\|^2) = {\rm tr} (Q)$.
\end{itemize} 
\end{theorem}
Now let $(e_k)_{k \in \mathbb N}$ be an orthonormal basis of eigenvectors of  a trace class operator $Q \in K(H)$ with eigenvalues
$\lambda_1 \ge \lambda_2 \ge \ldots \ge 0$. Then $X \sim \mathcal N(m,Q)$ if and only if
\begin{equation}\label{spect_decomp}
X = m + \sum_{k \in \mathbb N_Q} \sqrt{\lambda_k} \beta_k e_k, 
\end{equation}
where $\beta_k$, $k \in \mathbb N_Q \coloneqq \{j \in \mathbb N: \lambda_j >0\}$, 
are independent $\mathcal N(0,1)$ distributed random variables. 
The series converges in $L^2(\Omega,\mathcal F,\mathbb P;H)$ \cite[Theorem 3.15]{scheutzowLectureNotes}.

For a trace class operator $Q \in K(H)$ and $T>0$, a $H$-valued random process
$W^Q = (W_t^Q)_{t \in [0,T]}$ is called a $Q$-\emph{Wiener process} if 
$W_0^Q = 0$, $W$ has continuous trajectories and independent increments, 
and fulfills for $t,s \in [0,T]$ the relation
$$W^Q_t - W^Q_s \sim \mathcal N\left(0,(t-s)Q \right) \quad \text{for all } 
0\le s\le t \le T.$$

For fixed $T>0$, a \emph{filtration} 
$\mathbb F = (\mathcal F_t)_{t \in [0,T]}$ on $(\Omega,\mathcal F, \mathbb P)$
is called  \emph{complete}, if $\mathcal F_0$ contains all $A \subset \Omega$ such that $A \subset B$ for some $B \in \mathcal F$ such with $\mathbb P(B) = 0$. The \emph{completion} of a filtration $\mathbb F$ is the smallest complete filtration containing $\mathbb F$. The filtration is called \emph{right-continuous}, if $\mathcal F_t = \cap_{s>t} \mathcal F_s$ for all $t \in [0,T]$ and a complete and right-continuous filtration is called \emph{normal}. A $Q$-Wiener process $W^Q$
is a $Q$-Wiener process with respect to $\mathbb F$, if $W^Q$ is $\mathbb F$ adapted
and $W^Q_t-W^Q_s$ is independent of $\mathcal F_s$ for all $0 \le s\le t \le T$.
Every $Q$-Wiener process is such a process with respect to some normal filtration; 
e.g., the completion of the filtration generated by $W^Q$. Note that this filtration is always right-continuous due to the continuity of sample paths of $W^Q$.

Stochastic integrals with respect to $W^Q$ can be introduced
in a very similar way to the classical finite-dimensional case. 
First, one defines integrals for elementary functions with values in the space of Hilbert-Schmidt operators $ L_2(H_Q,K)$, where $H_Q\coloneqq Q^\frac12 H$ denotes the Cameron-Martin space of $Q$ and $K$ is some separable Hilbert space. Then one extends the stochastic integral via the It\^o isometry \cite[Eq.~(2.19)]{GawMan10}
$$\mathbb{E}\Big[ \Big\lVert \int_{0}^t \phi(s) \, {\rm d} W^Q_s\Big\rVert_K^2 \Big] = \mathbb{E} \left [ \int_0^t \left \Vert \phi(s)  \right \Vert_{L_2(\smash{Q^\frac12 H},K)}^2 \, {\rm d} s \right],$$
to adapted $L_2(H_Q,K)$-valued processes $\phi$ for which the right-hand side is finite. With this isometry at hand, the integral is finally defined for the larger class of adapted  $L_2(H_Q,K)$-valued processes with almost surely square integrable Hilbert-Schmidt norm. For more details, we refer to \cite{da2014stochastic,GawMan10,liu2015stochastic}. Note that one can also construct stochastic integration with respect to cylindrical Wiener process \cite[Section 4.4]{hairer2009introduction}, where the covariance operator is taken to be the identity operator. The cylindrical Wiener process is a trace class operator in an enlarged Hilbert space. 

\section{Complete Proofs}
This supplementary provides the proof of \cref{var} and
more detailed proofs of \cref{lem:reverse} and \cref{thm:main}.

\subsection{Proof of \cref{var}}\label{forward}

\begin{proof}
1.  First, we show the existence and uniqueness of a strong solution of \eqref{eq:SDE}. We verify that the assumptions of \cite[Thm.~3.3]{GawMan10} are satisfied. We start by noting that the measurability assumptions of the theorem are trivially satisfied in our setting since our drift and diffusion coefficients $f$ and $g$ are deterministic, i.e., they do not depend on the probability space $(\Omega,\mathcal F, \mathbb P)$, and continuous on $[0,T] \times H$ and $[0,T]$, respectively.
Furthermore, we have the linear growth bound  
\begin{equation*}
        \lVert f(t,x) \rVert + \lVert  g(t)\rVert_{L_2(H_Q,H)} 
				\le 
				\left( \frac{1}{2} \lVert \alpha\rVert_{\mathcal C([0,T],\mathbb R)}   + \sqrt{ \lVert \alpha \rVert_{\mathcal C([0,T],\mathbb R)} \, \mathrm{tr}(Q)} \right) \, (1+ \lVert x\rVert)
    \end{equation*}
    and the Lipschitz condition
    \begin{equation*}
        \lVert f(t,x_1)-f(t,x_2)  \rVert 
				\le \frac{1}{2} \lVert \alpha\rVert_{\mathcal C([0,T],\mathbb R)} \, \lVert x_1-x_2\rVert.
    \end{equation*}
Thus, we can apply \cite[Thm.~3.3]{GawMan10} and conclude that \eqref{eq:SDE} has a unique strong solution $(X_t)_{t \in [0,T]}$ which satisfies $\mathbb E \left[ \sup_{t \in [0,T]} \lVert X_t\rVert^2 \right]<\infty$ if $\mathbb E \left[ \lVert X_0 \rVert^2\right]<\infty$.
\\[1ex]
 2.   Next, we prove that $(X_t)_{t \in [0,T]}$ is given by \eqref{eq:var}. For $k \in \mathbb N$, we define 
    \begin{equation*}
        F^k \colon [0,T] \times H \to \mathbb R, \qquad  F^k(t,x) \coloneqq a_t \langle x,e_k \rangle.
    \end{equation*}
    Clearly, $F^k$ is continuous with Fr\'echet partial derivatives 
		$F^k_x(t,x)= a_t \, e_k$, $F^k_{xx}(t,x)=0$, and 
		$F^k_t(t,x)=\frac{1}{2}\alpha_t \,a_t \langle x,e_k\rangle$ that are continuous 
		and bounded on bounded subsets of $[0,T]\times H$. Let $(X_t)_{t \in [0,T]}$ be the unique strong solution of \eqref{eq:SDE}. 
		We obtain by It\^o's formula \cite[Thm.~2.9]{GawMan10} $\mathbb P$-almost surely for all $t \in [0,T]$ that 
    \begin{alignat}{3}
        & F^k(t,X_t) && = F^k(0,X_0) + \int_0^t \langle F^k_x(s,X_s), \sqrt{\alpha}_s \, {\rm d} W^Q_s\rangle 
				+ \int_0^t \Big( F^k_t&&(s,X_s)
				+ \langle F^k_x(s,X_s),-\frac{1}{2}\alpha_s X_s\rangle\\ 
        & && && +\frac{1}{2} \, \mathrm{tr}\big(F^k_{xx}(s,X_s)\sqrt{\alpha_s} Q\big)\Big) \, {\rm d} s\\
        & && = \langle X_0,e_k \rangle + \int_0^t \langle a_s \, e_k ,\sqrt{\alpha}_s \, {\rm d} W^Q_s\rangle, &&
    \end{alignat}
    with
  \begin{equation} \begin{aligned}
        \int_0^t \langle a_s \, e_k ,\sqrt{\alpha}_s \, {\rm d} W^Q_s\rangle 
				& \coloneqq 
				\int_0^t   a_s \, \sqrt{\alpha}_s \, \langle e_k , \cdot  \rangle\, {\rm d} W^Q_s \\
        &= 
				\sum_{k^\prime=1}^\infty \int_0^t a_s \, \sqrt{\alpha}_s \, \langle e_k , e_{k^\prime} \rangle
\, {\rm d} \langle W^Q_s,e_{k^\prime} \rangle\\
& = \sqrt{\lambda_k} \int_0^t   a_s \,  \sqrt{\alpha}_s \,{\rm d} \beta_k(s), 
   \end{aligned}\end{equation}
    where the first equality is the definition of the left-hand side and 
		the second equality follows from \cite[Lem.~2.8]{GawMan10}. 
		On the other hand, by \cite[Ex.~2.9]{GawMan10}, we have $\mathbb P$-almost surely
    \begin{equation*}
        \Big\langle \int_0^t  a_s \, \sqrt{\alpha_s} \, {\rm d} W^Q_s, e_k\Big\rangle 
				= \sqrt{\lambda_k} \int_0^t a_s \, \sqrt{\alpha_s} \,  {\rm d} \beta_k(s), \qquad t \in [0,T].
    \end{equation*}
    Overall, we obtain $\mathbb P$-almost surely for all $k \in \mathbb N$ that
    \begin{equation*}
        \langle X_t, e_k \rangle  
				= 
				a_t^{-1}  \Big(  \langle X_0,e_k \rangle 
				+ \Big\langle \int_0^t a_s \, \sqrt{\alpha_s} \, {\rm d} W^Q_s, e_k\Big\rangle \Big),
    \end{equation*}
    which proves \eqref{eq:var}.
\\[1ex]
3.     Finally, we consider $B_t$ in \eqref{eq:B}. 
By \cite[Lem.~2.8]{GawMan10}, we immediately obtain 
    \begin{equation*}
        B_t = \sum_{k \in \mathbb N} \sqrt{\lambda_k}\, \Big(\int_0^t a_s\, \sqrt{\alpha_s} \,   {\rm d}\beta_k(s)\Big) \, e_k
    \end{equation*}
    in  $L^2(\Omega,\mathcal F, \mathbb P; H)$
		for all $t \in [0,T]$. 
		By construction of the stochastic integral, the difference  
    \begin{equation*}
        \delta B_t^K \coloneqq B_t - \sum_{k =1}^K \sqrt{\lambda_k}\, \Big(\int_0^t a_s \, \sqrt{\alpha_s} \, \, {\rm d}\beta_k(s)\Big) \, e_k \in  L^2(\Omega,\mathcal F,\mathbb P;H)
    \end{equation*}
    is an $H$-valued square-integrable martingale for all $K \in \mathbb N$ and thus, by Doob's maximal inequality \cite[Thm.~2.2]{GawMan10},
    \begin{equation*}
        \mathbb E \left[ \sup_{t \in [0,T]} \left\lVert \delta B_t^K  \right\rVert^2 \right] \le 4 \, \mathbb E \left[  \left\lVert \delta B^K_T \right\rVert^2 \right] \to 0 \text{ as } K \to \infty,
    \end{equation*}
    which proves the second part of\eqref{eq:B}. 
		For $k \in \mathbb N$, we set 
    \begin{equation*}
        b_k(t)\coloneqq \int_0^t a_s \,\sqrt{\alpha_s} \,  {\rm d} \beta_k(s), \qquad t \in [0,T].
    \end{equation*}    
    In the finite-dimensional setting it is well-known that the stochastic integral of a deterministic function with respect to a Brownian motion is a centered Gaussian, see e.g. \cite[p.~108]{leGall}. 
		In particular $b_k(t)$ is a centered Gaussian for all $k \in \mathbb N$ and all $t \in [0,T]$. 
		By the It\^o isometry, we obtain
  \begin{equation} \begin{aligned}
        \mathrm{Var} \left[b_k(t)\right] = \mathbb E \left[ \Big(\int_0^t  a_s \,\sqrt{\alpha_s} \,  {\rm d} \beta_k(s) \Big)^2 \right]
				= \int_0^t  a_s^2 \,\alpha_s \, {\rm d} s
				= a_t^2 -1.
   \end{aligned}\end{equation}
   For $t>0$, we have $(a_t^2-1)^{-1/2} \, b_k(t) \sim \mathcal{N}(0,1)$ for all $k \in \mathbb N$ and by \eqref{eq:B} it holds 
  \begin{equation} \begin{aligned}
        B_t = \sum_{k \in \mathbb N} \sqrt{\lambda_k \, (a_t^2-1)} \, (a_t^2-1)^{-1/2} \, b_k(t) \, e_k
   \end{aligned}\end{equation} in $L^2(\Omega,\mathcal F,\mathbb P;H)$. By \cite[Thm.~3.15]{scheutzowLectureNotes}, we obtain that $B_t$ is a centered $(a_t^2 -1)Q$-Gaussian, since  $((a_t^2-1)^{-1/2} \, b_k(t))_{k \in \mathbb N}$ are mutually independent.
\end{proof}

\subsection{Proof of \cref{lem:reverse}}\label{3.4}

\begin{proof} 
1.  First, we show that $p^n_t$ has the form \eqref{eq:score} which then implies that $f^n$ in \eqref{reverse_disc} is well-defined
with the desired properties.
By \cref{var}, we know that $B_t^n$ in \eqref{eq:finvar} is a centered Gaussian in $H_n$ 
	with covariance $(a_t^2-1)Q_n$ and thus $\hat{B}_t^n$ is a centered Gaussian in $\mathbb R^n$ with covariance 
	$(a_t^2-1) \mathrm{diag}(\lambda_1,...,\lambda_n)$. In particular, for all $t \in (0,T]$,
    \begin{equation*}
        \hat{b}_t^n(z) \coloneqq \frac{(\lambda_1...\, \lambda_n)^{-1/2}}{\left(2\pi(a_t^2-1)\right)^{n/2}} \, b_t^n(z)
    \end{equation*}
    is the Lebesgue density of $\hat{B}_t^n$, cf. e.g. \cite[Thm~1.3]{leGall}. 
		By \eqref{eq:finvar}, we have
    \begin{equation*}
        \hat{X}^n_t =a_t^{-1} \, (\hat{X}_0^n+\hat{B}_t^n), \qquad t \in [0,T].
    \end{equation*}
    Since $X_0^n$ and $B_t^n$ are independent, $\hat X_0^n$ and 
		$\hat{B}_t^n$ are also independent and thus their sum $\hat{X}_0^n+\hat{B}_t^n$ 
		has density $p_0^n *\hat{b}_t^n$ for all $t \in (0,T]$. 
		By scaling we see that $\hat{X}_t^n$ has the density
    \begin{equation*}
        p^n_t(z) \coloneqq  \frac{a_t^n(\lambda_1...\, \lambda_n)^{-1/2}}{\left(2\pi(a_t^2-1) \right)^{n/2}} \, 
				(p^n_0*b^n_t)(a_t \, z).
    \end{equation*}
     For $t \in (0,T]$, we observe that $b_t^n$ is twice continuously differentiable 
		and it holds 
    \begin{equation*}
        \partial_k  b_t^n(z)= - (a_t^2-1)^{-1}(\lambda_k)^{-1} z_k  b_t^n(z), \qquad k=1,...,n,
    \end{equation*}
    and
    \begin{equation*}
        \partial_{\ell} \big(\partial_k  b_t^n\big)(z) = \begin{cases}
         (a_t^2-1)^{-2}(\lambda_\ell)^{-1}(\lambda_k)^{-1} \, z_\ell z_k  b_t^n(z), &\ell\not=k\\
         (a_t^2-1)^{-2}(\lambda_k)^{-2} \,  z_k^2  b_t^n(z) -(a_t^2-1)^{-1}(\lambda_k)^{-1} b_t^n(z), & \ell=k.
        \end{cases}
    \end{equation*}
    Clearly, we have $\partial_k b_t^n \in L^1(\mathbb R^n)$ 
		and $\partial_\ell \partial_k b_t^n \in L^1(\mathbb R^n)$ 
		and therefore $\partial_k (p_0^n*b_t^n)=p_0^n*(\partial_k b_t^n)$ 
		and 
		$\partial_\ell  \partial_k (p_0^n*b_t^n)=p_0^n*(\partial_\ell \partial_k b^n_t)$. 
		In particular, $p_t^n$ is twice continuously differentiable with
    \begin{equation*}
        \partial_k p^n_t(z) = 
				-\frac{a_t^{n+1} (\lambda_k)^{-1} (\lambda_1...\, \lambda_n)^{-1/2}}{{(a_t^2-1)}\left(2\pi(a_t^2-1) \right)^{n/2}} \, 
				\left(p^n_0* (z_k b^n_t) \right)(a_t \,z), \qquad k=1,...,n, 
    \end{equation*}
    and 
    \begin{equation*}
        \partial_\ell \partial_k p_t^n(z)=  -\frac{a_t^{n+2}
				(\lambda_1...\, \lambda_n)^{-1/2}}{\left(2\pi(a_t^2-1) \right)^{n/2}} \, (p^n_0* \partial_\ell \partial_k b^n_t)(a_t z), 
				\qquad \ell, \, k \in \{1,...,n\}.
    \end{equation*}
    As it is convolution of a probability density function and a strictly positive function, we have $(p^n_0*b_t^n)(z)>0$ and thus also $p_t^n(z)>0$ for all $z \in \mathbb R^n$. The chain rule yields
    \begin{equation*}
        \nabla \log p_t^n(z) =\frac{\nabla p_t^n(z)}{p_t^n(z)}=-\frac{a_t}{(a_t^2-1)}\Big(\frac{(\lambda_k)^{-1} 
				\, (p^n_0* \cdot_k b^n_t)(a_t \, z)}{(p^n_0* b^n_t)(a_t \, z)}\Big)_{k=1}^n, 
    \end{equation*}
    which is \eqref{eq:sco}, and
    \begin{equation*}
        \nabla^2 \log p_t^n(z)=\frac{\nabla^2 p_t^n(z)}{p_t^n(z)}-\frac{\nabla p_t^n(z)\nabla p_t^n(z)^\intercal}{p_t^n(z)^2}.
    \end{equation*}
    Since all terms are continuous in $z \in \mathbb R^n$ and $t \in (0,T]$, we conclude that $\nabla^2 \log p_t^n$ is bounded on $[\delta,T] \times \{z \in \mathbb R^n\mid \lVert z\rVert\le N\}$ for all $\delta>0$ and $N\in \mathbb N$. Thus, the function $f^n$ as in \eqref{eq:score} is well-defined and Lipschitz continuous in the second variable uniformly on $[0,T-\delta]\times \{x \in H_n\mid \lVert x\rVert\le N\}$. 
   \\[1ex] 
 2.   Next, we transfer \eqref{finiteforward} to an SDE on $\mathbb R^n$ 
to obtain a reverse equation. 
Then the process 
$\bm{\beta}^n_t\coloneqq (\beta_k(t))_{k=1}^n$ 
is an $n$-dimensional Brownian motion and  $(\hat{X}^n_t)_{t \in [0,T]}$ is the unique strong solution of 
    \begin{equation*}
        {\rm d} \hat{X}^n_t = - \frac{1}{2} \alpha_t \, \hat{X}^n_t \, {\rm d} t + \sqrt{\alpha_t} \, \mathrm{diag}(\sqrt{\lambda_1},...,\sqrt{\lambda_n}) \, {\rm d} \bm{\beta}^n_t
    \end{equation*}
    with initial value $\hat{X}_0^n$. We can  verify that for all $\delta>0$ and all bounded open sets $U \subset \mathbb R^n$
    \begin{equation}\label{eq:fin}
        \int_{\delta}^T \int_U \lvert p_t^n(z)\rvert^2+\alpha_t \lambda_k\, \lvert  \partial_k p_t^n(z) \rvert^2 \, {\rm d}z \, {\rm d}t<\infty, 
				\qquad k=1,...,n.
    \end{equation}
    By \cite[Thm~2.1]{haussmann1986time} we conclude that the time reversal $(\hat{X}^n_{T-t})_{t \in [0,T)}$ 
		is a solution of the martingale problem, cf. \cite[Section~3.3]{scheutzowWT4}, associated to the reverse SDE
    \begin{equation}\label{eq:revC}
        {\rm d} \hat{Y}^n_t 
				=
				\Big( \frac{1}{2} \alpha_{T-t} \hat{Y}^n_t + \alpha_{T-t} \, 
				\mathrm{diag}(\lambda_1,...,\lambda_k) \, \nabla \log p^n_{T-t}(\hat{Y}^n_t)\Big) \, {\rm d}t + \sqrt{\alpha_{T-t}} \, \mathrm{diag}(\sqrt{\lambda_1},...,\sqrt{\lambda_n}) \, {\rm d} {\bm{\hat{\beta}}}^n_t
    \end{equation}
    with initial condition $\hat{Y}^n_0 \sim \mathbb P_{\hat{X}^n_T}$. 
		Let $\mathbb F$ be the completion of the natural filtration of $(X_{T-t}^n)_{t\in[0,T)}$. It is well-known that solutions of martingale problems correspond to weak solutions of the associated SDE. 
		In particular, by \cite[Thm~3.15]{scheutzowWT4}, 
		there is a weak solution 
		$((\hat{\Omega}, \hat{\mathcal F}, \hat{\mathbb F}, \hat{\mathbb P}),(\hat{Y}^n_t)_{t \in [0,T)},\hat{\bm{\beta}}^n)$ 
		of \eqref{eq:revC} such that 
    $(\hat{\Omega}, \hat{\mathcal F}, \hat{\mathbb F}, \hat{\mathbb P})$ 
		is an extension of $(\Omega, \mathcal F, \mathbb F, \mathbb P)$ and $(\hat{Y}^n_t)_{t \in [0,T)}$ 
		is the canonical extension of 
		$(\hat{X}_{T-t}^n)_{t \in [0,T)}$ to $(\hat{\Omega}, \hat{\mathcal F},\hat{\mathbb F}, \hat{\mathbb P})$, cf.~\cite[Def.~3.13]{scheutzowWT4}. Thus, $(\hat{Y}^n_t)_{t \in [0,T)}$ is equal to 
		$(\hat{X}^n_{T-t})_{t \in [0,T)}$ in distribution. 
		Since $\hat{\bm \beta}^n$ is an $n$-dimensional Brownian motion, 
		we have 
		$\hat{\bm \beta}_t^n = (\hat{\beta}^n_k(t))_{k=1}^n$ 
		for some mutually independent standard Brownian motions $(\hat{\beta}^n_k(t))_{k=1}^n$. 
		Therefore, the process 
    \begin{equation*}    
		\hat{W}^{Q_n}_t \coloneqq \sum_{k=1}^n \sqrt{\lambda_k} \,\hat{\beta}^n_k(t) \, e_k
		\end{equation*}
		is a $Q_n$-Wiener process on $(\hat{\Omega}, \hat{\mathcal F}, \hat{\mathbb F},\hat{\mathbb P})$. 
		We set $Y^n_t \coloneqq \iota_n^{-1}(\hat{Y}_t^n).$ 
		Clearly, $(Y_t^n)_{t \in [0,T)}$ and $(X_{T-t}^n)_{t \in [0,T)}$ 
		are equal in distribution. 
		Since $((\hat{\Omega}, \hat{\mathcal F}, \hat{\mathbb F}, \hat{\mathbb P}),(\hat{Y}^n_t)_{t \in [0,T)},\hat{\bm{\beta}}^n)$ 
		is a weak solution of \eqref{eq:revC}, the process $(Y^n_t)_{t \in [0,T)}$ 
		has continuous sample-paths and is adapted to $\hat{\mathbb F}$. 
		Further, since $Y^n_t$ fulfills \eqref{sol} for the setting  \eqref{eq:revC}, we have 
		$\hat{\mathbb P}$-almost surely for all $t \in [0,T)$ that 
\begin{equation} \begin{aligned}
        Y^n_t &=
        \sum_{k=1}^n \Big( (\hat{Y}^n_0)_k+\int_0^t \tfrac{1}{2}\alpha_{T-s}  
				(\hat{Y}^n_s)_k+\alpha_{T-s} \, \lambda_k \, \partial_k \log p^n_{T-s}(\hat{Y}^n_s) \, {\rm d} s \\
				&+\int_0^t \sqrt{\alpha_{T-s}} \sqrt{\lambda_k} \, {\rm d}  \hat{\beta}^n_k(s)\Big) \, e_k\\
                &=  \iota_n^{-1}(\hat{Y}^n_0)
				+ \int_0^t 
				\frac{1}{2}\alpha_{T-s}\,  \iota_n^{-1}(\hat{Y}^n_s )
				+ \alpha_{T-s} \, Q_n \Big(\sum_{k=1}^n  \partial_k \log p^n_{T-s}\big(\iota_n(\iota_n^{-1}(\hat{Y}^n_s))\big) e_k\Big)  {\rm d} s\\
				&+\int_0^t \sqrt{\alpha_{T-s}}  \, {\rm d}  \hat{W}^{Q_n}_s\\
                &=  Y^n_0 + \int_0^t \frac{1}{2} \alpha_{T-s} \, Y^n_s 
				+ \underbrace{\alpha_{T-s} \, Q_n \big(\iota_n^{-1} \circ \nabla \log p_{T-t}^n \circ \iota_n\big)(Y_s^n)}_{=f^n(s,Y^n_s)} \, {\rm d} s 
				+ \int_0^t \sqrt{\alpha_{T-s}}  \, {\rm d}  \hat{W}^{Q_n}_s,
   \end{aligned}\end{equation}
    and thus  $((\hat{\Omega}, \hat{\mathcal F}, \hat{\mathbb F}, \hat{\mathbb P}),(Y^n_t)_{t \in [0,T)},\hat{W}^{Q_n})$ 
		is a weak solution of \eqref{reverse_disc} with $Y_0^n \sim \mathbb P_{X^n_T}$. 
Starting with a weak solution  $((\hat{\Omega}, \hat{\mathcal F}, \hat{\mathbb F}, \hat{\mathbb P}),(Y^n_t)_{t \in [0,T)},\hat{W}^{Q_n})$ of \eqref{reverse_disc}, we can analogously show that $((\hat{\Omega}, \hat{\mathcal F}, \hat{\mathbb F}, \hat{\mathbb P}),(\hat{Y}^n_t)_{t \in [0,T)},\hat{\bm{\beta}}^n)$ with $\hat{Y}^n_t\coloneqq \iota_n(Y_t^n)$ and $\hat{\bm \beta}^n_t \coloneqq \big( (\sqrt{\lambda_k})^{-1} \langle\hat{W}^{Q_n}_t,e_k\rangle \big)_{k=1}^n$ yields a weak solution of \eqref{eq:revC}. Thus, showing uniqueness in law of \eqref{eq:revC} is sufficient to prove uniqueness in law of \eqref{reverse_disc}. Because the drift and diffusion coefficients of \eqref{eq:revC} are Lipschitz continuous in the second variable on sets of the form $[0,T-\delta] \times \{z \in \mathbb R^n\mid \lVert z\rVert \le N\}$, we know by \cite[Lem.~12.4]{russo22} and \cite[Definition~12.1]{russo22} that pathwise uniqueness holds on $[0,T-\delta]$, i.e., any two weak solutions on $[0,T-\delta]$ that are defined on the same filtered probability space with the same Brownian motion are indistinguishable. A limiting argument $\delta \to 0$ yields pathwise uniqueness on $[0,T)$. It is a classical result by Yamada and Watanabe \cite[Prop.~13.1]{russo22} that pathwise uniqueness implies uniqueness in law. 
    
    Finally, we note that by \cref{var} the forward process satisfies $\mathbb E \big[ \sup_{t \in [0,T]} \lVert X^n_t\rVert^2 \big]<\infty$. As we have proven that any weak solution $((\hat{\Omega}, \hat{\mathcal F},\hat{\mathbb F}, \hat{\mathbb P}),(Y^n_t)_{t \in [0,T)},\hat{W}^{Q_n})$ of \eqref{reverse_disc} is equal to $(X^n_{T-t})_{t \in [0,T)}$ in distribution, this immediately implies $\mathbb E \big[\sup_{t \in [0,T)} \lVert Y^n_t \rVert^2 \big]<\infty$.
\end{proof}

\subsection{Proof of \cref{thm:main}} \label{3.6}

We will need the fact that the Wasserstein-2 distance metricizes 
weak convergence in those spaces \cite[Thm 6.9]{villani2009optimal}.
More precisely, we have for a sequence $(\mu_n)_n$ that
$W_2(\mu_n,\mu) \to 0$ as $n \to \infty$ 
if and only if
$\int_K \varphi d\mu_n  \to \int_K \varphi d\mu$ for all $\varphi \in C_b(K)$ 
and $\int_K d(x_0,x)^2 \, \text{d} \mu_n \to \int_X d(x_0,x)^2 \, \rm{d} \mu$ for any $x_0 \in K$ as $n \to \infty$.

To prepare \cref{thm:main}, we need the following lemma.

\begin{lemma}[Gronwall's inequality \cite{Emmrich_2004,gronwall_lemma}] \label{gronwall}
Let $h \in L^{\infty}([0,t_0])$ for some $t_0>0$. Assume that
there exist $a \ge 0$ and $b > 0$ such that
$h(t) \leq a + b \int_0^t h(s) \, {\rm d} s$ for all $t \in [0,t_0]$. 
Then it holds 
\[
h(t) \leq a \, e^{bt}, \qquad t \in [0,t_0].
\]
\end{lemma}

Now we can prove the theorem.

\begin{proof}
1.  First, we observe that since by assumption $\tilde{f}^n$ is continuous $[0,t_0]\times H$
		and Lipschitz continuous in the second variable, 
		we can apply \cite[Thm~3.3]{gawarecki2010stochastic} to conclude that for any initial condition and $Q_n$-Wiener process independent of the initial condition the SDE \eqref{reverse_learned} has a unique strong solution on $[0,t_0]$. 
		Note that the existence of a unique strong solution implies that solutions are unique in law \cite[Remark~1.10]{watanabe}. 
		In particular, the measure $\mathbb P_{\smash{(\widetilde{Y}}_t^n)_{t_0}}$ is well-defined. 
		We use $\mathbb P_{\smash{(\breve{Y}_t^n)_{t_0}}}$ to denote the path measure induced by \eqref{reverse_learned}, 
		but with initial distribution $\mathbb P_{X_T^n}$.
    Then we obtain by the triangle inequality 
  \begin{align*}
        W_2(\mathbb{P}_{\smash{(X_{T-t}^n)_{t_0}}}, \mathbb{P}_{\smash{(\widetilde{Y}}_{t}^n)_{t_0}}) 
				&\leq  
				W_2(\mathbb{P}_{\smash{(X_{T-t}^n)_{t_0}}}, \mathbb{P}_{\smash{(Y_{t}^n)_{t_0}}}) 
				+ W_2(\mathbb{P}_{\smash{(Y_{t}^n)_{t_0}}}, \mathbb{P}_{\smash{(\breve{Y}_{t}^n)_{t_0}}})\\ &+W_2(\mathbb P_{\smash{(\breve{Y}_t^n)_{t_0}}},\mathbb P_{\smash{(\widetilde{Y}}_t^n)_{t_0}}),\\
				W_2^2(\mathbb{P}_{\smash{(X_{T-t}^n)_{t_0}}}, \mathbb{P}_{\smash{(\widetilde{Y}}_{t}^n)_{t_0}}) 
				&\leq  
				4	W_2^2(\mathbb{P}_{\smash{(X_{T-t}^n)_{t_0}}}, \mathbb{P}_{\smash{(Y_{t}^n)_{t_0}}}) 
				\\ &+ 2 W_2^2(\mathbb{P}_{\smash{(Y_{t}^n)_{t_0}}}, \mathbb{P}_{\smash{(\breve{Y}_{t}^n})_{t_0}})
				+ 4 W_2^2(\mathbb P_{\smash{(\breve{Y}_t^n)_{t_0}}},\mathbb P_{\smash{(\widetilde{Y}}_t^n)_{t_0}}).
				 \end{align*}
    By \cref{lem:reverse}, the first term is zero. 
    \\
    2.
		Next, we bound the third term. Let $\breve{Y}^n_0 $ and $\widetilde{Y}^n_0$ be any realizations of $\mathbb P_{X^n_T}$ and $\mathcal{N}(0,Q_n)$, respectively, that are defined on the same probability space. For some driving $Q_n$-Wiener process independent of  $\breve{Y}^n_0 $ and $\widetilde{Y}^n_0$ and defined on the same probability space (or possibly an extension of it), let $(\breve{Y}_t^n)_{t \in [0,t_0]}$ and $(\widetilde{Y}^n_t)_{t \in [0,t_0]}$ be the unique strong solutions of \eqref{reverse_learned} started from  $\breve{Y}^n_0 $ and $\widetilde{Y}^n_0$, respectively. By \eqref{sol} and using Jensen's inequality, we obtain for $s \in [0,t_0]$,
  \begin{align*}
    \big\lVert \breve{Y}^n_s - \widetilde{Y}^n_s \big\rVert^2 
		& \leq 
		 3 \Big( \big\lVert \breve{Y}_0^n - \widetilde{Y}^n_0 \big\rVert^2 +
		 \Big\lVert \int_0^s  \tfrac{1}{2} \alpha_{T-r} \, (\breve{Y}_r^n - \tilde{Y}_r^n)  {\rm d} r \Big\rVert^2 \\
		&+  \Big\lVert\int_0^s  \tilde{f}^n(r,\breve{Y}^n_r)  - \tilde{f}^n(r,\widetilde{Y}^n_r) \, {\rm d} r \Big\rVert^2 \Big)
		\\
    & \le 
		3 \Big( \big\lVert  \breve{Y}_0^n - \widetilde{Y}^n_0 \big\rVert^2
		+ s  \int_0^s \big\lVert  \frac12 \alpha_{T-r}  (\breve{Y}_r^n - \tilde{Y}_r^n)\big\rVert^2  {\rm d} r
		\\ &+s \int_0^s \big\lVert \tilde{f}^n(r,\breve{Y}^n_r)  - \tilde{f}^n(r,\widetilde{Y}^n_r) \big\rVert^2 {\rm d} r \Big)\\
    & \le 3 \, \big\lVert \breve{Y}_0^n - \widetilde{Y}^n_0 \big\rVert^2
		+ 
		3 t_0 \sup_{r \in [0,t_0]} \Big(\frac{\alpha_{T-r}^2}{4}+  (L_r^n)^2 \Big)  
		\int_0^s \big\lVert  \breve{Y}_r^n - \tilde{Y}_r^n\big\rVert^2 \, {\rm d} r.		
 \end{align*}
Setting $h_1(t) \coloneqq 
\mathbb E \big[ \sup_{s \in [0,t]} \big\lVert \breve{Y}^n_s-\widetilde{Y}^n_s\big\rVert^2\big]$ for $t \in [0,t_0]$, we consequently obtain
\begin{align*}
h_1(t) &\leq
3 \mathbb E \big[\big\lVert \breve{Y}_0^n - \widetilde{Y}^n_0 \big\rVert^2 \big]
+ 
3 t_0 \xi(t_0)
\mathbb E \big[\sup_{s \in [0,t]} \int_0^s  \big\lVert \breve{Y}^n_r-\widetilde{Y}^n_r \big\rVert^2 \, {\rm d} r \big]\\
&\leq
3 \mathbb E \big[\big\lVert \breve{Y}_0^n - \widetilde{Y}^n_0 \big\rVert^2 \big]
+ 
3 t_0 \xi(t_0)
\mathbb E \big[ \int_0^t  \big\lVert \breve{Y}^n_r-\widetilde{Y}^n_r \big\rVert^2 \, {\rm d} r \big]
\\
&\leq
3 \mathbb E \big[\big\lVert \breve{Y}_0^n - \widetilde{Y}^n_0 \big\rVert^2 \big]
+ 
3 t_0 \xi(t_0)
\int_0^t  h_1(r)\, {\rm d} r,
   \end{align*}
where we used Fubini's theorem for the last inequality. By \cref{var}, the discretized forward process satisfies $\mathbb E \big[ \sup_{t \in [T-t_0,T]} \lVert X^n_t\rVert^2\big]<\infty$ and in particular $\mathbb E [\lVert \breve{Y}^n_0\rVert^2]=\mathbb E [\lVert X^n_T\rVert^2]<\infty$. 
		By \cref{thm:Gau}(vi), we also have $\mathbb E [\lVert \smash{\widetilde{Y}}^n_0 \rVert^2]=\mathrm{tr}(Q_n)<\infty$. 
		Together with \cite[Thm~3.3]{gawarecki2010stochastic}, we can argue that for every $t \in [0,t_0]$ it holds
    \begin{equation*}
        h_1(t)
				=
				\mathbb E \big[ \sup_{s \in [0,t]} \big\lVert \breve{Y}^n_s-\smash{\widetilde{Y}}^n_s\big\rVert^2\big]
				\le 
				2 \, \mathbb E \big[ \sup_{s \in [0,t_0]} \big\lVert \breve{Y}^n_s\big\rVert^2\big] +2 \, \mathbb E \big[ \sup_{s \in [0,t_0]} \big\lVert \widetilde{Y}^n_s\big\rVert^2\big]<\infty,
    \end{equation*}
    so that $h_1 \in L^\infty([0,t_0])$. 
		Thus, we can apply Gronwall's lemma to $h_1$ and obtain for all $t \in [0,t_0]$ that
    \begin{equation*}
     h_1(t) 
				\le  
				3\mathbb E \big[\big\lVert \breve{Y}^n_0-\widetilde{Y}^n_0 \big\rVert^2 \big] \, 
				e^{3 \xi(t_0) t_0 \, t }.
    \end{equation*}
    In particular, the Wasserstein-2 distance can be estimated by
  \begin{align*}
         W_2^2( \mathbb P_{\smash{(\breve{Y}^n_t)_{t_0}}},\mathbb P_{\smash{(\widetilde{Y}^n_t)_{t_0}}} )  
				&\leq
				\mathbb E \Big[ \sup_{s \in [0,t_0]}  \big\lVert \breve{Y}_s^n - \widetilde{Y}^n_s \big\rVert^2 \Big] = h_1(t_0)			
				\\
				&\leq  
				3 \,e^{3 \xi(t_0)  t_0^2 } 
				 \inf_{\substack{\breve{Y}_0^n \sim \mathbb P_{X^n_T}\\ \smash{\widetilde{Y}}^n_0\sim \mathcal{N}(0,Q_n)}}  \mathbb E \Big[
				\big\lVert \smash{\breve{Y}}_0^n - \widetilde{Y}^n_0 \big\rVert^2 \Big] \\
				&= 3 \,e^{3 \xi(t_0)  t_0^2 } \, W_2^2\left(\mathbb P_{X^n_T},\mathcal{N}(0,Q_n) \right),
   \end{align*}
   where we used that the previous estimates are valid for an arbitrary realization of  $\breve{Y}^n_0 $ and $\widetilde{Y}^n_0$.
   \\
    3. It remains to bound $W_2(\mathbb P_{\smash{(Y^n_t)_{t_0}}},\mathbb P_{\smash{(\breve{Y}^n_t})_{t_0}})$. 
		We pick a coupling for the probabilistic notion of the Wasserstein distance. 
		By \cref{lem:reverse}, 
		there exists a weak solution $$((\hat{\Omega},\hat{\mathcal F},\hat{\mathbb F},\hat{\mathbb P}),(Y^n_t)_{t \in [0,t_0]}, \hat{W}^{Q_n})$$
		of the discretized reverse SDE \eqref{reverse_disc} with initial condition $Y^n_0 \sim \mathbb P_{X^n_T}$ such that $(\hat{\Omega},\hat{\mathcal F},\hat{\mathbb F},\hat{\mathbb P})$ is an extension of $(\Omega,\mathcal F,\mathbb F,\mathbb P)$ and $(Y^n_t)_{t \in [0,t_0]}$ is the canonical extension of $(X_{T-t}^n)_{t \in [0,t_0]}$ to this extended probability space. 
		We define $(\breve{Y}^n_t)_{t \in [0,t_0]}$ 
		as the strong solution of the approximate reverse  \eqref{reverse_learned} with initial condition 
		$\breve{Y}^n_0=Y^n_0$ and the same driving noise $\hat{W}^{Q_n}$.
		As before, we can bound the difference as follows for $s \in [0,t_0]$:
 \begin{align*}
\big\lVert Y^n_s - \breve{Y}^n_s \big\rVert^2 
&\leq 
\Big\lVert Y_0^n - \breve{Y}^n_0 + 
\int_0^s \frac{1}{2} \alpha_{T-r} \, (Y_r^n - \breve{Y}_r^n) \, {\rm d} r 
+ \int_0^s f^n(r,Y^n_r)  - \tilde{f}^n(r,\breve{Y}^n_r) \, {\rm d} r \Big\rVert^2\\
 &\leq 
2 t_0  
\sup_{r \in [0,t_0]} \frac{\alpha_{T-r}^2}{4} \, \int_0^s \big\lVert Y_r^n - \breve{Y}_r^n \big\rVert^2 \, {\rm d} r 
+ 2 t_0 \, \int_0^s \big\lVert f^n(r,Y^n_r)  - \tilde{f}^n(r,\breve{Y}^n_r) \big\rVert^2\, {\rm d} r .
 \end{align*}
For the second term we apply again the triangle inequality and obtain 
\begin{align*}
    & 2 t_0 \int_0^s \big\lVert f^n(r,Y^n_r)  - \tilde{f}^n(r,\breve{Y}^n_r)\big\rVert^2 \, {\rm d}r \\ 
    \le \ & 4 t_0 \int_0^s \big\lVert f^n(r,Y^n_r)  - \tilde{f}^n(r,Y^n_r)\big\rVert^2 \, {\rm d} r
		+ 4t_0 \, \int_0^s \big\lVert \tilde{f}^n(r,Y^n_r)  - \tilde{f}^n(r,\breve{Y}^n_r)\big\rVert^2 \, { \rm d} r \\
    \le \ & 4t_0 \int_0^s \big\lVert f^n(r,Y^n_r)  - \tilde{f}^n(r,Y^n_r)\big\rVert^2 \, {\rm d} r
		+  4t_0 \,  \sup_{r \in [0,t_0]} (L_r^n)^2   \,\int_0^s \big\lVert Y_r^n - \breve{Y}^n_r \big\rVert^2 \, {\rm d} r.
 \end{align*}

Since $(Y^n_t)_{t \in [0,t_0]}$ is an extension of $(X^n_{T-t})_{t \in [0,t_0]}$, we obtain from assumption \eqref{xxx} that
$$
    \mathbb E \Big[ \sup_{t \in [0,s]} \int_0^s \big\lVert f^n(r,Y^n_r)-\tilde{f}^n(r,Y^n_r)\big\rVert^2 \, {\rm d}r \Big] \le \mathbb E \Big[\int_0^{t_0} \big\lVert f^n(r,Y^n_r)-\tilde{f}^n(r,Y^n_r)\big\rVert^2 \, {\rm d}r \Big]\le \varepsilon.
$$
We set $ h_2(t) \coloneqq 
		\mathbb E \big[ \sup_{s \in [0,t]} \big\lVert Y^n_s -\breve{Y}_s^n \big\rVert^2\big]$ for $t \in [0,t_0]$ and get similarly as in the previous calculations 
$$ 
h_2(t) \le 
4t_0 \, \xi(t_0) \, \int_0^t h_2(s) \, {\rm d} s + 4\, t_0 \varepsilon.
$$
Since $(Y^n_t)_{t \in [0,t_0]}$ is equal to $(X^n_{T-t})_{t \in [0,t_0]}$ in distribution 
and $\mathbb E \big[\sup_{t \in [0,T]} \lVert X^n_t\rVert^2 \big]<\infty$ by \cref{var},
we conclude
$$
    h_2(t) =
		\mathbb E \big[ \sup_{s \in [0,t]} \big\lVert Y^n_s -\breve{Y}_s^n \big\rVert^2\big] 
		\le 
		2\ \mathbb E \big[ \sup_{s \in [0,t_0]} \big\lVert Y^n_s \big\rVert^2\big]
		+2\ \mathbb E \big[ \sup_{s \in [0,t_0]} \big\lVert \breve{Y}_s^n \big\rVert^2\big] <\infty, \qquad t\in [0,t_0]
$$
and thus $h_2 \in L^\infty([0,t_0])$. Now Gronwall's lemma applied to $h_2$ yields 
$$
W_2^2(\mathbb{P}_{\smash{(Y_{t}^n)_{t_0}}}, \mathbb{P}_{\smash{(\breve{Y}_{t}^n)_{t_0}}})
\leq h_2(t_0)\le
4 \, t_0 \varepsilon \, e^{4 \xi(t_0)   t_0^2}.
$$
In summary, we obtain
\begin{align*}
W_2^2(\mathbb{P}_{(X_{T-t}^n)_{t_0}}, \mathbb{P}_{\smash{(\widetilde{Y}}_{t}^n)_{t_0}}) 
				\leq  12
				\left(t_0 \varepsilon \, e^{4 \xi(t_0)   t_0^2} + e^{3 \xi(t_0)  t_0^2 } \, W_2^2\left(\mathbb P_{X^n_T},\mathcal{N}(0,Q_n) \right)
				\right).
\end{align*}
\end{proof}

\section{Experiment details}
\label{sec:smt}

\subsection{GMM Example}
For the Gaussian Mixture Example, we take the Adam optimizer \cite{kingma2014adam} with learning rate 6e-4, the combined prior with $k = 1.1$ and (14,7) modes, (14,14) modes for the FNO with 4 layers and width 32.  We take a batch size of 128, and train the models for 300 epochs, where one epoch consists of 10 optimizer steps. 

\subsection{MNIST Example}
For the MNIST example, we take FNO with 4 layers, width 32, (14,14) modes. For the combined and FNO prior take (32,16) modes. For the Laplacian convolution we take the scaling equal to 20. For the combined prior we take the scaling to $\gamma_0 = 0.5$ and $\gamma_1 = 10$. We train the FNO models using the Adam optimizer with learning rate 6e-4 for 200 epochs, tracking the sliced Wasserstein and (time-rescaled) DSM loss on a validation set. We save the best performing validation loss (on the coarse resolution), with which we create the images and calculate the Vendi diversity. We use Fourier downsampling with a consistent scaling.

\subsection{Reaction Diffusion Equation}
We include \cref{tab:hpset} for detailed hyperparameter setups of the 2D Reaction-Diffusion Equation example.

\begin{table}[]
\centering
\begin{tabular}{|c|cc|}
\hline
\textbf{Category}                                 & \multicolumn{2}{c|}{\textbf{Hyperparameters}}                                                                                                                                                                              \\ \hhline{|=|=|=|}
\multicolumn{1}{|c|}{\textbf{Prior}} & \multicolumn{2}{c|}{Standard, FNO, \textbf{Combined, Bessel}}                                                                                                                              \\ \hline
\textbf{Network}                                  & \multicolumn{1}{c|}{U-Net}                                                                                     & \textbf{FNO}                                                                                                        \\ \hline
\textbf{Configuration}                           & \multicolumn{1}{c|}{\begin{tabular}[c]{@{}c@{}}Channel = \{32\}, \\ Residual block count=\{2, 4\}\end{tabular}} & \begin{tabular}[c]{@{}c@{}} cutoff $k_{\max}$ =\{8, \textbf{12}, 14, 15\}, \\ Layer 1 width = 32, \\ Layer 2 width = 64 \end{tabular} \\ \hline
\textbf{Batch size}                               & \multicolumn{2}{c|}{16}                                                                                                                                                                                       \\ \hline
\textbf{Learning rate}                            & \multicolumn{2}{c|}{\{\textbf{1e-3}, 1e-4\}}                                                                                                                            \\ \hline
\end{tabular}
    \caption{Hyperparameter space used in grid search of 2D Reaction-Diffusion Equation. The best choices are marked in bold. Note that there are also different choices in the scale $\gamma_0$, $\gamma_1$, $\gamma_2$ and power $k$ for the Combined priors and Bessel priors, which we explain the rationale in \cref{sub:latent} and in the end of subsection 5.2. When using FNO prior and combine prior, the prior modes $(32,16)$ give the best results.}
    \label{tab:hpset}
\end{table}

\end{document}